\documentclass[11pt,letterpaper]{Sengoku}

\usepackage{booktabs}
\usepackage{multirow}
\usepackage{longtable}
\usepackage{makecell}
\usepackage{algorithm}
\usepackage{algpseudocode}
\usepackage[table]{xcolor} 
\usepackage{multirow}
\usepackage{booktabs}
\definecolor{best1}{HTML}{FFCCCC} 
\definecolor{best2}{HTML}{CCEEFF}

\title{\gradienttitle{ReBridge-Flow:}{Re-Coupling Posterior Bridges}{in Flow Matching}{for Image Restoration}}
\date{September 1, 2026}

\affiliationlayout{inline}
\affiliationgap{1.0em}

\affiliation{jsu}{Jiangsu University}
\affiliation{griffith}{Griffith University}
\affiliation{scu}{Sichuan University}
\affiliation{um}{University of Malaya}
\affiliation{ustc}{University of Science and Technology of China}
\affiliation{tju}{Tianjin University}
\affiliation{usq}{University of Southern Queensland}
\affiliation{sustech}{Southern University of Science and Technology}

\author[jsu]{Jiaqi Zhang}
\author[griffith]{Yiqi Wang}
\author[scu]{Hongjie Wu}
\author[um]{Bohan Guo}
\author[ustc]{Xinan Wang}
\author[tju]{Zichen Luo}
\author[usq]{Taotao Cai}
\author[usq]{Zhi Chen}
\author[sustech]{Mingkai Zheng\texorpdfstring{\textsuperscript{*}}{*}}

\metadatarowgap{0.9mm}
\addmetadata[AuxiliaryResources/Icons/Calendar.pdf]{Publication Date}{\articledate}
\addmetadata[AuxiliaryResources/Icons/Home.pdf]{Project Website}{https://jiaqizhang-sengoku.github.io/ReBridge-Flow/}
\addmetadata[AuxiliaryResources/Icons/GitHubLogo.pdf]{Code Repository}{https://github.com/JiaqiZhang-Sengoku/ReBridge-Flow}

\addmetadata[AuxiliaryResources/Icons/Email.pdf]{Corresponding Author}{Mingkai Zheng}

\hypersetup{
  pdftitle={ReBridge-Flow: Re-Coupling Posterior Bridges in Flow Matching for Image Restoration},
  pdfsubject={Research article}
}

\begin{document}

\markboth{Jiaqi Zhang}{}
\markright{ReBridge-Flow}

\begin{abstract}
Flow Matching provides an efficient generative prior for image restoration by learning continuous transport between source and data distributions. However, existing methods typically incorporate measurement constraints through local corrections. Such corrections may disrupt the source-clean endpoint coupling implicitly encoded by the pretrained flow, making the corrected endpoint pair incompatible with the current state. To address this issue, we propose ReBridge-Flow, a posterior bridge re-coupling method. Specifically, given the current state, ReBridge-Flow first decodes the corresponding local source and clean endpoints. It then incorporates measurement information through clean-side anchoring and synchronously re-couples the source endpoint, yielding a measurement-aware endpoint pair with improved local bridge compatibility. The re-coupled endpoints further define a posterior-informed transport direction for advancing the sampling process. We also introduce the Posterior Bridge Defect, which jointly characterizes measurement error, deviation from the flow prior, and bridge mismatch, and leads to explicit updates for clean-side anchoring and source-side re-coupling. Extensive experiments on multiple natural and medical image restoration tasks demonstrate that ReBridge-Flow effectively alleviates bridge mismatch and improves the structural consistency of restored images.
\end{abstract}


\maketitle

\section{Introduction}

\begin{figure}[!t]
\centering
\includegraphics[width=0.55\linewidth]{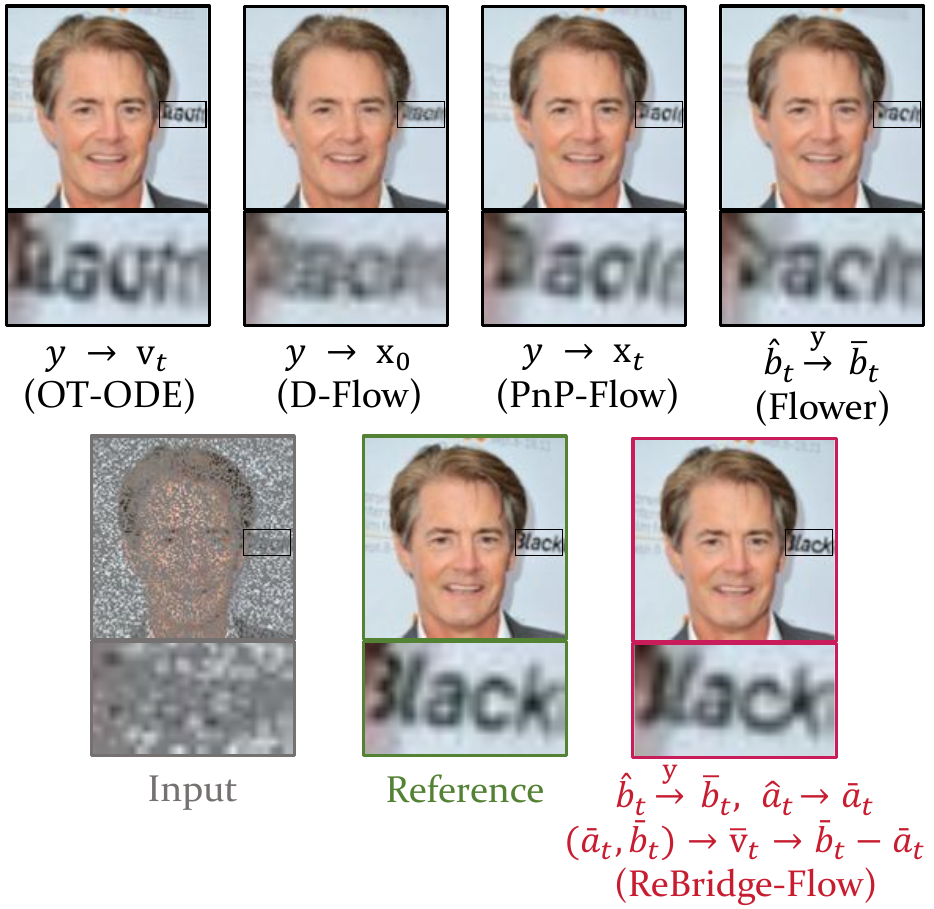} 
\caption{Comparison of restoration results under different measurement injection methods. ReBridge-Flow achieves clearer reconstructions through endpoint re-coupling.}
\label{Figure1}
\end{figure}

Image restoration aims to recover an underlying clean image
$\mathrm{x}\in\mathbb{R}^{\mathrm{n}}$ from a degraded observation
$\mathrm{y}\in\mathbb{R}^{\mathrm{m}}$~\cite{DBLP:journals/pami/Wang0H21,DBLP:journals/tmi/WangYMF18}.
This task is commonly formulated as:
\begin{equation}
\mathrm{y}=\mathrm{H}\mathrm{x}+\epsilon,
\quad
\epsilon\sim{\mathrm{N}}(0,\sigma_\mathrm{y}^2\mathrm{I}),
\label{Eq111}
\end{equation}
where $\mathrm{H}\in\mathbb{R}^{\mathrm{m}\times \mathrm{n}}$ is a known degradation operator
and $\epsilon$ denotes measurement noise.
Since $\mathrm{H}$ is typically non-invertible, the observation often admits
multiple feasible solutions.
Image restoration therefore requires both data consistency and an
effective image prior~\cite{DBLP:journals/spm/McCannJU17}.
Traditional deep restoration methods are usually trained for specific
degradations using paired data~\cite{DBLP:journals/tip/ZhangZCM017,DBLP:conf/eccv/ChenCZS22}.
In contrast, generative models can learn complex image distributions
and serve as general priors across different restoration tasks.

Diffusion models have become important tools for solving image inverse problems because of their strong generative capability~\cite{DDPM,DDIM}. Existing methods typically incorporate likelihood gradients~\cite{DBLP:conf/nips/ChungSRY22,DPS,DPPS,SPGD,SITCOM} or data-consistency projections~\cite{DDRM,DDNM,DDPG,PIRP,DPG} into the reverse sampling process to constrain generation using degraded observations. However, diffusion models usually require long iterative sampling chains, while their highly noisy intermediate states further complicate the design of measurement constraints and numerical solvers.

Recently, Flow Matching has been increasingly adopted for image restoration because it connects source and data distributions through continuous transport and enables efficient sampling~\cite{OT-ODE,FlowPriors,PnPFlow,RestoraFlow,D-Flow,Flower}. Existing methods inject measurement information at different sampling stages and impose constraints on local sampling variables to improve consistency between the restored output and the degraded observation. However, these constraints are usually treated as local interventions on individual variables, without explicitly considering the coupling between the source and clean endpoints. Since an intermediate Flow Matching state is jointly determined by an endpoint pair, correcting the clean endpoint while keeping the source endpoint fixed disrupts their original pairing. The corrected endpoint pair may then no longer accurately explain the current state. If subsequent updates continue to be constructed from such mismatched pairs, local directional errors may gradually accumulate and eventually manifest as structural drift, artifacts, or over-smoothing.

To address this issue, we propose ReBridge-Flow, which reformulates Flow Matching-based image restoration as measurement-conditioned posterior bridge re-coupling rather than the independent correction of a single sampling variable. As shown in Figure~\ref{Figure1}, existing methods apply measurement information to the velocity field, source endpoint, intermediate state, or clean-endpoint estimate. In contrast, ReBridge-Flow explicitly re-couples the source--clean endpoint pair, thereby improving detail recovery. Specifically, we first decode local source and clean endpoints from the current state using the pretrained velocity field. We then incorporate measurement information through clean-side anchoring and synchronously update the source endpoint to pair it with the corrected clean endpoint, thereby improving the local bridge compatibility between the endpoint pair and the current state. The re-coupled endpoint pair further defines a posterior-informed transport direction. To characterize this process in a unified manner, we introduce the Posterior Bridge Defect, which jointly accounts for measurement error, flow-prior preservation, and bridge residual. Our main contributions are summarized as follows:

\begin{itemize}
    \item We identify the bridge mismatch problem in Flow Matching-based image restoration: locally correcting an individual state or endpoint may disrupt the source-clean endpoint coupling and affect subsequent local transport.

    \item We propose ReBridge-Flow, which re-couples the posterior endpoint pair through clean-side anchoring and source-side re-coupling, yielding a measurement-aware endpoint pair with improved local bridge compatibility.

    \item Experiments across diverse natural and medical image restoration tasks demonstrate that ReBridge-Flow effectively suppresses error propagation, reduces structural drift and artifacts, and preserves clearer image details.
\end{itemize}

\section{Related Work}
With the development of generative models, learning-based generative priors have become an important paradigm for solving image restoration inverse problems. Among them, diffusion models and Flow Matching are widely adopted for their strong capability to model complex image distributions.
\paragraph{Diffusion-Based Image Restoration (DBIR).}
Existing DBIR methods can be broadly divided into two categories. The first category introduces measurement constraints into the reverse sampling process through gradient guidance.
DPS~\cite{DPS} guides posterior sampling using the gradient of a measurement-consistency loss.
$\Pi$GDM~\cite{PiGDM} combines a pseudoinverse operator with Jacobian computation to improve guidance accuracy.
RED-Diff~\cite{RED-Diff} formulates restoration as a
measurement-consistency optimization problem with score-matching regularization. SITCOM and SPGD~\cite{SITCOM, SPGD} explicitly control gradient updates
to improve sampling stability. The second category corrects restoration estimates through projection or
structured constraints.
DDRM~\cite{DDRM} and DDNM~\cite{DDNM} enforce data consistency using
singular value decomposition and range--null-space decomposition,
respectively.
DiffPIR~\cite{DiffPIR} employs half-quadratic splitting to convert
measurement constraints into proximal updates.
EquS~\cite{EquS} further introduces transformation consistency through an
equivariant inverse mapping and dual-trajectory sampling.

\paragraph{Flow Matching-Based Image Restoration (FMBIR).}
FMBIR methods typically incorporate measurement constraints into the sampling process of a pretrained flow model. OT-ODE~\cite{OT-ODE} directly injects measurement gradients into the ODE
dynamics to modify the velocity field and transport direction. Flow-Priors~\cite{FlowPriors} decomposes the global restoration objective into a sequence of local trajectory optimization problems. PnP-Flow~\cite{PnPFlow} alternates between data-consistency updates and flow-prior mappings. Restora-Flow~\cite{RestoraFlow} combines mask guidance with trajectory correction to keep intermediate states consistent with the degraded
observation. Another line of work operates on endpoint variables. D-Flow~\cite{D-Flow} optimizes the initial source point by backpropagating
through the complete Flow ODE. Flower~\cite{Flower} first estimates a flow-consistent clean endpoint and
then refines it using the measurement information.

\section{ReBridge-Flow}

ReBridge-Flow aims to incorporate measurement information while reducing the source-clean endpoint mismatch caused by local observation correction. As shown in Figure~\ref{Figure2}, ReBridge-Flow first decodes a local endpoint pair from the current state. It then incorporates the measurement through clean-side anchoring and re-couples the source endpoint to reduce the bridge mismatch between the corrected endpoint pair and the current state. The re-coupled endpoint pair further defines a posterior-informed local transport direction for subsequent sampling.

\paragraph{Motivation: Why Local Correction Breaks the Bridge}
In Flow Matching, the intermediate state at time $t$ is jointly determined by a source endpoint $a$ and a clean endpoint $b$:
\begin{equation}
e_t(a,b)=(1-t)a+tb.
\label{Eq2}
\end{equation}

Given the state $\mathrm{x}_t$ and a pretrained velocity field $\mathrm{v}_\theta$, the corresponding local endpoint estimates can be decoded as:
\begin{equation}
\hat a_t=\mathrm{x}_t-t\,\mathrm{v}_\theta(t,\mathrm{x}_t),
\quad
\hat b_t=\mathrm{x}_t+(1-t)\mathrm{v}_\theta(t,\mathrm{x}_t).
\label{Eq3}
\end{equation}

By construction, Eqs.~\eqref{Eq2} and~\eqref{Eq3} satisfy
$e_t(\hat a_t,\hat b_t)=\mathrm{x}_t$. We therefore refer to
$(\hat a_t,\hat b_t)$ as a pair of \emph{local pseudo-endpoints}: they provide an endpoint representation consistent with the current state and the predicted velocity.


\begin{figure}[!t]
\centering
\includegraphics[width=0.90\linewidth]{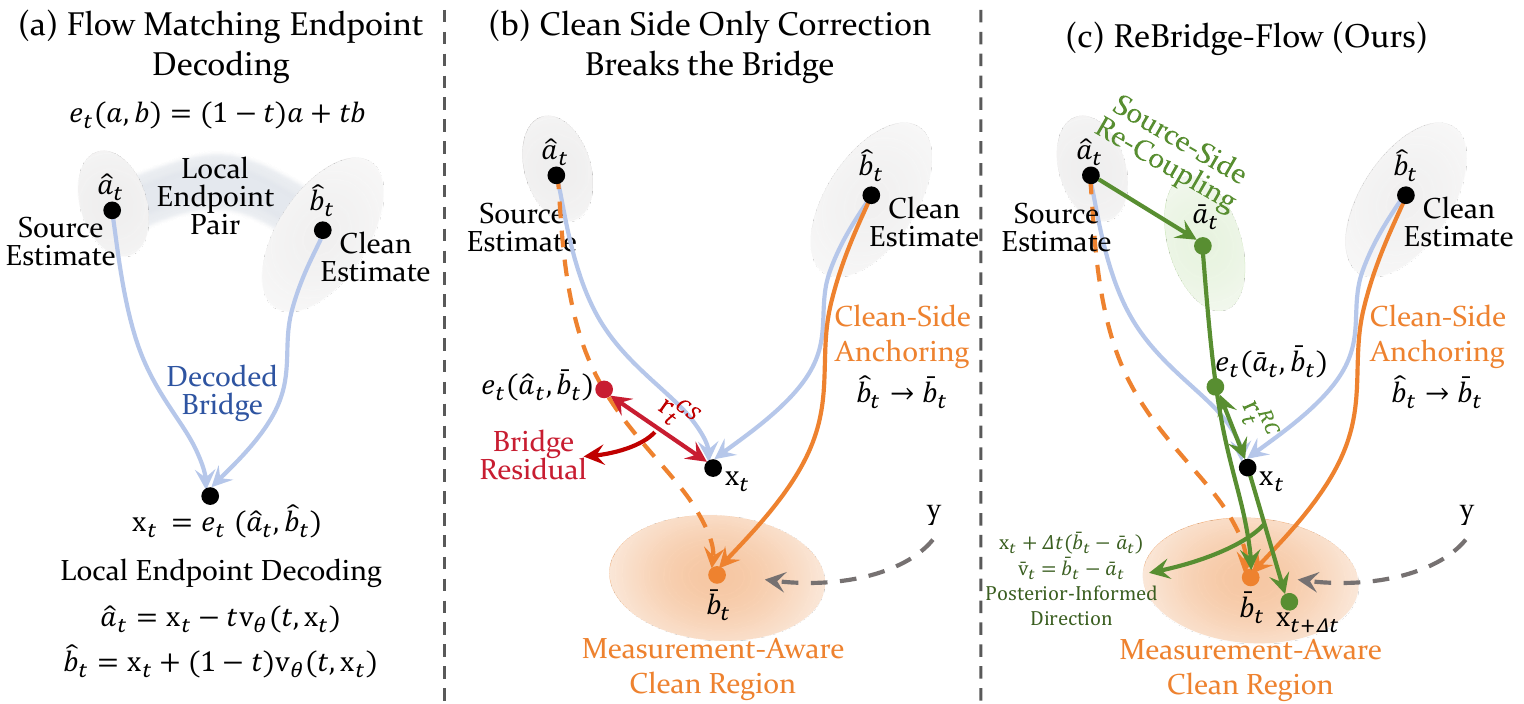} 
\caption{(a) The current state $\mathrm{x}_t$ lies on the original bridge defined by the local endpoint pair $(\hat{a}_t,\hat{b}_t)$. (b) Applying only clean-side anchoring disrupts the local compatibility between the endpoint pair and the current state. (c) ReBridge-Flow reduces the bridge residual through clean-side anchoring and source-side re-coupling.}
\label{Figure2}
\end{figure}

\begin{figure}[!t]
\centering
\includegraphics[width=\textwidth]{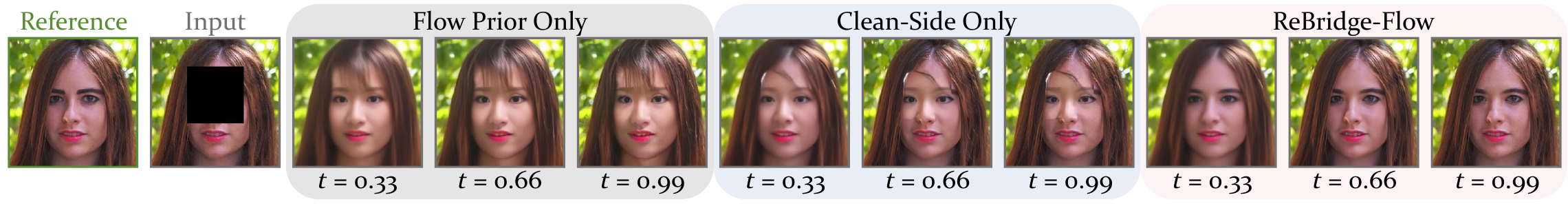} 
\caption{Bridge mismatch caused by local measurement correction. Flow Prior Only produces plausible but observation-inconsistent results, while Clean-Side Anchoring Only may accumulate structural drift and artifacts. ReBridge-Flow suppresses error propagation through source–clean endpoint re-coupling, yielding more stable restoration.}
\label{Figure3}
\end{figure}

Existing Flow Matching restoration methods typically apply local observation corrections to the current state, predicted velocity, or clean-side estimate. Although these corrections can reduce the measurement error, they do not necessarily preserve the compatibility between the corrected endpoint pair and the current state. If only the clean endpoint is corrected from $\hat b_t$ to $\bar b_t$, while the source endpoint $\hat a_t$ remains unchanged, the bridge residual after clean-side-only correction is:
\begin{equation}
\begin{gathered}
\mathrm{r}_t^{\mathrm{CS}} :=\mathrm{x}_t-e_t(\hat a_t,\bar b_t) =t(\hat b_t-\bar b_t),\\
\left\|\mathrm{r}_t^{\mathrm{CS}}\right\|_2 =t\left\|\hat b_t-\bar b_t\right\|_2.
\end{gathered}
\label{Eq4}
\end{equation}


Thus, any nonzero clean-side correction introduces an endpoint--state consistency residual proportional to the correction magnitude and to $t$.
The state itself remains unchanged; the inconsistency arises because the corrected endpoint pair no longer interpolates to the current state. This observation motivates our hypothesis that repeatedly constructing transport directions from such inconsistent pairs can accumulate trajectory error, manifested empirically as structural drift, artifacts, or over-smoothing.
Figure~\ref{Figure3} compares representative measurement-injection
strategies and supports this hypothesis.

\subsection{Posterior Bridge Defect}
Let $\pi(a,b)$ denote the unconditional endpoint coupling defined by the pretrained Flow Matching model, and let $\mathrm{P}_t=(e_t)_{\#}\pi$ be the intermediate distribution induced at time $t$. When the observation depends only on the clean endpoint, i.e., $\mathrm{p}(\mathrm{y}\mid a,b)=\mathrm{p}(\mathrm{y}\mid b)$, the measurement-conditioned endpoint coupling and its induced probability path are:
\begin{equation}
\mathrm d\pi^\mathrm{y}(a,b)
=
\frac{\mathrm{p}(\mathrm{y}\mid b)}{\mathrm{Z_y}}\,
\mathrm d\pi(a,b),
\quad
\mathrm{P}_t^\mathrm{y}=(e_t)_{\#}\pi^\mathrm{y},
\label{Eq5}
\end{equation}
where $\mathrm{Z_y}$ is the normalization constant. This relation shows that the observation not only restricts the feasible set of clean endpoints but also changes the posterior pairing between the source and clean endpoints. Therefore, correcting the clean side alone is generally insufficient to characterize the local endpoint coupling under the measurement posterior.

To jointly control the measurement error, local flow prior, and bridge mismatch, we define the Posterior Bridge Defect and its associated joint objective as:
\begin{equation}
\begin{gathered}
\operatorname{PBD}_t(\mathrm{x}_t,\mathrm{y})
:=
\min_{a,b}\mathcal J_t(a,b;\mathrm{x}_t,\mathrm{y}),\\
\begin{aligned}
&\mathcal J_t(a,b;\mathrm{x}_t,\mathrm{y})
=
\underbrace{
\frac{\|\mathrm{H}b-\mathrm{y}\|_2^2}{2\sigma_\mathrm{y}^2}
}_{\text{Measurement Defect}} + \underbrace{
\frac{\rho}{2}\|b-\hat b_t\|_2^2
+
\frac{\lambda}{2}\|a-\hat a_t\|_2^2
}_{\text{Flow-Prior Deviation}}
+
\underbrace{
\frac{\kappa}{2}
\|\mathrm{x}_t-e_t(a,b)\|_2^2
}_{\text{Bridge Residual}}.
\end{aligned}
\end{gathered}
\label{Eq6}
\end{equation}

The Measurement Defect evaluates how well the clean endpoint explains the observation. The Flow-Prior Deviation prevents the corrected endpoints from deviating excessively from the local predictions of the pretrained velocity field. The Bridge Residual measures the mismatch between the corrected endpoint pair and the current state. The parameters $\rho$, $\lambda$, and $\kappa$ control the clean-side prior, source-side prior, and bridge re-coupling strength, respectively.

\subsection{Closed-Form Posterior Bridge Re-Coupling}
For a linear degradation operator $H$, when $\rho>0$, $\lambda>0$, and $\kappa\geq0$, Eq.~\eqref{Eq6} defines a strictly convex quadratic objective in $(a,b)$. By analytically eliminating the source endpoint, we obtain the closed-form minimizer of the joint objective.

\paragraph{Posterior Clean-Side Anchoring.} For a linear degradation operator $\mathrm{H}$, minimizing the joint PBD objective in Eq.~\eqref{Eq6} and analytically eliminating the source endpoint $a$ yields the following closed-form update for the clean endpoint:

\begin{equation}
\begin{gathered}
\bar b_t
=
\hat b_t+
\mathrm{H}^\top
\left(\mathrm{H}\mathrm{H}^\top+\gamma_t \mathrm{I}\right)^{-1}
\left(\mathrm{y}-\mathrm{H}\hat b_t\right),
\end{gathered}
\label{Eq7}
\end{equation}
where $\gamma_t:=\sigma_\mathrm{y}^2\left[\rho+\frac{\kappa\lambda t^2}
{\lambda+\kappa(1-t)^2}\right]$. This update is a noise-aware proximal correction. A smaller $\gamma_t$ enforces a stronger observation constraint, whereas a larger $\gamma_t$ preserves more clean-side structure predicted by the pretrained flow. Unlike an independent data projection, $\gamma_t$ jointly accounts for the clean-side prior and the constraint induced by source-side re-coupling.

\paragraph{Source-Side Re-Coupling.} After obtaining $\bar b_t$, the source endpoint is updated in closed form as:
\begin{equation}
\bar a_t
=
\frac{
\lambda\hat a_t+
\kappa(1-t)\bigl(\mathrm{x}_t-t\bar b_t\bigr)
}{
\lambda+\kappa(1-t)^2
}.
\label{Eq8}
\end{equation}

Eq.~\eqref{Eq8} balances preservation of the source-side flow prior and reduction of the bridge residual. Therefore, source-side re-coupling is not an independent heuristic adjustment of the source endpoint. Instead, it is the exact optimal source update of the joint PBD objective.

\paragraph{Posterior-Informed State Propagation.} The re-coupled endpoint pair defines a posterior-informed local transport direction, and the current state is advanced through an explicit Euler update:
\begin{equation}
\mathrm{\bar v}_t:=\bar b_t-\bar a_t,
\quad
\mathrm{x}_{t+\Delta t}
=
\mathrm{x}_t+\Delta t\,\mathrm{\bar v}_t.
\label{Eq9}
\end{equation}

Eq.~\eqref{Eq9} does not directly project the current state onto a new bridge. Instead, it redefines the local transport direction through the re-coupled endpoints. Measurement information is therefore encoded in the new endpoint pairing and its induced local velocity, rather than being added to the pretrained velocity field as an independent external guidance term.

\subsection{ReBridge-Flow Sampling Algorithm}
The complete sampling procedure is summarized in Algorithm~\ref{Algorithm1}. Given a time schedule $0=t_0<t_1<\cdots<t_K=1$, ReBridge-Flow decodes a local source--clean endpoint pair from the current state at each sampling step. It then performs clean-side anchoring and source-side re-coupling, and advances the state using the posterior-informed direction defined by the re-coupled endpoints.

\begin{figure}[!b]
\centering
\includegraphics[width=0.60\linewidth]{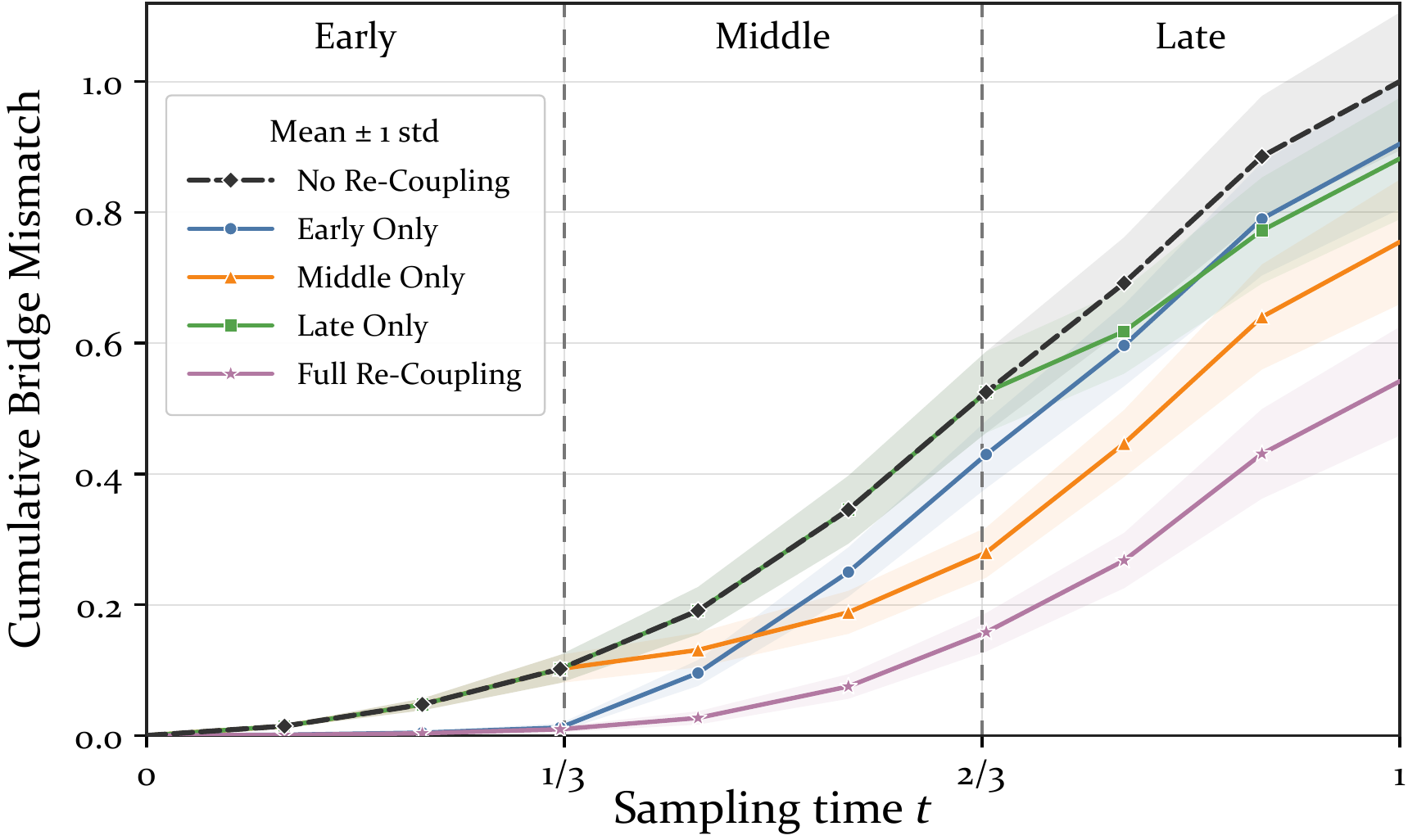}
\caption{Stage-wise analysis of source-side re-coupling.
Only the activation interval of source-side re-coupling is varied.
The curves show the cumulative normalized bridge mismatch over
the sampling trajectory.}
\label{Figure_ADD}
\end{figure}

\subsection{Theoretical Analysis}
We analyze ReBridge-Flow from two perspectives: the unique optimality of the PBD objective and the exact contraction of the bridge residual. Complete proofs are provided in the supplementary material.

\begin{proposition}[Unique Minimizer of the PBD Objective]
\label{Proposition1}
Assume that $\mathrm{H}$ is a linear operator and that $\sigma_\mathrm{y}^2>0$, $\rho>0$, $\lambda>0$, and $\kappa\geq0$. Since the endpoint-prior terms in $\mathcal J_t$ have positive weights and the remaining quadratic terms are positive semidefinite, $\mathcal J_t$ is at least $\mu$-strongly convex, where $\mu:=\min\{\lambda,\rho\}>0$. The endpoint pair obtained from Eqs.~\eqref{Eq7} and Eq.~\eqref{Eq8} satisfies:
\begin{equation}
\begin{gathered}
(\bar a_t,\bar b_t) = \operatorname*{arg\,min}_{a,b} \mathcal J_t(a,b;\mathrm{x}_t,\mathrm{y}),\\
\begin{aligned}
\mathcal J_t(a,b;\mathrm{x}_t,\mathrm{y}) &- \mathcal J_t(\bar a_t,\bar b_t;\mathrm{x}_t,\mathrm{y}) \geq {} \frac{\mu}{2} \left( \|a-\bar a_t\|_2^2+ \|b-\bar b_t\|_2^2 \right).
\end{aligned}
\end{gathered}
\label{Eq10}
\end{equation}
\end{proposition}

Therefore, Eqs.~\eqref{Eq7} and~\eqref{Eq8} jointly produce the unique global minimizer of $\mathcal J_t$. The clean- and source-side updates should thus be viewed as two components of one joint optimization problem rather than as independently designed corrections.


\begin{algorithm}[!t]
\caption{ReBridge-Flow Sampling Algorithm}
\label{Algorithm1}
\small
\begin{algorithmic}[1]
\Require Degraded Observation $\mathrm{y}$, Measurement Operator $\mathrm{H}$, Parameters $\rho,\lambda,\kappa,\sigma_{\mathrm{y}}$,
Pretrained Velocity Field $\mathrm{v}_{\theta}$, Time Schedule
$0=t_0<t_1<\cdots<t_K=1$, and Time-Dependent Noise-Aware Anchoring Coefficient $\gamma_k$
\Ensure Restored Image $\hat{\mathrm{x}}$

\State Sample the initial state $\mathrm{x}_0\sim p_0$

\For{$k=0,1,\ldots,K-1$}

    \Statex \textit{Step 1: Local Endpoint Decoding}
    \State $\mathrm{v}_k
    \gets
    \mathrm{v}_{\theta}(t_k,\mathrm{x}_k)$
    \State $\hat a_k
    \gets
    \mathrm{x}_k-t_k\mathrm{v}_k$
    \State $\hat b_k
    \gets
    \mathrm{x}_k+(1-t_k)\mathrm{v}_k$

    \Statex \textit{Step 2: Posterior Clean-Side Anchoring}
    \State $\displaystyle
    \bar b_k
    \gets
    \hat b_k+
    \mathrm{H}^{\top}
    \left(
    \mathrm{H}\mathrm{H}^{\top}
    +\gamma_k\mathrm{I}
    \right)^{-1}
    \left(
    \mathrm{y}-\mathrm{H}\hat b_k
    \right)$

    \Statex \textit{Step 3: Source-Side Re-Coupling}
    \State $\displaystyle
    \bar a_k
    \gets
    \frac{
    \lambda\hat a_k+
    \kappa(1-t_k)
    \left(
    \mathrm{x}_k-t_k\bar b_k
    \right)
    }{
    \lambda+\kappa(1-t_k)^2
    }$

    \Statex \textit{Step 4: Posterior-Informed State Propagation}
    \State $\bar{\mathrm{v}}_k
    \gets
    \bar b_k-\bar a_k$
    \State $\Delta t_k
    \gets
    t_{k+1}-t_k$
    \State $\mathrm{x}_{k+1}
    \gets
    \mathrm{x}_k+
    \Delta t_k\,\bar{\mathrm{v}}_k$

\EndFor

\State $\hat{\mathrm{x}}\gets\mathrm{x}_K$
\State \Return $\hat{\mathrm{x}}$
\end{algorithmic}
\end{algorithm}

\begin{proposition}[Exact Contraction of the Bridge Residual]
\label{Proposition2}
Define the residual after clean-side-only correction as $\mathrm{r}_t^{\mathrm{CS}}:=\mathrm{x}_t-e_t(\hat a_t,\bar b_t)$, and define the residual after source-side re-coupling as $\mathrm{r}_t^{\mathrm{RC}}:=\mathrm{x}_t-e_t(\bar a_t,\bar b_t)$. Let $\eta_t:=\lambda/[\lambda+\kappa(1-t)^2]\in(0,1]$. Eq.~\eqref{Eq8} gives:
\begin{equation}
\mathrm{r}_t^{\mathrm{RC}}
=
\eta_t \mathrm{r}_t^{\mathrm{CS}},
\quad
\left\|\mathrm{r}_t^{\mathrm{RC}}\right\|_2
=
\eta_t t\left\|\hat b_t-\bar b_t\right\|_2
\leq
\left\|\mathrm{r}_t^{\mathrm{CS}}\right\|_2.
\label{Eq11}
\end{equation}
\end{proposition}

Eq.~\eqref{Eq11} provides both the exact vector relation and the norm contraction between the residuals before and after re-coupling. When $\kappa>0$, $t\in(0,1)$, and the clean-side correction is nonzero, $\eta_t<1$, and the residual is therefore strictly contracted. Increasing the relative re-coupling strength $\kappa/\lambda$ decreases $\eta_t$ and strengthens local bridge repair. To illustrate this trajectory-level effect, Figure~\ref{Figure_ADD} compares the cumulative bridge mismatch when source-side re-coupling is enabled only in the early, middle, or late stage, or throughout the entire sampling process.
Full re-coupling consistently suppresses the mismatch, while middle-stage re-coupling provides the most pronounced reduction among the partial schedules.




\section{Experiment}
\paragraph{Datasets and Metrics.}
We evaluate ReBridge-Flow on six natural and medical image datasets. The natural image datasets include CelebA~\cite{Celeba}, AFHQ-Cat~\cite{AFHQ}, and COCO~\cite{COCO}, with resolutions of $128\times128$, $256\times256$, and $128\times128$, respectively. The medical image datasets include IXI-Brain~\cite{IXI}, PMUB~\cite{PMUB}, and X-Ray Hand~\cite{Hand1,Hand2}, with all images resized to $256\times256$. We select 100 images from each dataset for testing and use the same test samples for all methods. Following prior work, reconstruction quality is evaluated using Peak Signal-to-Noise Ratio (PSNR), Structural Similarity Index (SSIM), and Learned Perceptual Image Patch Similarity (LPIPS).

\paragraph{Tasks.}
For natural images, we consider five restoration tasks: Gaussian denoising, Gaussian deblurring, super-resolution (SR), random inpainting, and box inpainting. The noise standard deviation for denoising is $\sigma_\mathrm{y}=0.2$, and $70\%$ of pixels are removed for random inpainting. CelebA and COCO use $2\times$ SR and a $40\times40$ centered mask, while AFHQ-Cat uses $4\times$ SR and an $80\times80$ centered mask. For medical images, we evaluate Gaussian denoising, $2\times$ SR, random inpainting, and box inpainting. The denoising noise level is $\sigma_\mathrm{y}=0.08$, and $30\%$ of pixels are removed for random inpainting. IXI-Brain and X-Ray Hand use a $32\times32$ centered mask, while PMUB uses a $60\times60$ centered mask. Additionally, the added noise level is set to $\sigma_\mathrm{y}=0.05$ for deblurring and $\sigma_\mathrm{y}=0.01$ for all other tasks.

\paragraph{Implementation Details.}
For CelebA and AFHQ-Cat, we use the pretrained models from PnP-Flow~\cite{PnPFlow}. For COCO and X-Ray Hand, we use the pretrained models from Restora-Flow~\cite{RestoraFlow}. For IXI-Brain and PMUB, we train separate Flow Matching models with a learning rate of $1\times10^{-4}$, a batch size of 64, and 400 epochs. Model training is performed on a single NVIDIA A6000 48\,GB GPU, while all methods are evaluated on a single RTX 5090 32\,GB GPU. Across all datasets and restoration tasks, we set $\rho=1$, $\lambda=1$, and $\kappa=5$. The number of sampling steps is set to $K=100$, consistent with PnP-Flow.

\subsection{Comparison with State-of-the-Art Methods}
We compare ReBridge-Flow with six advanced flow-based image restoration methods: OT-ODE~\cite{OT-ODE}, Flow-Priors~\cite{FlowPriors}, D-Flow~\cite{D-Flow}, PnP-Flow~\cite{PnPFlow}, Restora-Flow~\cite{RestoraFlow}, and Flower~\cite{Flower}. For CelebA and AFHQ-Cat, we use the configurations provided by the official implementations. For COCO, IXI-Brain, PMUB, and X-Ray Hand, we validate the official parameter settings provided for CelebA and AFHQ-Cat, and report the results obtained with the best-performing configuration on the validation set.

\begin{table}[!t]
\centering
\caption{Quantitative comparison on CelebA (Top), AFHQ-Cat (Middle), and COCO (Bottom) under different restoration tasks. The \colorbox{best1}{best} and \colorbox{best2}{suboptimal} results are highlighted.}
\label{Table1}

\footnotesize 
\renewcommand{\arraystretch}{1.7} 
\setlength{\tabcolsep}{2.2pt} 
\setlength{\aboverulesep}{0pt}
\setlength{\belowrulesep}{0pt}

\resizebox{\textwidth}{!}{
\begin{tabular}{@{}l ccc ccc ccc ccc ccc@{}}
\toprule

\multicolumn{16}{c}{CelebA} \\
\midrule
\multirow{2}{*}{Model}
& \multicolumn{3}{c}{Denoising $\sigma_\mathrm{y}=0.2$}
& \multicolumn{3}{c}{Deblurring}
& \multicolumn{3}{c}{Super Res. $2\times$}
& \multicolumn{3}{c}{Rand. Inpaint. 70\%}
& \multicolumn{3}{c}{Box Inpaint. $40\times40$} \\
\cmidrule(lr){2-4} 
\cmidrule(lr){5-7} 
\cmidrule(lr){8-10} 
\cmidrule(lr){11-13} 
\cmidrule(l){14-16}
& PSNR$\uparrow$ & SSIM$\uparrow$ & LPIPS$\downarrow$
& PSNR$\uparrow$ & SSIM$\uparrow$ & LPIPS$\downarrow$
& PSNR$\uparrow$ & SSIM$\uparrow$ & LPIPS$\downarrow$
& PSNR$\uparrow$ & SSIM$\uparrow$ & LPIPS$\downarrow$
& PSNR$\uparrow$ & SSIM$\uparrow$ & LPIPS$\downarrow$ \\
\midrule
Degraded & 23.38$_{(0.25)}$ & 0.526$_{(0.066)}$ & 0.343$_{(0.133)}$ & 26.88$_{(2.01)}$ & 0.837$_{(0.019)}$ & 0.184$_{(0.091)}$ & 12.73$_{(1.32)}$ & 0.352$_{(0.074)}$ & 0.707$_{(0.250)}$ & 12.64$_{(1.32)}$ & 0.262$_{(0.065)}$ & 0.977$_{(0.250)}$ & 22.77$_{(1.56)}$ & 0.895$_{(0.008)}$ & 0.175$_{(0.060)}$ \\
\midrule

OT-ODE~\textcolor{gray}{\scriptsize [TMLR2024]} & 31.28$_{(1.00)}$ & 0.885$_{(0.024)}$ & 0.036$_{(0.014)}$ & 30.34$_{(1.67)}$ & 0.885$_{(0.019)}$ & 0.051$_{(0.023)}$ & 29.82$_{(1.60)}$ & 0.894$_{(0.026)}$ & 0.059$_{(0.023)}$ & 29.40$_{(1.74)}$ & 0.892$_{(0.026)}$ & 0.055$_{(0.029)}$ & 30.54$_{(3.21)}$ & 0.950$_{(0.016)}$ & 0.023$_{(0.010)}$ \\
Flow-Priors~\textcolor{gray}{\scriptsize [NeurIPS2024]} & 30.65$_{(0.61)}$ & 0.832$_{(0.039)}$ & 0.126$_{(0.059)}$ & 31.40$_{(0.75)}$ & 0.897$_{(0.019)}$ & 0.055$_{(0.024)}$ & 30.56$_{(0.73)}$ & 0.819$_{(0.037)}$ & 0.097$_{(0.041)}$ & 33.91$_{(2.59)}$ & 0.957$_{(0.013)}$ & 0.018$_{(0.008)}$ & 31.53$_{(3.88)}$ & \cellcolor{best2}0.966$_{(0.014)}$ & 0.028$_{(0.010)}$ \\
D-Flow~\textcolor{gray}{\scriptsize [PMLR2024]} & 28.36$_{(4.96)}$ & 0.775$_{(0.161)}$ & 0.083$_{(0.036)}$ & 32.07$_{(1.61)}$ & 0.922$_{(0.054)}$ & 0.048$_{(0.021)}$ & 32.67$_{(1.83)}$ & 0.917$_{(0.032)}$ & 0.033$_{(0.013)}$ & 33.83$_{(2.12)}$ & 0.945$_{(0.020)}$ & 0.021$_{(0.011)}$ & 30.61$_{(2.29)}$ & 0.916$_{(0.025)}$ & 0.042$_{(0.015)}$ \\
PnP-Flow~\textcolor{gray}{\scriptsize [ICLR2025]} & 32.75$_{(1.04)}$ & 0.921$_{(0.022)}$ & 0.060$_{(0.026)}$ & 34.52$_{(1.40)}$ & 0.941$_{(0.007)}$ & 0.046$_{(0.024)}$ & 32.23$_{(1.41)}$ & 0.920$_{(0.024)}$ & 0.063$_{(0.027)}$ & 33.82$_{(2.15)}$ & 0.953$_{(0.011)}$ & 0.022$_{(0.010)}$ & 31.30$_{(2.78)}$ & 0.944$_{(0.016)}$ & 0.045$_{(0.020)}$ \\
Restora-Flow~\textcolor{gray}{\scriptsize [WACV2026]} & 32.91$_{(0.94)}$ & 0.926$_{(0.014)}$ & \cellcolor{best1}0.020$_{(0.008)}$ & 30.75$_{(1.67)}$ & 0.892$_{(0.014)}$ & 0.049$_{(0.027)}$ & 32.37$_{(2.17)}$ & 0.925$_{(0.010)}$ & \cellcolor{best2}0.024$_{(0.009)}$ & \cellcolor{best1}34.08$_{(2.21)}$ & \cellcolor{best2}0.958$_{(0.013)}$ & \cellcolor{best2}0.015$_{(0.007)}$ & 31.77$_{(3.12)}$ & 0.961$_{(0.013)}$ & \cellcolor{best1}0.018$_{(0.008)}$ \\
Flower~\textcolor{gray}{\scriptsize [ICLR2026]} & 32.23$_{(1.64)}$ & 0.912$_{(0.020)}$ & \cellcolor{best2}0.033$_{(0.012)}$ & \cellcolor{best2}34.96$_{(1.62)}$ & \cellcolor{best2}0.947$_{(0.056)}$ & \cellcolor{best2}0.034$_{(0.036)}$ & \cellcolor{best2}33.92$_{(1.92)}$ & \cellcolor{best2}0.948$_{(0.015)}$ & 0.038$_{(0.006)}$ & 33.05$_{(1.85)}$ & 0.944$_{(0.016)}$ & 0.018$_{(0.006)}$ & \cellcolor{best2}31.85$_{(1.97)}$ & 0.965$_{(0.010)}$ & \cellcolor{best1}0.018$_{(0.007)}$ \\
\midrule
ReBridge-Flow~\textcolor{gray}{\scriptsize [Ours]} & \cellcolor{best1}33.33$_{(0.74)}$ & \cellcolor{best1}0.938$_{(0.005)}$ & \cellcolor{best1}0.020$_{(0.008)}$ & \cellcolor{best1}35.68$_{(0.78)}$ & \cellcolor{best1}0.956$_{(0.001)}$ & \cellcolor{best1}0.026$_{(0.012)}$ & \cellcolor{best1}34.51$_{(2.64)}$ & \cellcolor{best1}0.962$_{(0.009)}$ & \cellcolor{best1}0.014$_{(0.005)}$ & \cellcolor{best2}34.06$_{(2.43)}$ & \cellcolor{best1}0.962$_{(0.004)}$ & \cellcolor{best1}0.013$_{(0.007)}$ & \cellcolor{best1}32.53$_{(3.09)}$ & \cellcolor{best1}0.971$_{(0.006)}$ & \cellcolor{best2}0.019$_{(0.007)}$ \\

\midrule[0.8pt]

\multicolumn{16}{c}{AFHQ-Cat} \\
\midrule
\multirow{2}{*}{Model}
& \multicolumn{3}{c}{Denoising $\sigma_\mathrm{y}=0.2$}
& \multicolumn{3}{c}{Deblurring}
& \multicolumn{3}{c}{Super Res. $4\times$}
& \multicolumn{3}{c}{Rand. Inpaint. 70\%}
& \multicolumn{3}{c}{Box Inpaint. $80\times80$} \\
\cmidrule(lr){2-4} 
\cmidrule(lr){5-7} 
\cmidrule(lr){8-10} 
\cmidrule(lr){11-13} 
\cmidrule(l){14-16}
& PSNR$\uparrow$ & SSIM$\uparrow$ & LPIPS$\downarrow$
& PSNR$\uparrow$ & SSIM$\uparrow$ & LPIPS$\downarrow$
& PSNR$\uparrow$ & SSIM$\uparrow$ & LPIPS$\downarrow$
& PSNR$\uparrow$ & SSIM$\uparrow$ & LPIPS$\downarrow$
& PSNR$\uparrow$ & SSIM$\uparrow$ & LPIPS$\downarrow$ \\
\midrule
Degraded & 24.02$_{(0.23)}$ & 0.513$_{(0.080)}$ & 0.465$_{(0.169)}$ & 22.95$_{(1.99)}$ & 0.612$_{(0.091)}$ & 0.511$_{(0.237)}$ & 12.39$_{(1.90)}$ & 0.269$_{(0.080)}$ & 0.868$_{(0.250)}$ & 14.31$_{(1.90)}$ & 0.317$_{(0.087)}$ & 0.950$_{(0.250)}$ & 22.13$_{(2.33)}$ & 0.906$_{(0.010)}$ & 0.127$_{(0.056)}$ \\
\midrule

OT-ODE~\textcolor{gray}{\scriptsize [TMLR2024]} & 31.32$_{(1.34)}$ & 0.916$_{(0.022)}$ & 0.083$_{(0.033)}$ & 24.57$_{(2.09)}$ & 0.631$_{(0.074)}$ & 0.196$_{(0.105)}$ & 26.65$_{(1.71)}$ & 0.781$_{(0.080)}$ & 0.191$_{(0.072)}$ & 30.54$_{(2.05)}$ & 0.872$_{(0.030)}$ & 0.094$_{(0.039)}$ & 25.30$_{(3.07)}$ & 0.905$_{(0.014)}$ & 0.085$_{(0.035)}$ \\
Flow-Priors~\textcolor{gray}{\scriptsize [NeurIPS2024]} & 31.55$_{(0.79)}$ & 0.905$_{(0.027)}$ & 0.081$_{(0.035)}$ & 26.17$_{(2.40)}$ & 0.705$_{(0.063)}$ & \cellcolor{best1}0.185$_{(0.097)}$ & 25.89$_{(1.78)}$ & 0.730$_{(0.040)}$ & 0.186$_{(0.077)}$ & 32.23$_{(2.62)}$ & 0.890$_{(0.017)}$ & 0.068$_{(0.028)}$ & \cellcolor{best2}27.29$_{(3.80)}$ & 0.935$_{(0.015)}$ & 0.058$_{(0.023)}$ \\
D-Flow~\textcolor{gray}{\scriptsize [PMLR2024]} & 30.41$_{(3.14)}$ & 0.892$_{(0.098)}$ & 0.062$_{(0.027)}$ & 24.84$_{(2.10)}$ & 0.635$_{(0.079)}$ & 0.204$_{(0.098)}$ & 25.82$_{(2.83)}$ & 0.721$_{(0.075)}$ & 0.168$_{(0.069)}$ & 33.06$_{(2.73)}$ & 0.908$_{(0.021)}$ & 0.058$_{(0.031)}$ & 26.90$_{(2.99)}$ & 0.931$_{(0.019)}$ & 0.072$_{(0.028)}$ \\
PnP-Flow~\textcolor{gray}{\scriptsize [ICLR2025]} & 32.59$_{(1.31)}$ & 0.913$_{(0.022)}$ & 0.095$_{(0.044)}$ & 27.69$_{(2.47)}$ & 0.755$_{(0.059)}$ & 0.347$_{(0.185)}$ & 26.82$_{(2.47)}$ & 0.772$_{(0.053)}$ & 0.184$_{(0.083)}$ & 33.54$_{(2.77)}$ & \cellcolor{best2}0.922$_{(0.017)}$ & \cellcolor{best1}0.042$_{(0.022)}$ & 27.06$_{(3.64)}$ & 0.917$_{(0.016)}$ & 0.074$_{(0.031)}$ \\
Restora-Flow~\textcolor{gray}{\scriptsize [WACV2026]} & \cellcolor{best2}32.99$_{(1.28)}$ & \cellcolor{best2}0.925$_{(0.012)}$ & \cellcolor{best1}0.052$_{(0.023)}$ & 25.46$_{(1.96)}$ & 0.683$_{(0.063)}$ & 0.234$_{(0.108)}$ & \cellcolor{best2}27.95$_{(2.37)}$ & \cellcolor{best2}0.804$_{(0.056)}$ & \cellcolor{best2}0.162$_{(0.074)}$ & 33.41$_{(2.75)}$ & 0.914$_{(0.021)}$ & 0.055$_{(0.028)}$ & 27.18$_{(3.60)}$ & 0.940$_{(0.015)}$ & \cellcolor{best1}0.048$_{(0.019)}$ \\
Flower~\textcolor{gray}{\scriptsize [ICLR2026]} & 31.66$_{(1.80)}$ & 0.908$_{(0.030)}$ & 0.089$_{(0.030)}$ & 27.77$_{(1.55)}$ & \cellcolor{best2}0.764$_{(0.083)}$ & 0.253$_{(0.042)}$ & 26.31$_{(1.39)}$ & 0.745$_{(0.057)}$ & 0.201$_{(0.067)}$ & \cellcolor{best2}33.83$_{(2.62)}$ & 0.913$_{(0.018)}$ & \cellcolor{best2}0.043$_{(0.007)}$ & 27.01$_{(3.78)}$ & \cellcolor{best2}0.943$_{(0.028)}$ & 0.058$_{(0.019)}$ \\
\midrule
ReBridge-Flow~\textcolor{gray}{\scriptsize [Ours]} & \cellcolor{best1}33.51$_{(1.11)}$ & \cellcolor{best1}0.936$_{(0.008)}$ & \cellcolor{best2}0.058$_{(0.026)}$ & \cellcolor{best1}28.66$_{(2.58)}$ & \cellcolor{best1}0.781$_{(0.012)}$ & \cellcolor{best2}0.190$_{(0.101)}$ & \cellcolor{best1}28.16$_{(2.84)}$ & \cellcolor{best1}0.820$_{(0.030)}$ & \cellcolor{best1}0.119$_{(0.056)}$ & \cellcolor{best1}34.13$_{(2.81)}$ & \cellcolor{best1}0.935$_{(0.008)}$ & \cellcolor{best1}0.042$_{(0.020)}$ & \cellcolor{best1}27.42$_{(3.70)}$ & \cellcolor{best1}0.949$_{(0.008)}$ & \cellcolor{best2}0.054$_{(0.018)}$ \\

\midrule[0.8pt]

\multicolumn{16}{c}{COCO} \\
\midrule
\multirow{2}{*}{Model}
& \multicolumn{3}{c}{Denoising $\sigma_\mathrm{y}=0.2$}
& \multicolumn{3}{c}{Deblurring}
& \multicolumn{3}{c}{Super Res. $2\times$}
& \multicolumn{3}{c}{Rand. Inpaint. 70\%}
& \multicolumn{3}{c}{Box Inpaint. $40\times40$} \\
\cmidrule(lr){2-4} 
\cmidrule(lr){5-7} 
\cmidrule(lr){8-10} 
\cmidrule(lr){11-13} 
\cmidrule(l){14-16}
& PSNR$\uparrow$ & SSIM$\uparrow$ & LPIPS$\downarrow$
& PSNR$\uparrow$ & SSIM$\uparrow$ & LPIPS$\downarrow$
& PSNR$\uparrow$ & SSIM$\uparrow$ & LPIPS$\downarrow$
& PSNR$\uparrow$ & SSIM$\uparrow$ & LPIPS$\downarrow$
& PSNR$\uparrow$ & SSIM$\uparrow$ & LPIPS$\downarrow$ \\
\midrule
Degraded & 20.00$_{(0.26)}$ & 0.442$_{(0.110)}$ & 0.285$_{(0.115)}$ & 24.24$_{(2.11)}$ & 0.748$_{(0.043)}$ & 0.309$_{(0.138)}$ & 12.74$_{(1.98)}$ & 0.241$_{(0.080)}$ & 0.667$_{(0.250)}$ & 13.03$_{(1.99)}$ & 0.258$_{(0.079)}$ & 0.941$_{(0.250)}$ & 22.16$_{(2.63)}$ & 0.904$_{(0.011)}$ & 0.144$_{(0.049)}$ \\
\midrule
OT-ODE~\textcolor{gray}{\scriptsize [TMLR2024]} & 27.53$_{(1.75)}$ & 0.810$_{(0.043)}$ & 0.065$_{(0.028)}$ & 25.91$_{(2.28)}$ & 0.788$_{(0.046)}$ & 0.128$_{(0.064)}$ & 23.79$_{(2.50)}$ & 0.744$_{(0.062)}$ & 0.146$_{(0.067)}$ & 23.97$_{(2.65)}$ & 0.763$_{(0.060)}$ & 0.132$_{(0.064)}$ & 23.37$_{(3.52)}$ & 0.913$_{(0.017)}$ & \cellcolor{best1}0.072$_{(0.029)}$ \\
Flow-Priors~\textcolor{gray}{\scriptsize [NeurIPS2024]} & 27.06$_{(1.13)}$ & 0.750$_{(0.067)}$ & 0.116$_{(0.048)}$ & 27.08$_{(1.22)}$ & 0.769$_{(0.041)}$ & 0.117$_{(0.060)}$ & 24.85$_{(2.18)}$ & 0.699$_{(0.055)}$ & 0.110$_{(0.043)}$ & 25.97$_{(3.33)}$ & 0.855$_{(0.046)}$ & 0.055$_{(0.027)}$ & 23.58$_{(3.45)}$ & 0.927$_{(0.016)}$ & \cellcolor{best2}0.084$_{(0.034)}$ \\
D-Flow~\textcolor{gray}{\scriptsize [PMLR2024]} & 21.11$_{(2.07)}$ & 0.549$_{(0.085)}$ & 0.254$_{(0.099)}$ & 25.57$_{(2.09)}$ & 0.762$_{(0.094)}$ & 0.242$_{(0.127)}$ & 24.82$_{(3.04)}$ & 0.778$_{(0.062)}$ & 0.082$_{(0.034)}$ & 26.32$_{(3.18)}$ & 0.841$_{(0.052)}$ & 0.052$_{(0.025)}$ & 23.40$_{(3.05)}$ & 0.824$_{(0.041)}$ & 0.115$_{(0.044)}$ \\
PnP-Flow~\textcolor{gray}{\scriptsize [ICLR2025]} & 28.98$_{(1.82)}$ & 0.856$_{(0.044)}$ & 0.127$_{(0.050)}$ & 26.32$_{(2.29)}$ & 0.785$_{(0.019)}$ & 0.142$_{(0.071)}$ & 26.73$_{(2.24)}$ & 0.827$_{(0.047)}$ & 0.118$_{(0.048)}$ & 28.14$_{(3.04)}$ & \cellcolor{best2}0.896$_{(0.035)}$ & 0.052$_{(0.024)}$ & 24.57$_{(3.47)}$ & 0.892$_{(0.023)}$ & 0.121$_{(0.044)}$ \\
Restora-Flow~\textcolor{gray}{\scriptsize [WACV2026]} & \cellcolor{best2}30.56$_{(1.82)}$ & \cellcolor{best2}0.905$_{(0.024)}$ & \cellcolor{best1}0.025$_{(0.012)}$ & 25.21$_{(1.91)}$ & 0.747$_{(0.021)}$ & 0.135$_{(0.072)}$ & \cellcolor{best2}27.42$_{(3.13)}$ & \cellcolor{best2}0.877$_{(0.038)}$ & \cellcolor{best1}0.044$_{(0.021)}$ & 27.36$_{(3.12)}$ & 0.881$_{(0.041)}$ & \cellcolor{best2}0.040$_{(0.016)}$ & \cellcolor{best2}24.79$_{(3.57)}$ & \cellcolor{best2}0.930$_{(0.017)}$ & \cellcolor{best2}0.084$_{(0.028)}$ \\
Flower~\textcolor{gray}{\scriptsize [ICLR2026]} & 29.98$_{(1.85)}$ & 0.885$_{(0.034)}$ & 0.059$_{(0.015)}$ & \cellcolor{best2}28.82$_{(0.88)}$ & \cellcolor{best2}0.872$_{(0.068)}$ & \cellcolor{best2}0.112$_{(0.046)}$ & 27.33$_{(1.84)}$ & 0.876$_{(0.030)}$ & 0.051$_{(0.016)}$ & \cellcolor{best2}28.18$_{(2.13)}$ & 0.880$_{(0.030)}$ & 0.046$_{(0.016)}$ & 24.35$_{(4.19)}$ & 0.924$_{(0.035)}$ & 0.097$_{(0.040)}$ \\
\midrule
ReBridge-Flow~\textcolor{gray}{\scriptsize [Ours]} & \cellcolor{best1}30.65$_{(1.53)}$ & \cellcolor{best1}0.906$_{(0.009)}$ & \cellcolor{best2}0.026$_{(0.012)}$ & \cellcolor{best1}29.16$_{(1.12)}$ & \cellcolor{best1}0.909$_{(0.001)}$ & \cellcolor{best1}0.059$_{(0.029)}$ & \cellcolor{best1}27.90$_{(3.18)}$ & \cellcolor{best1}0.892$_{(0.036)}$ & \cellcolor{best2}0.045$_{(0.019)}$ & \cellcolor{best1}28.77$_{(3.32)}$ & \cellcolor{best1}0.910$_{(0.019)}$ & \cellcolor{best1}0.034$_{(0.015)}$ & \cellcolor{best1}26.01$_{(4.00)}$ & \cellcolor{best1}0.932$_{(0.011)}$ & 0.086$_{(0.032)}$ \\

\bottomrule
\end{tabular}
}
\end{table}

\begin{table}[!t]
\centering
\caption{Average quantitative comparison on IXI-Brain, PMUB, and X-Ray Hand datasets. The \colorbox{best1}{best} and \colorbox{best2}{suboptimal} results are highlighted.}
\label{Table2}

\footnotesize 
\renewcommand{\arraystretch}{1.5} 
\setlength{\tabcolsep}{1.4pt} 
\setlength{\aboverulesep}{0pt}
\setlength{\belowrulesep}{0pt}

\resizebox{0.6\columnwidth}{!}{%
\begin{tabular}{@{}l ccc ccc ccc@{}}
\toprule

\multirow{2}{*}{Model} 
& \multicolumn{3}{c}{IXI-Brain Avg.} 
& \multicolumn{3}{c}{PMUB Avg.} 
& \multicolumn{3}{c}{X-Ray Hand Avg.} \\
\cmidrule(lr){2-4} \cmidrule(lr){5-7} \cmidrule(l){8-10}
& PSNR$\uparrow$ & SSIM$\uparrow$ & LPIPS$\downarrow$ 
& PSNR$\uparrow$ & SSIM$\uparrow$ & LPIPS$\downarrow$ 
& PSNR$\uparrow$ & SSIM$\uparrow$ & LPIPS$\downarrow$ \\
\midrule

Degraded & 16.82 & 0.442 & 0.466 & 14.48 & 0.535 & 0.340 & 15.23 & 0.368 & 0.554 \\

\midrule

OT-ODE & 26.99 & 0.832 & 0.072 & 20.84 & 0.829 & 0.057 & 23.63 & 0.688 & 0.090 \\

Flow-Priors & 26.59 & 0.877 & \cellcolor{best2}0.054 & 22.82 & 0.888 & \cellcolor{best2}0.034 & 21.41 & 0.807 & 0.071 \\

D-Flow & 24.44 & 0.767 & 0.137 & 21.53 & 0.826 & 0.049 & 20.96 & 0.739 & 0.083 \\

PnP-Flow & 28.47 & 0.877 & 0.080 & \cellcolor{best2}23.61 & \cellcolor{best2}0.900 & 0.037 & 25.77 & \cellcolor{best2}0.855 & \cellcolor{best2}0.048 \\

Restora-Flow & 27.02 & 0.864 & 0.070 & 21.05 & 0.841 & 0.035 & 23.76 & 0.764 & 0.058 \\

Flower & \cellcolor{best2}28.97 & \cellcolor{best2}0.905 & 0.056 & 22.90 & 0.888 & 0.043 & \cellcolor{best2}26.16 & 0.851 & 0.053 \\

\midrule

ReBridge-Flow & \cellcolor{best1}30.60 & \cellcolor{best1}0.931 & \cellcolor{best1}0.037 & \cellcolor{best1}24.44 & \cellcolor{best1}0.920 & \cellcolor{best1}0.027 & \cellcolor{best1}27.48 & \cellcolor{best1}0.877 & \cellcolor{best1}0.046 \\

\bottomrule
\end{tabular}%
}
\end{table}

\begin{figure}[!t]
\centering
\includegraphics[width=0.98\textwidth]{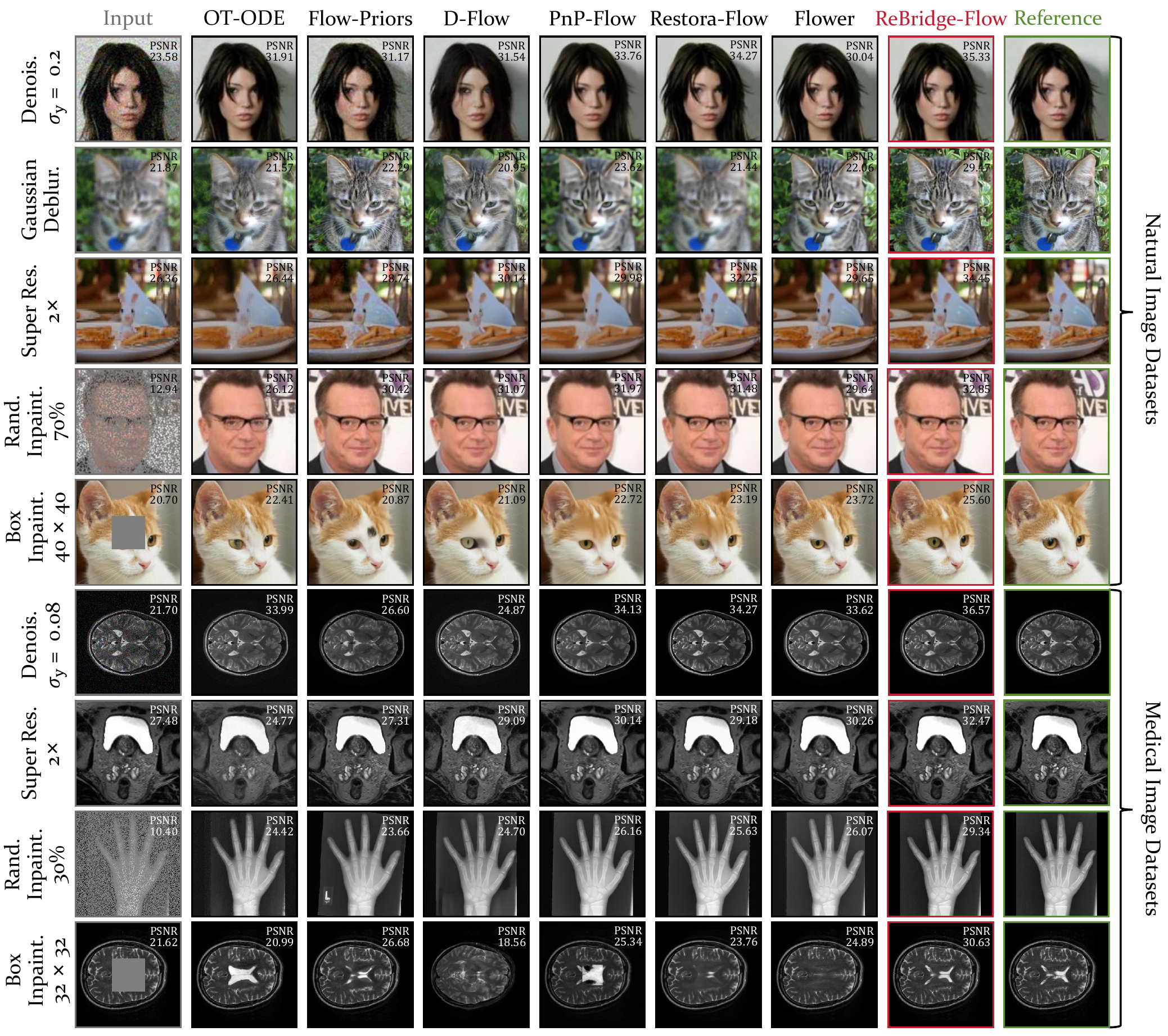} 
\caption{Qualitative comparisons of five image restoration tasks on six natural and medical datasets.}
\label{Figure4}
\end{figure}

To evaluate the restoration capability of ReBridge-Flow across different data domains and degradation settings, we conduct quantitative and qualitative comparisons on six datasets. As shown in Tables~\ref{Table1} and~\ref{Table2}, ReBridge-Flow achieves consistently competitive performance across various tasks. Compared with OT-ODE, Flow-Priors, and D-Flow, ReBridge-Flow exhibits more consistent performance across tasks and better balances reconstruction accuracy and perceptual quality. It also generally achieves higher PSNR and SSIM than PnP-Flow, Restora-Flow, and Flower. This trend is more pronounced on medical images, confirming that ReBridge-Flow preserves critical structures more reliably. Furthermore, the qualitative results in Figure~\ref{Figure4} show that baselines may retain residual noise or produce over-smoothed results in denoising and deblurring, indicating that local measurement correction struggles to balance data consistency and the flow prior. In contrast, ReBridge-Flow more effectively removes degradation while preserving sharp textures and edges. In SR and random inpainting, some methods recover plausible global content but suffer from blurred boundaries or local structural shifts. ReBridge-Flow better preserves contours and structures while recovering more natural local details. On medical images, it also reconstructs tissue boundaries with greater continuity and completeness, while producing fewer artifacts that alter the original morphology. These results demonstrate that posterior bridge re-coupling preserves local bridge compatibility while incorporating measurement information, thereby suppressing error propagation and improving reconstruction reliability.

\begin{table}[!t]
\centering
\caption{Ablation results for different endpoint-handling strategies. The \colorbox{best1}{best} and \colorbox{best2}{suboptimal} results are highlighted.}
\label{Table3}

\footnotesize 
\renewcommand{\arraystretch}{1.6} 
\setlength{\tabcolsep}{3pt} 
\setlength{\aboverulesep}{0pt}
\setlength{\belowrulesep}{0pt}

\resizebox{0.8\columnwidth}{!}{
\begin{tabular}{@{}l ccc ccc@{}}
\toprule
\multirow{2}{*}{Method} 
& \multicolumn{3}{c}{CelebA | Rand. Inpaint. 70\%} 
& \multicolumn{3}{c}{AFHQ-Cat | Super Res. $4\times$} \\
\cmidrule(lr){2-4} \cmidrule(l){5-7}
& PSNR$\uparrow$ & SSIM$\uparrow$ & LPIPS$\downarrow$ 
& PSNR$\uparrow$ & SSIM$\uparrow$ & LPIPS$\downarrow$ \\
\midrule
Flow Prior Only      & 29.80 & 0.884 & 0.083 & 27.04 & 0.782 & 0.150 \\
Clean-Side Anchoring Only ($a = \hat{a}_t$)  
& 28.89 & 0.891 & 0.059 & \cellcolor{best2}27.26 & \cellcolor{best2}0.793 & \cellcolor{best2}0.138 \\
Hard Re-Coupling ($\lambda = 0$)   
& \cellcolor{best2}29.83 & \cellcolor{best2}0.906 & \cellcolor{best2}0.028 & 27.19 & 0.786 & 0.143 \\
No Clean-Side Prior ($\rho = 0$)        
& 14.63 & 0.645 & 0.363 & 11.90 & 0.225 & 0.898 \\
\midrule
ReBridge-Flow   & \cellcolor{best1}34.06 & \cellcolor{best1}0.962 & \cellcolor{best1}0.013 & \cellcolor{best1}28.16 & \cellcolor{best1}0.820 & \cellcolor{best1}0.119 \\
\bottomrule
\end{tabular}%
}
\end{table}

\subsection{Ablation Study}
To compare different endpoint-handling strategies, we conduct ablation experiments on CelebA with Random Inpainting 70\% and AFHQ-Cat with $4\times$ SR. As shown in Table~\ref{Table3}, using the flow prior alone cannot incorporate the observation and therefore yields limited restoration performance. Clean-side anchoring ($a=\hat{a}_t$) improves measurement consistency, but keeping the source endpoint fixed still breaks the local bridge compatibility between the endpoint pair and the current state. Hard Re-Coupling ($\lambda=0$) further reduces the bridge residual, but becomes less stable because the source-side prior is removed. No Clean-Side Prior ($\rho=0$) causes a substantial performance drop, demonstrating that this constraint is essential for preventing measurement corrections from deviating from the local flow prior. The full ReBridge-Flow achieves the best performance on both tasks, indicating that effective restoration requires preserving the flow priors at both endpoints while jointly correcting their coupling.

\subsection{Hyperparameter Sensitivity Analysis}
To analyze the effects of the source-side prior weight $\lambda$ and the clean-side prior weight $\rho$, we fix $\kappa=5$ and vary $\rho$ and $\lambda$. As shown in Figure~\ref{Figure5}, ReBridge-Flow remains stable over a moderate parameter range and achieves the best performance at $\rho=\lambda=1$, with a PSNR of 28.16 and an LPIPS of 0.119. A large $\lambda$ overly restricts source-endpoint adaptation and weakens bridge re-coupling, whereas a small $\lambda$ fails to preserve the source-side flow prior. Similarly, extreme values of $\rho$ disrupt the balance between measurement consistency and the clean-side prior.

\begin{figure}[!t]
\centering
\includegraphics[width=0.8\linewidth]{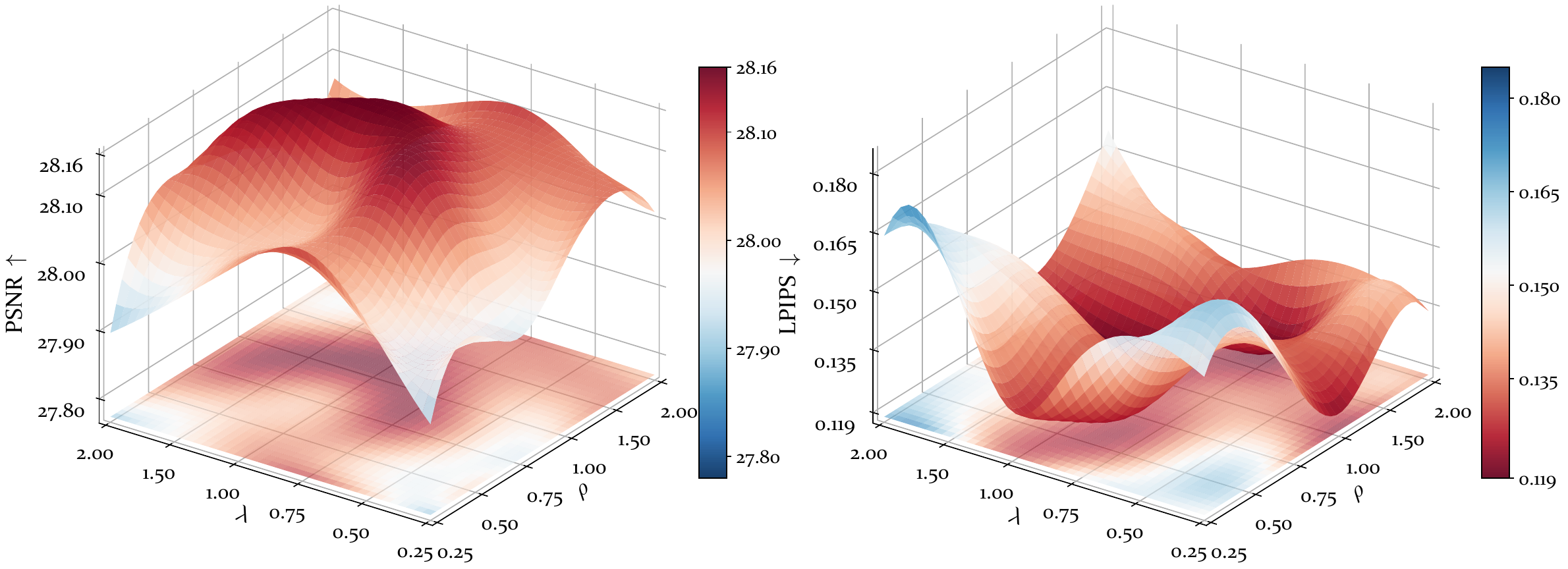} 
\caption{Parameter sensitivity analysis on $\lambda$ and $\rho$ of ReBridge-Flow.}
\label{Figure5}
\end{figure}

\subsection{Computational Efficiency}
Table~\ref{Table4} compares the restoration performance and computational cost of different methods on the CelebA deblurring task. ReBridge-Flow achieves the best results across all three quantitative metrics, with an average inference time of 6.75s and a GPU memory footprint of only 0.79~GB. Compared with most baselines that rely on iterative optimization or backpropagation, ReBridge-Flow runs faster and requires less GPU memory. Although Restora-Flow has a shorter inference time and OT-ODE uses slightly less memory, both deliver substantially lower restoration performance than ReBridge-Flow. These results demonstrate that the closed-form endpoint re-coupling mechanism achieves a favorable balance between restoration quality and computational efficiency without additional iterative optimization.

\begin{table}[!t]
\centering
\caption{Comparison of deblurring performance and computational efficiency across different methods. The \colorbox{best1}{best} results are highlighted.}
\label{Table4}

\footnotesize 
\renewcommand{\arraystretch}{1.4} 
\setlength{\tabcolsep}{4pt} 
\setlength{\aboverulesep}{0pt}
\setlength{\belowrulesep}{0pt}

\resizebox{0.90\columnwidth}{!}{%
\begin{tabular}{lccccccc}
\toprule
Method & OT-ODE & Flow-Priors & D-Flow & PnP-Flow & Restora-Flow & Flower & ReBridge-Flow \\
\midrule
PSNR$\uparrow$
& 30.34
& 31.40
& 32.07
& 34.52
& 30.75
& 34.96
& \cellcolor{best1}35.68 \\

SSIM$\uparrow$
& 0.885
& 0.897
& 0.922
& 0.941
& 0.892
& 0.947
& \cellcolor{best1}0.956 \\

LPIPS$\downarrow$
& 0.051
& 0.055
& 0.048
& 0.046
& 0.049
& 0.034
& \cellcolor{best1}0.026 \\

Avg Time (s)$\downarrow$
& 7.91
& 46.07
& 125.01
& 8.86
& \cellcolor{best1}3.41
& 11.46
& 6.75 \\

Memory (GB)$\downarrow$
& \cellcolor{best1}0.70
& 4.62
& 6.30
& 6.30
& 0.79
& 6.32
& 0.79 \\
\bottomrule
\end{tabular}%
}
\end{table}
\section{Discussion}
At the level of the resulting state update, both ReBridge-Flow and velocity-guidance methods modify the local transport direction. ReBridge-Flow may therefore appear to be merely a stronger form of measurement guidance. However, they incorporate observation corrections in different ways. OT-ODE~\cite{OT-ODE} directly adds the measurement gradient to the ODE velocity, whereas Flower~\cite{Flower} uses the observation to correct the clean endpoint while leaving the source endpoint determined by the original local flow prediction. Given the same corrected clean endpoint $\bar b_t$, a clean-side-only correction produces the direction $\bar b_t-\hat a_t$. ReBridge-Flow instead synchronously updates the source endpoint through the joint PBD objective and propagates with $\bar b_t-\bar a_t$. Re-coupling therefore does not simply amplify the measurement residual. Instead, it controls how the observation correction is allocated across the two endpoints while preserving the local flow priors on both the source and clean sides. Simply increasing the strength of external guidance may further reduce the measurement error, but it does not directly resolve the inconsistency between the corrected endpoint pair and the current state.

The bridge residual is neither equivalent to the final reconstruction error nor a substitute for a measurement-consistency metric. It indicates whether the corrected endpoint pair can still explain the current state. Reducing this residual alone does not guarantee monotonic improvements in PSNR or SSIM at every step. It can, however, help prevent subsequent transport directions from being repeatedly constructed from an inconsistent local endpoint pair. In the PBD objective, the Measurement Defect, Flow-Prior Deviation, and Bridge Residual constrain measurement consistency, preservation of the generative prior, and local propagation consistency, respectively. They therefore serve distinct roles. The importance of re-coupling also varies with sampling time. At early times, the effect of clean-endpoint correction on the bridge residual is limited by the coefficient $t$. Near the terminal time, the interpolation weight $1-t$ of the source endpoint becomes small, which reduces its ability to compensate for the residual. At intermediate times, both endpoints have non-negligible interpolation weights. Re-coupling therefore typically produces a more pronounced residual reduction, consistent with the stage-wise experimental results. Similar to Flow-Priors~\cite{FlowPriors} and Restora-Flow~\cite{RestoraFlow}, which emphasize local-trajectory or intermediate-state consistency, ReBridge-Flow constrains local propagation errors as they arise rather than replacing the final data-consistency objective with the bridge residual.

\paragraph{Limitations.}
When the degradation operator is known and linear, and the pretrained flow model is reasonably matched to the target image domain, ReBridge-Flow can coordinate measurement constraints and endpoint priors in closed form and achieve stable performance across multiple restoration tasks. However, its current scope remains primarily limited to known linear degradations and assumes that the measurement-noise level can be reasonably estimated. For general large-scale operators, solving the linear system required by clean-side anchoring may also introduce additional computational cost. Meanwhile, the quality of the local endpoints depends on the pretrained velocity field. When test images deviate substantially from the training distribution or the observation loses too much information, the restored results may still be affected by biases in the generative prior. The current theoretical results guarantee local PBD optimality and bridge-residual contraction, but they do not directly guarantee a monotonic decrease in the global reconstruction error. Moreover, the medical-image experiments use known synthetic degradations. These results are intended mainly to validate restoration performance and cannot replace systematic evaluation under real clinical degradations.

\paragraph{Future Work.}
In future work, we will extend ReBridge-Flow to nonlinear and unknown degradations, for example by deriving corresponding endpoint updates through local linearization or joint estimation of the degradation operator. We will also investigate adaptive adjustment of $\rho$, $\lambda$, and $\kappa$ according to the sampling stage, measurement residual, and bridge residual. This may reduce manual parameter selection across different data domains and degradation levels.

\section{Conclusion}
We proposed ReBridge-Flow, which reformulated FMBIR as a measurement-conditioned posterior bridge re-coupling problem. To address the disruption of source-clean endpoint pairing caused by local measurement correction, we constructed a measurement-consistent and locally bridge-compatible endpoint pair through clean-side anchoring and source-side re-coupling, and used it to define a posterior-informed transport direction. The Posterior Bridge Defect further unified measurement error, deviation from the flow prior, and bridge residual, yielding closed-form updates. Extensive experiments demonstrated that ReBridge-Flow effectively alleviated bridge mismatch and achieved stable performance across diverse restoration tasks. In the future, we will extend the method to image restoration under nonlinear and unknown degradations.


\bibliography{References}

\appendix
\clearpage
\section{More Related Work}
Existing generative image restoration methods differ not only in the generative priors they employ, but also in where measurement information is injected into the generative process and whether the resulting local correction remains consistent with the transport structure of the pretrained model. Diffusion-based posterior sampling methods typically convert the measurement likelihood into stepwise corrections to the score or denoising direction. For example, DPS~\cite{DPS}, $\Pi$GDM~\cite{PiGDM}, SITCOM~\cite{SITCOM}, SPGD~\cite{SPGD}, PIRP~\cite{PIRP}, DPG~\cite{DPG}, LEADer~\cite{LEADer}, and~\cite{DDGF} guide reverse sampling through measurement-residual gradients, pseudoinverse operators, or controlled gradient updates. These methods can handle various degradation tasks. However, the measurement constraints mainly act on local estimates around the current noisy state and usually require repeated gradient or Jacobian evaluations throughout the sampling chain. Another class of methods formulates data consistency as an explicit optimization problem or a proximal subproblem. RED-Diff~\cite{RED-Diff} and DiffPIR~\cite{DiffPIR} coordinate the generative prior and measurement constraints through variational regularization and half-quadratic splitting, respectively. DDRM~\cite{DDRM}, DDNM~\cite{DDNM}, and EquS~\cite{EquS} instead construct structured data-consistency updates using singular value decomposition, range--null-space decomposition, or equivariance constraints. These methods impose more direct constraints on the restoration result, but their update variables are usually still the current state or the current clean-image estimate. The relationships among different local variables along the generative trajectory are not explicitly modeled. In Flow Matching, image restoration is formulated as continuous transport from a source distribution to a data distribution. This formulation allows measurement constraints to act on the velocity field, trajectory state, or endpoint variables. OT-ODE~\cite{OT-ODE} directly modifies the ODE velocity, Flow-Priors~\cite{FlowPriors} decomposes the global inverse problem into local trajectory optimization problems, PnP-Flow~\cite{PnPFlow} alternates between data-consistency and flow-prior mappings, and Restora-Flow~\cite{RestoraFlow} constrains intermediate states through mask guidance and trajectory correction. Although these methods use different solution strategies, they mainly treat measurement injection as an intervention on a single local variable, while the interpolation relation among the source endpoint, clean endpoint, and current state remains implicit. Endpoint-based methods make more explicit use of the transport structure of Flow Matching. D-Flow~\cite{D-Flow} optimizes the source endpoint by backpropagating through the full Flow ODE. This enables control over the final result from the initial condition, but requires gradient computation along the entire generative trajectory. Flower~\cite{Flower} estimates a flow-consistent clean endpoint from the current state and then corrects it using the measurement, thereby avoiding backpropagation through the full trajectory. Prior work has therefore progressed from velocity guidance and state projection to one-sided endpoint correction. However, it has not directly addressed whether the source and clean endpoints can still jointly explain the current state after measurement correction. Building on this line of work, ReBridge-Flow treats the two endpoints as joint variables of the same local bridge. It explicitly re-couples the measurement-conditioned source--clean endpoint pair by jointly constraining the measurement error, endpoint deviations from the flow prior, and the bridge residual between the endpoint pair and the current state.

\section{Symbols Table}
We summarize the main notations used in this paper and their meanings in Table~\ref{Symbols}.
\begingroup 
\small 
\renewcommand{\arraystretch}{1.15}
\setlength{\tabcolsep}{6pt}
\setlength{\aboverulesep}{0pt}
\setlength{\belowrulesep}{0pt}

\begin{longtable}{@{}p{0.30\linewidth}p{0.60\linewidth}@{}}
\caption{Symbols Table.}
\label{Symbols} \\
\toprule
Symbol & Description \\
\midrule
\endfirsthead

\multicolumn{2}{c}{{\tablename\ \thetable{} -- continued from previous page}} \\
\toprule
Symbol & Description \\
\midrule
\endhead

\midrule
\multicolumn{2}{r}{{Continued on next page}} \\
\endfoot

\bottomrule
\endlastfoot

\multicolumn{2}{@{}l}{\textit{Observation Model and Degradation Process}} \\[2pt]
$\mathrm{x}\in\mathbb{R}^{\mathrm{n}}$ & Clean image to be recovered \\
$\mathrm{y}\in\mathbb{R}^{\mathrm{m}}$ & Degraded observation \\
$\mathrm{n},\mathrm{m}$ & Dimensions of the image and observation \\
$\mathrm{H}\in\mathbb{R}^{\mathrm{m}\times\mathrm{n}}$ & Known linear degradation operator \\
$\mathrm{H}^{\top}$ & Transpose of $\mathrm{H}$ \\
$\epsilon$ & Measurement noise \\
$\sigma_\mathrm{y}$ & Measurement-noise standard deviation \\
$\mathrm{I}$ & Identity operator \\
${\mathrm{N}}(0,\sigma_\mathrm{y}^{2}\mathrm{I})$ & Zero-mean Gaussian noise distribution \\

\midrule
\multicolumn{2}{@{}l}{\textit{Flow Matching and Local Endpoint Decoding}} \\[2pt]
$t\in[0,1]$ & Continuous sampling time \\
$a,b$ & Source and clean endpoints \\
$e_t(a,b)$ & Linear interpolation map between endpoints \\
$\mathrm{x}_t$ & Intermediate state at time $t$ \\
$\mathrm{v}_\theta(t,\mathrm{x}_t)$ & Pretrained velocity field \\
$\theta$ & Parameters of the velocity field \\
$\hat a_t,\hat b_t$ & Decoded local source and clean endpoints \\
$\pi(a,b)$ & Unconditional endpoint coupling \\
$\mathrm{P}_t=(e_t)_{\#}\pi$ & Intermediate distribution induced by $\pi$ \\
$(e_t)_{\#}$ & Pushforward operator associated with $e_t$ \\

\midrule
\multicolumn{2}{@{}l}{\textit{Posterior Bridge Defect and Endpoint Re-Coupling}} \\[2pt]
$\mathrm{p}(\mathrm{y}\mid b)$ & Observation likelihood conditioned on $b$ \\
$\pi^\mathrm{y}$ & Measurement-conditioned endpoint coupling \\
$\mathrm{P}_t^\mathrm{y}$ & Measurement-conditioned probability path \\
$\mathrm{Z_y}$ & Normalization constant of $\pi^\mathrm{y}$ \\
$\mathcal J_t(a,b;\mathrm{x}_t,\mathrm{y})$ & Joint Posterior Bridge Defect objective \\
$\operatorname{PBD}_t(\mathrm{x}_t,\mathrm{y})$ & Minimum value of the joint objective \\
$\rho,\lambda,\kappa$ & Clean-side prior, source-side prior, and re-coupling weights \\
$\gamma_t$ & Noise-aware clean-side anchoring coefficient \\
$\bar b_t$ & Measurement-aware clean endpoint \\
$\bar a_t$ & Re-coupled source endpoint \\
$\mathrm{\bar v}_t$ & Posterior-informed direction from the re-coupled endpoints \\
$\mathrm{r}_t^{\mathrm{CS}}$ & Bridge residual after clean-side-only correction \\
$\mathrm{r}_t^{\mathrm{RC}}$ & Bridge residual after source-side re-coupling \\
$\eta_t$ & Bridge-residual contraction factor \\
$\mu=\min\{\lambda,\rho\}$ & Strong-convexity lower bound of the PBD objective \\

\midrule
\multicolumn{2}{@{}l}{\textit{Discrete Sampling Process}} \\[2pt]
$K$ & Total number of sampling steps \\
$k$ & Sampling-step index \\
$t_k$ & Sampling time at step $k$ \\
$\Delta t_k=t_{k+1}-t_k$ & Time increment at step $k$ \\
$p_0$ & Source distribution of Flow Matching \\
$\mathrm{x}_k$ & Sampling state at step $k$ \\
$\mathrm{v}_k$ & Pretrained velocity at step $k$ \\
$\hat a_k,\hat b_k$ & Decoded local endpoints at step $k$ \\
$\bar a_k,\bar b_k$ & Re-coupled endpoints at step $k$ \\
$\bar{\mathrm{v}}_k$ & Posterior-informed direction at step $k$ \\
$\hat{\mathrm{x}}$ & Final restored image \\

\midrule
\multicolumn{2}{@{}l}{\textit{Mathematical Operations and Evaluation Metrics}} \\[2pt]
$\|\cdot\|_2$ & Euclidean $\ell_2$ norm \\
PSNR & Peak signal-to-noise ratio \\
SSIM & Structural similarity index \\
LPIPS & Learned perceptual image patch similarity \\

\end{longtable}
\endgroup

\section{Additional Preliminaries and Standing Identities}

\subsection{Linear Flow Matching Bridge}

For a source endpoint $a$ and a clean endpoint $b$, we use the linear Flow Matching bridge~\cite{DBLP:conf/iclr/LipmanCBNL23, DBLP:conf/iclr/LiuG023, AgentTraces}:
\begin{equation}
e_t(a,b)=(1-t)a+tb.
\label{EqSuppLinearBridge}
\end{equation}

This interpolation satisfies the endpoint conditions $e_0(a,b)=a$ and $e_1(a,b)=b$. Differentiating it with respect to time gives:
\begin{equation}
\frac{\partial e_t(a,b)}{\partial t}=b-a.
\label{EqSuppBridgeVelocity}
\end{equation}

Therefore, under the linear bridge, the endpoint difference $b-a$ gives the constant conditional velocity from the source endpoint to the clean endpoint. In the following, we use the pretrained velocity field at the current state to predict this local endpoint difference, from which we decode a local endpoint representation compatible with the current state.

\subsection{Local Pseudo-Endpoint Identities}

Given the current state $\mathrm{x}_t$ and the pretrained velocity field $\mathrm{v}_\theta(t,\mathrm{x}_t)$, the local source and clean endpoints are defined as:
\begin{equation}
\hat a_t=\mathrm{x}_t-t\,\mathrm{v}_\theta(t,\mathrm{x}_t),
\quad
\hat b_t=\mathrm{x}_t+(1-t)\mathrm{v}_\theta(t,\mathrm{x}_t).
\label{EqSuppPseudoEndpoints}
\end{equation}

We next verify two basic identities satisfied by this endpoint pair. First, substituting Eq.~\eqref{EqSuppPseudoEndpoints} into the linear bridge gives:
\begin{align}
e_t(\hat a_t,\hat b_t)
&=(1-t)\hat a_t+t\hat b_t \notag\\
&=(1-t)\left[\mathrm{x}_t-t\,\mathrm{v}_\theta(t,\mathrm{x}_t)\right]
+t\left[\mathrm{x}_t+(1-t)\mathrm{v}_\theta(t,\mathrm{x}_t)\right] \notag\\
&=(1-t)\mathrm{x}_t+t\mathrm{x}_t \notag\\
&=\mathrm{x}_t.
\label{EqSuppPseudoInterpolation}
\end{align}

Second, the difference between the two local endpoints is:
\begin{align}
\hat b_t-\hat a_t
&=\mathrm{x}_t+(1-t)\mathrm{v}_\theta(t,\mathrm{x}_t)
-\left[\mathrm{x}_t-t\,\mathrm{v}_\theta(t,\mathrm{x}_t)\right] \notag\\
&=\mathrm{v}_\theta(t,\mathrm{x}_t).
\label{EqSuppPseudoVelocity}
\end{align}

Therefore, $(\hat a_t,\hat b_t)$ satisfies both:
\begin{equation}
e_t(\hat a_t,\hat b_t)=\mathrm{x}_t,
\quad
\hat b_t-\hat a_t=\mathrm{v}_\theta(t,\mathrm{x}_t).
\label{EqSuppTwoIdentities}
\end{equation}

The first identity shows that this endpoint pair exactly reconstructs the current state through linear interpolation. The second shows that the endpoint difference equals the current pretrained velocity. Since this pair is decoded locally at each current state and is not required to share fixed endpoints with a single global sample trajectory, we refer to it as a pair of local pseudo-endpoints.

\subsection{Clean-Side-Only Bridge Residual}

Consider the case in which measurement correction is applied only to the clean endpoint. Specifically, $\hat b_t$ is corrected to $\bar b_t$, while the source endpoint remains $\hat a_t$. The corrected endpoint pair generally no longer interpolates exactly to the current state. We define the resulting bridge residual as:
\begin{equation}
\mathrm{r}_t^{\mathrm{CS}}
:=
\mathrm{x}_t-e_t(\hat a_t,\bar b_t).
\label{EqSuppCSDefinition}
\end{equation}

Using Eq.~\eqref{EqSuppPseudoInterpolation}, this residual can be expanded directly as:
\begin{align}
\mathrm{r}_t^{\mathrm{CS}}
&=e_t(\hat a_t,\hat b_t)-e_t(\hat a_t,\bar b_t) \notag\\
&=\left[(1-t)\hat a_t+t\hat b_t\right]
-\left[(1-t)\hat a_t+t\bar b_t\right] \notag\\
&=t(\hat b_t-\bar b_t).
\label{EqSuppCSResidual}
\end{align}

Since $t\in[0,1]$, taking the $\ell_2$-norm of Eq.~\eqref{EqSuppCSResidual} gives:
\begin{equation}
\left\|\mathrm{r}_t^{\mathrm{CS}}\right\|_2
=
t\left\|\hat b_t-\bar b_t\right\|_2.
\label{EqSuppCSResidualNorm}
\end{equation}

This result shows that the current state $\mathrm{x}_t$ itself is not changed by clean-side anchoring. Instead, the inconsistency arises from the corrected endpoint representation. When $\bar b_t\neq\hat b_t$ and $t>0$, correcting only the clean endpoint introduces a nonzero bridge residual. For clean-endpoint corrections of the same magnitude, the residual grows linearly with the current time $t$.

\subsection{Measurement-Conditioned Endpoint Coupling}

Let $\pi(a,b)$ denote the unconditional source--clean endpoint coupling induced by the pretrained Flow Matching model. The intermediate distribution obtained by pushing this coupling forward through the linear map $e_t$ is:
\begin{equation}
\mathrm{P}_t=(e_t)_{\#}\pi.
\label{EqSuppUnconditionalPath}
\end{equation}

Under the linear observation model considered in this work, the observation depends only on the clean endpoint $b$. The corresponding Gaussian likelihood satisfies:
\begin{equation}
\mathrm{p}(\mathrm{y}\mid a,b)
=
\mathrm{p}(\mathrm{y}\mid b)
\propto
\exp\left(
-\frac{\|\mathrm{H}b-\mathrm{y}\|_2^2}{2\sigma_\mathrm{y}^2}
\right).
\label{EqSuppLikelihood}
\end{equation}

Using this likelihood to perform Bayesian reweighting of the unconditional endpoint coupling gives the measurement-conditioned endpoint coupling:
\begin{equation}
\mathrm d\pi^\mathrm{y}(a,b)
=
\frac{\mathrm{p}(\mathrm{y}\mid b)}{\mathrm{Z_y}}\,
\mathrm d\pi(a,b).
\label{EqSuppPosteriorCoupling}
\end{equation}
The normalization constant is:
\begin{equation}
\mathrm{Z_y}
:=
\int
\mathrm{p}(\mathrm{y}\mid b)\,
\mathrm d\pi(a,b).
\label{EqSuppNormalization}
\end{equation}

The corresponding measurement-conditioned probability path is:
\begin{equation}
\mathrm{P}_t^\mathrm{y}
=
(e_t)_{\#}\pi^\mathrm{y}.
\label{EqSuppPosteriorPath}
\end{equation}

Although the likelihood term explicitly depends only on the clean endpoint $b$, posterior reweighting acts on the joint endpoint pair $(a,b)$. The observation therefore changes not only the relative weights of feasible clean endpoints, but also their posterior pairing with source endpoints. ReBridge-Flow uses the joint PBD objective to construct a measurement-aware endpoint pair around the current state. This jointly accounts for observation consistency, the local flow prior, and endpoint--state bridge compatibility.

\section{Complete Derivation of Closed-Form Posterior Bridge Re-Coupling}

\subsection{Restatement of the PBD Objective}

The Posterior Bridge Defect is defined as:
\begin{equation}
\operatorname{PBD}_t(\mathrm{x}_t,\mathrm{y})
:=
\min_{a,b}
\mathcal J_t(a,b;\mathrm{x}_t,\mathrm{y}).
\label{EqSuppPBD}
\end{equation}

Its joint objective is:
\begin{equation}
\begin{aligned}
\mathcal J_t(a,b;\mathrm{x}_t,\mathrm{y})
={}&
\frac{\|\mathrm{H}b-\mathrm{y}\|_2^2}{2\sigma_\mathrm{y}^2}
+
\frac{\rho}{2}\|b-\hat b_t\|_2^2+
\frac{\lambda}{2}\|a-\hat a_t\|_2^2
+
\frac{\kappa}{2}
\|\mathrm{x}_t-(1-t)a-tb\|_2^2.
\end{aligned}
\label{EqSuppPBDObjective}
\end{equation}

The first term measures the consistency between the candidate clean endpoint $b$ and the observation $\mathrm{y}$. The second and third terms constrain the clean and source endpoints from deviating excessively from their local pseudo-endpoints. The fourth term measures whether the candidate endpoint pair $(a,b)$ can explain the current state $\mathrm{x}_t$ through the linear bridge. Throughout the following derivation, we assume $\sigma_\mathrm{y}^2>0$, $\rho>0$, $\lambda>0$, and $\kappa\geq0$.

\subsection{First-Order Optimality Condition for \texorpdfstring{$a$}{a}}

We first fix an arbitrary clean endpoint $b$ and regard Eq.~\eqref{EqSuppPBDObjective} as a function of the source endpoint $a$. The measurement term and the clean-side prior term are independent of $a$. Therefore:
\begin{equation}
\nabla_a\mathcal J_t
=
\lambda(a-\hat a_t)
-
\kappa(1-t)
\left[
\mathrm{x}_t-(1-t)a-tb
\right].
\label{EqSuppGradientA}
\end{equation}

Setting $\nabla_a\mathcal J_t=0$ gives:
\begin{equation}
\lambda(a-\hat a_t)
-
\kappa(1-t)
\left[
\mathrm{x}_t-(1-t)a-tb
\right]
=0.
\label{EqSuppStationaryA}
\end{equation}

Collecting all terms involving $a$, we obtain:
\begin{equation}
\left[
\lambda+\kappa(1-t)^2
\right]a
=
\lambda\hat a_t
+
\kappa(1-t)(\mathrm{x}_t-tb).
\label{EqSuppLinearA}
\end{equation}

For notational convenience, define:
\begin{equation}
D_t:=\lambda+\kappa(1-t)^2.
\label{EqSuppDt}
\end{equation}

Since $\lambda>0$ and $\kappa\geq0$, we have $D_t>0$ for every $t\in[0,1]$. Therefore, for a fixed $b$, the unique minimizer with respect to $a$ is:
\begin{equation}
a^\star(b)
=
\frac{
\lambda\hat a_t+
\kappa(1-t)(\mathrm{x}_t-tb)
}{
D_t
}.
\label{EqSuppConditionalSource}
\end{equation}

Eq.~\eqref{EqSuppConditionalSource} shows that, for any candidate clean endpoint $b$, the corresponding optimal source endpoint is an analytic solution that balances the source-side flow prior with the current bridge compatibility. Setting $b=\bar b_t$ recovers the source-side re-coupling update used in the main paper.

\subsection{Elimination of the Source Endpoint}

We next substitute $a^\star(b)$ back into the original objective. To simplify the resulting expression, we first use the local pseudo-endpoint identity:
\begin{equation}
\mathrm{x}_t=(1-t)\hat a_t+t\hat b_t.
\label{EqSuppLocalIdentityAgain}
\end{equation}

From Eqs.~\eqref{EqSuppConditionalSource} and~\eqref{EqSuppLocalIdentityAgain}, the displacement of the optimal source endpoint relative to $\hat a_t$ is:
\begin{align}
a^\star(b)-\hat a_t
&=
\frac{
\lambda\hat a_t+
\kappa(1-t)(\mathrm{x}_t-tb)
-D_t\hat a_t
}{
D_t
} \notag\\
&=
\frac{
\kappa(1-t)
\left[
\mathrm{x}_t-(1-t)\hat a_t-tb
\right]
}{
D_t
} \notag\\
&=
\frac{
\kappa t(1-t)
}{
D_t
}
(\hat b_t-b).
\label{EqSuppSourceDeviation}
\end{align}

Therefore, for a fixed $b$, the source-side prior term becomes:
\begin{equation}
\frac{\lambda}{2}
\|a^\star(b)-\hat a_t\|_2^2
=
\frac{
\lambda\kappa^2t^2(1-t)^2
}{
2D_t^2
}
\|\hat b_t-b\|_2^2.
\label{EqSuppSourcePriorReduced}
\end{equation}

We then compute the bridge residual after substituting $a^\star(b)$. Using Eq.~\eqref{EqSuppSourceDeviation} gives:
\begin{align}
\mathrm{x}_t-e_t(a^\star(b),b)
&=
(1-t)\hat a_t+t\hat b_t
-(1-t)a^\star(b)-tb \notag\\
&=
(1-t)\left[\hat a_t-a^\star(b)\right]
+t(\hat b_t-b) \notag\\
&=
-\frac{\kappa t(1-t)^2}{D_t}
(\hat b_t-b)
+t(\hat b_t-b) \notag\\
&=
\frac{\lambda t}{D_t}
(\hat b_t-b).
\label{EqSuppConditionalBridgeResidual}
\end{align}

The corresponding bridge-residual term is:
\begin{equation}
\frac{\kappa}{2}
\|\mathrm{x}_t-e_t(a^\star(b),b)\|_2^2
=
\frac{
\kappa\lambda^2t^2
}{
2D_t^2
}
\|\hat b_t-b\|_2^2.
\label{EqSuppBridgeTermReduced}
\end{equation}

Adding Eqs.~\eqref{EqSuppSourcePriorReduced} and~\eqref{EqSuppBridgeTermReduced} gives:
\begin{align}
&
\frac{\lambda}{2}
\|a^\star(b)-\hat a_t\|_2^2
+
\frac{\kappa}{2}
\|\mathrm{x}_t-e_t(a^\star(b),b)\|_2^2 \notag\\
&=
\frac{
\lambda\kappa^2t^2(1-t)^2
+
\kappa\lambda^2t^2
}{
2D_t^2
}
\|\hat b_t-b\|_2^2 \notag\\
&=
\frac{
\kappa\lambda t^2
\left[
\kappa(1-t)^2+\lambda
\right]
}{
2D_t^2
}
\|\hat b_t-b\|_2^2 \notag\\
&=
\frac{\kappa\lambda t^2}{2D_t}
\|b-\hat b_t\|_2^2.
\label{EqSuppCombinedReducedTerms}
\end{align}

After eliminating the source endpoint, the reduced objective depending only on $b$ is:
\begin{equation}
\widetilde{\mathcal J}_t(b)
=
\frac{\|\mathrm{H}b-\mathrm{y}\|_2^2}{2\sigma_\mathrm{y}^2}
+
\frac{1}{2}
\left[
\rho+
\frac{\kappa\lambda t^2}
{\lambda+\kappa(1-t)^2}
\right]
\|b-\hat b_t\|_2^2.
\label{EqSuppReducedObjective}
\end{equation}

The second term in the reduced objective contains both the original clean-side prior weight $\rho$ and an additional weight jointly induced by source-side re-coupling and the bridge residual.

\subsection{Normal Equation for the Clean Endpoint}

Define:
\begin{equation}
\gamma_t
:=
\sigma_\mathrm{y}^2
\left[
\rho+
\frac{\kappa\lambda t^2}
{\lambda+\kappa(1-t)^2}
\right].
\label{EqSuppGamma}
\end{equation}

Because $\sigma_\mathrm{y}^2>0$, $\rho>0$, and $D_t>0$, we have $\gamma_t>0$. Using Eq.~\eqref{EqSuppGamma}, the gradient of the reduced objective with respect to $b$ is:
\begin{equation}
\nabla_b\widetilde{\mathcal J}_t(b)
=
\frac{1}{\sigma_\mathrm{y}^2}
\mathrm{H}^{\top}(\mathrm{H}b-\mathrm{y})
+
\frac{\gamma_t}{\sigma_\mathrm{y}^2}
(b-\hat b_t).
\label{EqSuppGradientB}
\end{equation}

Setting the gradient to zero and multiplying by $\sigma_\mathrm{y}^2$ yields the normal equation:
\begin{equation}
\mathrm{H}^{\top}(\mathrm{H}\bar b_t-\mathrm{y})
+
\gamma_t(\bar b_t-\hat b_t)
=0.
\label{EqSuppNormalEquationB}
\end{equation}

Rearranging Eq.~\eqref{EqSuppNormalEquationB} gives:
\begin{equation}
\left(
\mathrm{H}^{\top}\mathrm{H}
+
\gamma_t\mathrm{I}
\right)
(\bar b_t-\hat b_t)
=
\mathrm{H}^{\top}
(\mathrm{y}-\mathrm{H}\hat b_t).
\label{EqSuppNormalEquationShifted}
\end{equation}

Since $\gamma_t>0$, the matrix $\mathrm{H}^{\top}\mathrm{H}+\gamma_t\mathrm{I}$ is positive definite. Indeed, for any nonzero vector $z\in\mathbb{R}^{\mathrm{n}}$:
\begin{equation}
z^{\top}
\left(
\mathrm{H}^{\top}\mathrm{H}
+
\gamma_t\mathrm{I}
\right)z
=
\|\mathrm{H}z\|_2^2
+
\gamma_t\|z\|_2^2
>0.
\label{EqSuppPositiveDefiniteClean}
\end{equation}

The matrix is therefore invertible, and:
\begin{equation}
\bar b_t
=
\hat b_t
+
\left(
\mathrm{H}^{\top}\mathrm{H}
+
\gamma_t\mathrm{I}
\right)^{-1}
\mathrm{H}^{\top}
(\mathrm{y}-\mathrm{H}\hat b_t).
\label{EqSuppCleanUpdatePrimal}
\end{equation}

\subsection{Push-Through Identity}

To obtain the measurement-space expression used in the main paper, we use the following push-through identity:
\begin{equation}
\left(
\mathrm{H}^{\top}\mathrm{H}
+
\gamma_t\mathrm{I}
\right)^{-1}
\mathrm{H}^{\top}
=
\mathrm{H}^{\top}
\left(
\mathrm{H}\mathrm{H}^{\top}
+
\gamma_t\mathrm{I}
\right)^{-1}.
\label{EqSuppPushThrough}
\end{equation}

This identity follows directly from:
\begin{equation}
\left(
\mathrm{H}^{\top}\mathrm{H}
+
\gamma_t\mathrm{I}
\right)
\mathrm{H}^{\top}
=
\mathrm{H}^{\top}
\left(
\mathrm{H}\mathrm{H}^{\top}
+
\gamma_t\mathrm{I}
\right).
\label{EqSuppPushThroughBase}
\end{equation}

Since $\gamma_t>0$, both $\mathrm{H}^{\top}\mathrm{H}+\gamma_t\mathrm{I}$ and $\mathrm{H}\mathrm{H}^{\top}+\gamma_t\mathrm{I}$ are positive definite. Left-multiplying Eq.~\eqref{EqSuppPushThroughBase} by $\left(\mathrm{H}^{\top}\mathrm{H}+\gamma_t\mathrm{I}\right)^{-1}$ and right-multiplying it by $\left(\mathrm{H}\mathrm{H}^{\top}+\gamma_t\mathrm{I}\right)^{-1}$ gives Eq.~\eqref{EqSuppPushThrough}. The dimensions of the two identity matrices are determined by context: the former is $\mathrm{n}\times\mathrm{n}$, whereas the latter is $\mathrm{m}\times\mathrm{m}$.

Substituting Eq.~\eqref{EqSuppPushThrough} into Eq.~\eqref{EqSuppCleanUpdatePrimal} gives the final posterior clean-side anchoring update:
\begin{equation}
\bar b_t
=
\hat b_t
+
\mathrm{H}^{\top}
\left(
\mathrm{H}\mathrm{H}^{\top}
+
\gamma_t\mathrm{I}
\right)^{-1}
\left(
\mathrm{y}-\mathrm{H}\hat b_t
\right).
\label{EqSuppCleanUpdate}
\end{equation}

Eq.~\eqref{EqSuppCleanUpdate} is identical to Eq.~\eqref{Eq7}. This expression only requires solving a regularized linear system in the measurement space.

\subsection{Recovery of the Source Endpoint}

After obtaining $\bar b_t$, substituting $b=\bar b_t$ into Eq.~\eqref{EqSuppConditionalSource} recovers the corresponding optimal source endpoint:
\begin{equation}
\bar a_t
=
\frac{
\lambda\hat a_t+
\kappa(1-t)
\bigl(
\mathrm{x}_t-t\bar b_t
\bigr)
}{
\lambda+\kappa(1-t)^2
}.
\label{EqSuppSourceUpdate}
\end{equation}

Eq.~\eqref{EqSuppSourceUpdate} is identical to Eq.~\eqref{Eq8}. Using Eq.~\eqref{EqSuppSourceDeviation}, the source-side update can also be written explicitly relative to the original local source endpoint:
\begin{equation}
\bar a_t-\hat a_t
=
\frac{\kappa t(1-t)}
{\lambda+\kappa(1-t)^2}
(\hat b_t-\bar b_t).
\label{EqSuppSourceResponse}
\end{equation}

Equivalently:
\begin{equation}
\bar a_t-\hat a_t
=
-
\frac{\kappa t(1-t)}
{\lambda+\kappa(1-t)^2}
(\bar b_t-\hat b_t).
\label{EqSuppSourceResponseOpposite}
\end{equation}

Thus, when clean-side anchoring moves the clean endpoint from $\hat b_t$ to $\bar b_t$, the optimal source endpoint is adjusted synchronously in the opposite direction. The adjustment magnitude is jointly determined by $t$, $\lambda$, and $\kappa$. A larger $\lambda$ keeps the source endpoint closer to $\hat a_t$, whereas a larger $\kappa$ makes the source endpoint compensate more strongly for the bridge mismatch introduced by clean-side anchoring.

\subsection{Summary of the Joint Minimization}

The derivation above first performs exact conditional minimization over the source endpoint, and then minimizes the reduced clean-side objective. Specifically:
\begin{equation}
a^\star(b)
=
\operatorname*{arg\,min}_{a}
\mathcal J_t(a,b;\mathrm{x}_t,\mathrm{y}),
\quad
\bar b_t
=
\operatorname*{arg\,min}_{b}
\widetilde{\mathcal J}_t(b),
\quad
\bar a_t=a^\star(\bar b_t).
\label{EqSuppNestedMinimization}
\end{equation}

Therefore, for any candidate endpoint pair $(a,b)$:
\begin{equation}
\mathcal J_t(a,b;\mathrm{x}_t,\mathrm{y})
\geq
\mathcal J_t(a^\star(b),b;\mathrm{x}_t,\mathrm{y})
=
\widetilde{\mathcal J}_t(b)
\geq
\widetilde{\mathcal J}_t(\bar b_t)
=
\mathcal J_t(\bar a_t,\bar b_t;\mathrm{x}_t,\mathrm{y}).
\label{EqSuppNestedInequality}
\end{equation}

This shows that clean-side anchoring and source-side re-coupling are not two independently designed correction steps. The clean endpoint $\bar b_t$ is the minimizer of the reduced objective obtained after eliminating the source endpoint, while $\bar a_t$ is the conditionally optimal source endpoint associated with this optimal clean endpoint. Together, they form the global minimizer of the original joint PBD objective. The next section further proves that this minimizer is unique and establishes the quadratic-growth lower bound in Eq.~\eqref{Eq10}.

\section{Proof of Proposition 1: Unique Minimizer of the PBD Objective}

\begin{proof}
In this proof, we fix $t$, $\mathrm{x}_t$, $\mathrm{y}$, $\hat a_t$, and $\hat b_t$. We first prove that $\mathcal J_t$ is $\mu$-strongly convex in the joint variable $(a,b)$. We then use the closed-form derivation above to show that Eqs.~\eqref{Eq7} and~\eqref{Eq8} give its unique global minimizer. Finally, we establish the objective-gap lower bound in Eq.~\eqref{Eq10}.

From Eq.~\eqref{EqSuppPBDObjective}, the Hessian of $\mathcal J_t$ with respect to $(a,b)$ is:
\begin{equation}
\nabla_{(a,b)}^2\mathcal J_t
=
\begin{bmatrix}
\left[\lambda+\kappa(1-t)^2\right]\mathrm{I}
&
\kappa t(1-t)\mathrm{I}
\\
\kappa t(1-t)\mathrm{I}
&
\displaystyle
\frac{1}{\sigma_\mathrm{y}^2}
\mathrm{H}^{\top}\mathrm{H}
+
\left(\rho+\kappa t^2\right)\mathrm{I}
\end{bmatrix}.
\label{EqSuppHessian}
\end{equation}

Let $\delta a,\delta b\in\mathbb{R}^{\mathrm{n}}$ be arbitrary perturbations. The quadratic form induced by this Hessian is:
\begin{equation}
\label{EqSuppHessianQuadraticForm}
\left\langle
\begin{bmatrix}
\delta a\\
\delta b
\end{bmatrix},
\nabla_{(a,b)}^2\mathcal J_t
\begin{bmatrix}
\delta a\\
\delta b
\end{bmatrix}
\right\rangle =
\lambda\|\delta a\|_2^2
+
\rho\|\delta b\|_2^2
+
\frac{1}{\sigma_\mathrm{y}^2}
\|\mathrm{H}\delta b\|_2^2
+
\kappa
\|(1-t)\delta a+t\delta b\|_2^2.
\end{equation}
The measurement and bridge-residual terms are both nonnegative. Therefore:
\begin{align}
&
\left\langle
\begin{bmatrix}
\delta a\\
\delta b
\end{bmatrix},
\nabla_{(a,b)}^2\mathcal J_t
\begin{bmatrix}
\delta a\\
\delta b
\end{bmatrix}
\right\rangle \notag\geq
\lambda\|\delta a\|_2^2
+
\rho\|\delta b\|_2^2 \notag\geq
\mu
\left(
\|\delta a\|_2^2+
\|\delta b\|_2^2
\right),
\label{EqSuppStrongConvexityBound}
\end{align}
where:
\begin{equation}
\mu:=\min\{\lambda,\rho\}>0.
\label{EqSuppMu}
\end{equation}

Thus:
\begin{equation}
\nabla_{(a,b)}^2\mathcal J_t
\succeq
\mu\mathrm{I}.
\label{EqSuppHessianLowerBound}
\end{equation}

This shows that $\mathcal J_t$ is $\mu$-strongly convex in the joint variable $(a,b)$. In particular, its Hessian is positive definite. Hence, $\mathcal J_t$ is a coercive, strictly convex quadratic function. It has at most one stationary point, and any stationary point must be its unique global minimizer.

The preceding derivation shows that, for any fixed $b$, Eq.~\eqref{EqSuppConditionalSource} gives the unique conditional minimizer $a^\star(b)$ with respect to $a$. Substituting this conditional minimizer into the original objective gives the reduced objective $\widetilde{\mathcal J}_t(b)$. Eq.~\eqref{EqSuppCleanUpdate} gives the unique minimizer $\bar b_t$ of this reduced objective, while Eq.~\eqref{EqSuppSourceUpdate} satisfies $\bar a_t=a^\star(\bar b_t)$.

Therefore, for any $(a,b)$, Eq.~\eqref{EqSuppNestedInequality} gives:
\begin{equation}
\mathcal J_t(a,b;\mathrm{x}_t,\mathrm{y})
\geq
\mathcal J_t(\bar a_t,\bar b_t;\mathrm{x}_t,\mathrm{y}).
\label{EqSuppGlobalOptimality}
\end{equation}

Hence:
\begin{equation}
(\bar a_t,\bar b_t)
=
\operatorname*{arg\,min}_{a,b}
\mathcal J_t(a,b;\mathrm{x}_t,\mathrm{y}).
\label{EqSuppArgmin}
\end{equation}

Together with the strong convexity established in the Hessian, this shows that the global minimizer is unique.

Define:
\begin{equation}
\delta a:=a-\bar a_t,
\quad
\delta b:=b-\bar b_t.
\label{EqSuppDeltaAB}
\end{equation}

Because $\mathcal J_t$ is a quadratic function with a constant Hessian and:
\begin{equation}
\nabla_{(a,b)}
\mathcal J_t(\bar a_t,\bar b_t;\mathrm{x}_t,\mathrm{y})
=0,
\label{EqSuppZeroGradient}
\end{equation}
its second-order expansion around $(\bar a_t,\bar b_t)$ is exact. Therefore:
\begin{align}
&
\mathcal J_t(a,b;\mathrm{x}_t,\mathrm{y})
-
\mathcal J_t(\bar a_t,\bar b_t;\mathrm{x}_t,\mathrm{y}) \notag=
\frac{1}{2\sigma_\mathrm{y}^2}
\|\mathrm{H}\delta b\|_2^2
+
\frac{\rho}{2}\|\delta b\|_2^2
+
\frac{\lambda}{2}\|\delta a\|_2^2
+
\frac{\kappa}{2}
\|(1-t)\delta a+t\delta b\|_2^2.
\label{EqSuppExactObjectiveGap}
\end{align}

Dropping the two nonnegative measurement and bridge-residual terms gives:
\begin{equation}
\label{EqSuppObjectiveGapBound}
\mathcal J_t(a,b;\mathrm{x}_t,\mathrm{y})
-
\mathcal J_t(\bar a_t,\bar b_t;\mathrm{x}_t,\mathrm{y}) \geq
\frac{\lambda}{2}\|\delta a\|_2^2
+
\frac{\rho}{2}\|\delta b\|_2^2 \geq
\frac{\mu}{2}
\left(
\|\delta a\|_2^2+
\|\delta b\|_2^2
\right).
\end{equation}

Substituting Eq.~\eqref{EqSuppDeltaAB} into the inequality above gives:
\begin{equation}
\label{EqSuppFinalObjectiveGap}
\mathcal J_t(a,b;\mathrm{x}_t,\mathrm{y})
-
\mathcal J_t(\bar a_t,\bar b_t;\mathrm{x}_t,\mathrm{y})
\geq
\frac{\mu}{2}
\left(
\|a-\bar a_t\|_2^2
+
\|b-\bar b_t\|_2^2
\right).
\end{equation}

Therefore, Eqs.~\eqref{Eq7} and~\eqref{Eq8} jointly give the unique global minimizer of the PBD objective. Moreover, the objective value of any candidate pair away from this endpoint pair increases by at least a quadratic amount in its joint endpoint distance.
\end{proof}

\section{Proof of Proposition 2: Exact Contraction of the Bridge Residual}

\begin{proof}
Define the bridge residual after clean-side-only correction as:
\begin{equation}
\mathrm{r}_t^{\mathrm{CS}}
:=
\mathrm{x}_t-e_t(\hat a_t,\bar b_t),
\label{EqSuppCSResidualDefinitionProof}
\end{equation}
and the bridge residual after source-side re-coupling as:
\begin{equation}
\mathrm{r}_t^{\mathrm{RC}}
:=
\mathrm{x}_t-e_t(\bar a_t,\bar b_t).
\label{EqSuppRCResidualDefinition}
\end{equation}

Again, let:
\begin{equation}
D_t:=\lambda+\kappa(1-t)^2.
\label{EqSuppDtProofTwo}
\end{equation}

Using the local pseudo-endpoint identity $\mathrm{x}_t=e_t(\hat a_t,\hat b_t)$, we have:
\begin{align}
\mathrm{r}_t^{\mathrm{CS}}
&=
e_t(\hat a_t,\hat b_t)
-
e_t(\hat a_t,\bar b_t) \notag\\
&=
t(\hat b_t-\bar b_t).
\label{EqSuppCSResidualProof}
\end{align}

From Eq.~\eqref{EqSuppSourceResponse}, source-side re-coupling satisfies:
\begin{equation}
\bar a_t-\hat a_t
=
\frac{\kappa t(1-t)}{D_t}
(\hat b_t-\bar b_t).
\label{EqSuppSourceResponseProof}
\end{equation}

Therefore:
\begin{equation}
\hat a_t-\bar a_t
=
-
\frac{\kappa t(1-t)}{D_t}
(\hat b_t-\bar b_t).
\label{EqSuppOppositeSourceResponseProof}
\end{equation}

Using the linear interpolation relations for the two endpoint pairs:
\begin{align}
\mathrm{r}_t^{\mathrm{RC}}
&=
e_t(\hat a_t,\hat b_t)
-
e_t(\bar a_t,\bar b_t) \notag\\
&=
(1-t)(\hat a_t-\bar a_t)
+
t(\hat b_t-\bar b_t).
\label{EqSuppRCResidualExpansion}
\end{align}

Substituting Eq.~\eqref{EqSuppOppositeSourceResponseProof} into Eq.~\eqref{EqSuppRCResidualExpansion} gives:
\begin{align}
\mathrm{r}_t^{\mathrm{RC}}
&=
-\frac{\kappa t(1-t)^2}{D_t}
(\hat b_t-\bar b_t)
+
t(\hat b_t-\bar b_t) \notag\\
&=
\left[
1-
\frac{\kappa(1-t)^2}{D_t}
\right]
t(\hat b_t-\bar b_t) \notag\\
&=
\frac{
D_t-\kappa(1-t)^2
}{
D_t
}
t(\hat b_t-\bar b_t) \notag\\
&=
\frac{\lambda}{D_t}
t(\hat b_t-\bar b_t).
\label{EqSuppRCResidualReduced}
\end{align}

Define:
\begin{equation}
\eta_t
:=
\frac{\lambda}
{\lambda+\kappa(1-t)^2}
=
\frac{\lambda}{D_t}.
\label{EqSuppEta}
\end{equation}

Combining this definition with Eq.~\eqref{EqSuppCSResidualProof} gives the exact vector relation:
\begin{equation}
\mathrm{r}_t^{\mathrm{RC}}
=
\eta_t\mathrm{r}_t^{\mathrm{CS}}.
\label{EqSuppExactResidualContraction}
\end{equation}

Since $\lambda>0$, $\kappa\geq0$, and $t\in[0,1]$, we have:
\begin{equation}
0<\eta_t\leq1.
\label{EqSuppEtaRange}
\end{equation}

Taking the $\ell_2$-norm of Eq.~\eqref{EqSuppExactResidualContraction} gives:
\begin{align}
\left\|\mathrm{r}_t^{\mathrm{RC}}\right\|_2
&=
\eta_t
\left\|\mathrm{r}_t^{\mathrm{CS}}\right\|_2 \notag\\
&=
\eta_t t
\left\|\hat b_t-\bar b_t\right\|_2 \notag\\
&\leq
\left\|\mathrm{r}_t^{\mathrm{CS}}\right\|_2.
\label{EqSuppResidualNormContraction}
\end{align}

When $\kappa=0$, we have $\eta_t=1$. The source endpoint is not re-coupled, and the residual is not contracted. When $t=1$, the linear bridge is determined entirely by the clean endpoint, so changing the source endpoint cannot reduce the current bridge residual; again, $\eta_t=1$. When $t=0$, $\mathrm{r}_t^{\mathrm{CS}}=0$, and both residuals are zero.

When $\kappa>0$, $t\in(0,1)$, and $\bar b_t\neq\hat b_t$, we have $\eta_t<1$ and $\|\mathrm{r}_t^{\mathrm{CS}}\|_2>0$. Therefore:
\begin{equation}
\left\|\mathrm{r}_t^{\mathrm{RC}}\right\|_2
<
\left\|\mathrm{r}_t^{\mathrm{CS}}\right\|_2.
\label{EqSuppStrictResidualContraction}
\end{equation}

Thus, source-side re-coupling provides an exact multiplicative contraction of the local endpoint-state consistency residual introduced by clean-side anchoring. This conclusion concerns the local bridge residual at the current sampling time. It does not additionally assume or claim that the global reconstruction error contracts by the same factor.
\end{proof}

\section{Dataset Descriptions and Detailed Experimental Details}

\subsection{Dataset Descriptions}

\paragraph{CelebA.}
CelebA~\cite{Celeba} is a large-scale face image dataset containing celebrity face images with diverse identities, poses, expressions, and appearance attributes. Following the experimental protocols of PnP-Flow and Restora-Flow, we resize all images to a resolution of $128\times128$. For evaluation, we use the same 100 test images as PnP-Flow and Restora-Flow.

\paragraph{AFHQ-Cat.}
AFHQ-Cat~\cite{AFHQ} is the cat subset of the Animal Faces-HQ dataset. It contains approximately 5,000 high-quality training images of cat faces. The images cover diverse cat breeds, poses, fur colors, textures, backgrounds, and lighting conditions. All images are resized to a resolution of $256\times256$. Following the setup of PnP-Flow, we evaluate the model on the same 100 test images used in the related work.

\paragraph{COCO.}
COCO~\cite{COCO} is a large-scale image dataset containing multiple object categories and complex natural scenes. Its training set contains approximately 118,000 images covering a wide range of visual categories, including people, animals, vehicles, indoor objects, and outdoor environments. Following the setup of Restora-Flow, the training images are resized to $128\times128$. A fixed set of 100 images is selected from the COCO validation set for image restoration experiments. We directly use the COCO-pretrained Flow Matching model provided by Restora-Flow.

\paragraph{IXI-Brain.}
IXI-Brain~\cite{IXI} is a publicly available brain magnetic resonance imaging dataset containing scans from nearly 600 healthy subjects collected at three hospitals in London. The dataset provides multiple MRI modalities, including T1-weighted, T2-weighted, and proton-density-weighted scans. We use the T2-weighted brain MRI data to train and evaluate the generative prior. The original 3D volumes have an approximate voxel size of $0.94\times0.94\times1.2\,\mathrm{mm}^{3}$ and a matrix size of approximately $256\times256\times n$, where $n$ denotes the number of axial slices. Since we employ a 2D Flow Matching model, we extract axial 2D slices from the 3D volumes and discard boundary slices containing little anatomical information. We use 20,000 2D slices to train the Flow Matching model and reserve 100 slices as a candidate test set. All slices are processed to a resolution of $256\times256$ and normalized to $[-1,1]$.

\paragraph{PMUB.}
PMUB~\cite{PMUB}, formally known as the Prostate-MRI-US-Biopsy dataset, is a publicly available medical imaging dataset for prostate cancer diagnosis and biopsy research. It provides T2-weighted prostate MRI scans, together with multiparametric MRI sequences such as T1-weighted, diffusion-weighted, and dynamic contrast-enhanced imaging. We use the T2-weighted MRI data to train a Flow Matching-based generative prior for prostate images. The PMUB images have an in-plane spatial resolution of $0.547\,\mathrm{mm}\times0.547\,\mathrm{mm}$ and an inter-slice spacing of $1.5\,\mathrm{mm}$. Following the preprocessing protocol of DDGF, we extract axial 2D slices from the 3D MRI volumes and discard boundary slices without sufficient anatomical information. The training set contains 7,219 slices from 120 MRI volumes. The independent test set contains 100 slices from an additional 10 MRI volumes. All images are resized to $256\times256$ and normalized to $[-1,1]$.

\paragraph{X-Ray Hand.}
X-Ray Hand~\cite{Hand1,Hand2} is a medical imaging dataset composed of hand radiographs. It contains 895 hand X-ray images in total. Following the setup of Restora-Flow, all images are resized to $256\times256$. We use 100 images as the test set, matching the evaluation set size used for the other datasets. For this dataset, we directly use the pretrained Flow Matching model provided by Restora-Flow without additional training.

\subsection{Model Sources}

The six Flow Matching generative priors used in this work consist of publicly available pretrained models and models trained by us. For CelebA and AFHQ-Cat, we use the pretrained Flow Matching models released by PnP-Flow~\cite{PnPFlow}. These models use a standard Gaussian source distribution, employ a U-Net as the velocity network, and are trained with Mini-Batch Optimal Transport Flow Matching. The CelebA model is trained for 200 epochs with a learning rate of $1\times10^{-4}$ and a batch size of 128. The AFHQ-Cat model uses the same learning rate and is trained for 400 epochs with a batch size of 64.

For COCO and X-Ray Hand, we use the pretrained models released by Restora-Flow~\cite{RestoraFlow}. The COCO Flow Matching model is trained for 300 epochs with a learning rate of $1\times10^{-4}$ and a batch size of 64.

For IXI-Brain and PMUB, we separately train unconditional Flow Matching models from scratch and parameterize the velocity field $\mathrm{v}_\theta(t,\mathrm{x}_t)$ with a U-Net. During training, a clean endpoint $b$ is sampled from the corresponding dataset, a source endpoint $a\sim{\mathrm{N}}(0,\mathrm{I})$ is sampled from the standard Gaussian distribution, and $t\sim\mathcal{U}[0,1]$ is sampled independently. The intermediate state is then constructed along the linear path $\mathrm{x}_t=(1-t)a+tb$, with $b-a$ used as the target velocity. The models are trained by minimizing the Conditional Flow Matching loss:
\begin{equation}
\mathcal{L}_{\mathrm{FM}}(\theta)=\mathbb{E}_{t,(a,b)\sim\pi}\left[\frac{1}{2}\left\|\mathrm{v}_\theta(t,\mathrm{x}_t)-(b-a)\right\|_2^2\right].
\label{EqSuppFMLoss}
\end{equation}
All training images are resized to $256\times256$ and normalized to $[-1,1]$. Both models are trained for 400 epochs with a learning rate of $1\times10^{-4}$ and a batch size of 64. Training is performed on a single NVIDIA A6000 GPU with 48GB of memory.

\subsection{Restoration Task Settings}

For natural images, we consider five restoration tasks: Gaussian denoising, Gaussian deblurring, super-resolution, random inpainting, and box inpainting. Gaussian denoising uses a noise level of $\sigma_\mathrm{y}=0.2$. For random inpainting, 70\% of the pixels are removed. CelebA and COCO are evaluated on $2\times$ super-resolution and box inpainting with a centered $40\times40$ mask. AFHQ-Cat is evaluated on $4\times$ super-resolution and box inpainting with a centered $80\times80$ mask. For Gaussian deblurring, the additional measurement-noise level is set to $\sigma_\mathrm{y}=0.05$. For the other natural-image inverse tasks, the additional noise level is set to $\sigma_\mathrm{y}=0.01$.

For medical images, we evaluate Gaussian denoising, $2\times$ super-resolution, random inpainting, and box inpainting. Gaussian denoising uses a noise level of $\sigma_\mathrm{y}=0.08$. For random inpainting, 30\% of the pixels are removed. IXI-Brain and X-Ray Hand use a centered $32\times32$ mask for box inpainting, whereas PMUB uses a centered $60\times60$ mask. Except for denoising, Gaussian measurement noise with a standard deviation of $\sigma_\mathrm{y}=0.01$ is added to all medical image restoration tasks.

\begin{table}[!htbp]
\centering
\renewcommand{\arraystretch}{1.2}
\setlength{\tabcolsep}{4pt}
\caption{Mechanism-level comparison of Flow Matching-based image restoration methods according to the variables explicitly corrected during inference. A check mark indicates a direct correction or optimization of the corresponding object.}
\label{Table_Mechanism_Comparison}
\resizebox{\textwidth}{!}{%
\begin{tabular}{@{}l l c c c c c c@{}}
\toprule
\multirow{3.5}{*}{Method} & 
\multirow{3.5}{*}{\begin{tabular}{@{}l@{}}Primary Explicit Correction\end{tabular}} & 
\multicolumn{6}{c}{Explicitly Corrected Object or Relation} \\
\cmidrule(lr){3-8}
& & 
\begin{tabular}{@{}c@{}}Velocity \\ Field\end{tabular} & 
\begin{tabular}{@{}c@{}}Source \\ Variable\end{tabular} & 
\begin{tabular}{@{}c@{}}State / Local \\ Trajectory\end{tabular} & 
\begin{tabular}{@{}c@{}}Clean / Destination \\ Endpoint\end{tabular} & 
\begin{tabular}{@{}c@{}}Joint Source-Clean \\ Endpoint Pair\end{tabular} & 
\begin{tabular}{@{}c@{}}Explicit Pairwise \\ Bridge Residual\end{tabular} \\
\midrule
OT-ODE~\cite{OT-ODE}          & ODE Vector field                        & $\checkmark$ & $\times$     & $\times$     & $\times$     & $\times$     & $\times$     \\
Flow-Priors~\cite{FlowPriors} & Intermediate State via Local MAP        & $\times$     & $\times$     & $\checkmark$ & $\times$     & $\times$     & $\times$     \\
D-Flow~\cite{D-Flow}          & Initial Source Point                    & $\times$     & $\checkmark$ & $\times$     & $\times$     & $\times$     & $\times$     \\
PnP-Flow~\cite{PnPFlow}       & State with Flow-Path Reprojection       & $\times$     & $\times$     & $\checkmark$ & $\times$     & $\times$     & $\times$     \\
Restora-Flow~\cite{RestoraFlow}& Mask-Guided Trajectory State           & $\times$     & $\times$     & $\checkmark$ & $\times$     & $\times$     & $\times$     \\
Flower~\cite{Flower}          & Clean Destination and Reprojected State & $\times$     & $\times$     & $\checkmark$ & $\checkmark$ & $\times$     & $\times$     \\
\midrule
ReBridge-Flow                 & Source-Clean Endpoint Pair              & $\times$     & $\checkmark$ & $\times$     & $\checkmark$ & $\checkmark$ & $\checkmark$ \\
\bottomrule
\end{tabular}%
}
\end{table}

\subsection{Comparison Methods}
As shown in Table~\ref{Table_Mechanism_Comparison}, we select six representative Flow Matching-based image restoration methods for comparison. These methods cover different mechanisms, including velocity-field correction, local-trajectory optimization, source-variable optimization, intermediate-state correction, and clean-endpoint updating. For reproduction, we prioritize the official implementations and recommended configurations of each method. For CelebA and AFHQ-Cat, we directly use the officially provided parameter settings. For COCO, IXI-Brain, PMUB, and X-Ray Hand, we evaluate the official configurations used for CelebA and AFHQ-Cat on the validation set and select the configuration with the best validation performance. We also report the reproduction hyperparameters of all methods across different datasets and tasks in Tables~\ref{Param2}, ~\ref{Param3}, ~\ref{Param4}, ~\ref{Param5}, ~\ref{Param6}, and ~\ref{Param7}. All methods are evaluated using the same test images, degradation settings, evaluation metrics, and hardware environment.

\paragraph{OT-ODE.}
OT-ODE~\cite{OT-ODE} directly injects the measurement-consistency gradient into the ODE dynamics, thereby correcting the pretrained velocity field and the local transport direction. We use its official implementation and determine the parameters for each dataset and degradation task according to the unified protocol described above.

\paragraph{Flow-Priors.}
Flow-Priors~\cite{FlowPriors} decomposes the global inverse problem into a sequence of local-trajectory optimization problems. During sampling, it alternately incorporates the measurement constraint and the pretrained flow prior. We retain its official local optimization procedure and use the recommended initialization and parameter configurations.

\paragraph{D-Flow.}
D-Flow~\cite{D-Flow} optimizes the initial source variable by backpropagating through the complete Flow ODE, such that the generated result satisfies the given observation. We use the official source-point optimization procedure and preserve its gradient computation and iterative optimization scheme.

\paragraph{PnP-Flow.}
PnP-Flow~\cite{PnPFlow} alternates between data-consistency updates and flow-prior mappings to progressively constrain the current restoration state. We reproduce it using the official code and recommended settings. We also set the number of sampling steps for our method to $K=100$, consistent with the experimental setting of PnP-Flow.

\paragraph{Restora-Flow.}
Restora-Flow~\cite{RestoraFlow} uses mask guidance and trajectory correction to constrain intermediate states and reduce the inconsistency between the restoration trajectory and the degraded observation. We adopt its official implementation and task configurations. The pretrained Flow Matching models used for the COCO and X-Ray Hand experiments are also provided by Restora-Flow.

\paragraph{Flower.}
Flower~\cite{Flower} estimates a clean endpoint that is consistent with the pretrained flow from the current state. It then corrects this endpoint using the measurement information and uses the updated result for subsequent sampling. We follow its official endpoint estimation and measurement-correction procedures and apply the same validation-based parameter-selection protocol as for the other baselines.

\section{More Experiments}

\subsection{Effect of the Number of Sampling Steps}

As shown in Table~\ref{Table_Sampling_Steps}, all three evaluation metrics improve consistently as the number of sampling steps increases from $K=10$ to $K=100$. In particular, $K=100$ achieves the highest PSNR of $28.16$\,dB and the highest SSIM of $0.820$. When the number of sampling steps is further increased to $K=200$, LPIPS reaches its lowest value of $0.116$.

\begin{table}[!ht]
\centering
\caption{Effect of the number of sampling steps on $4\times$ SR using the AFHQ-Cat dataset. The \colorbox{best1}{best} results are highlighted.}
\label{Table_Sampling_Steps}

\small 
\renewcommand{\arraystretch}{1.4} 
\setlength{\tabcolsep}{25pt} 
\setlength{\aboverulesep}{0pt}
\setlength{\belowrulesep}{0pt}

\resizebox{0.6\columnwidth}{!}{%
\begin{tabular}{@{}c c c c@{}}
\toprule
$K$ & PSNR$\uparrow$ & SSIM$\uparrow$ & LPIPS$\downarrow$ \\
\midrule
10  & 26.91 & 0.782 & 0.158 \\
20  & 27.48 & 0.798 & 0.141 \\
50  & 27.94 & 0.812 & 0.127 \\
100 & \colorbox{best1}{28.16} & \colorbox{best1}{0.820} & 0.119 \\
200 & 28.14 & 0.811 & \colorbox{best1}{0.116} \\
\bottomrule
\end{tabular}%
}
\end{table}

\subsection{Sensitivity to the Bridge Re-Coupling Strength \texorpdfstring{$\kappa$}{kappa}}
We analyze the effect of the bridge re-coupling strength $\kappa$ on the CelebA $2\times$ super-resolution task while fixing $\rho=\lambda=1$. As shown in Figure~\ref{Figure_Kappa}, when $\kappa=0$, bridge re-coupling is disabled, resulting in a PSNR of $33.72$\,dB, an SSIM of $0.950$, and an LPIPS of $0.021$. As $\kappa$ increases, the restoration performance generally improves. For example, at $\kappa=0.5$, PSNR increases to $33.98$\,dB, although LPIPS temporarily rises to $0.024$. When $\kappa$ is increased to $1$ and $2$, PSNR reaches $34.01$\,dB and $34.12$\,dB, respectively, while SSIM improves from $0.956$ to $0.959$. At $\kappa=5$, all three metrics achieve their best values: $34.51$\,dB PSNR, $0.962$ SSIM, and $0.014$ LPIPS. Compared with $\kappa=0$, this setting improves PSNR by $0.79$\,dB and SSIM by $0.012$, while reducing LPIPS by $0.007$. When $\kappa$ is further increased to $8$ and $10$, PSNR decreases to $33.94$\,dB and $33.30$\,dB, respectively. This indicates that an excessively strong bridge constraint may impair the balance between measurement consistency and the endpoint priors.

\begin{figure}[!htbp]
\centering
\includegraphics[width=0.85\textwidth]{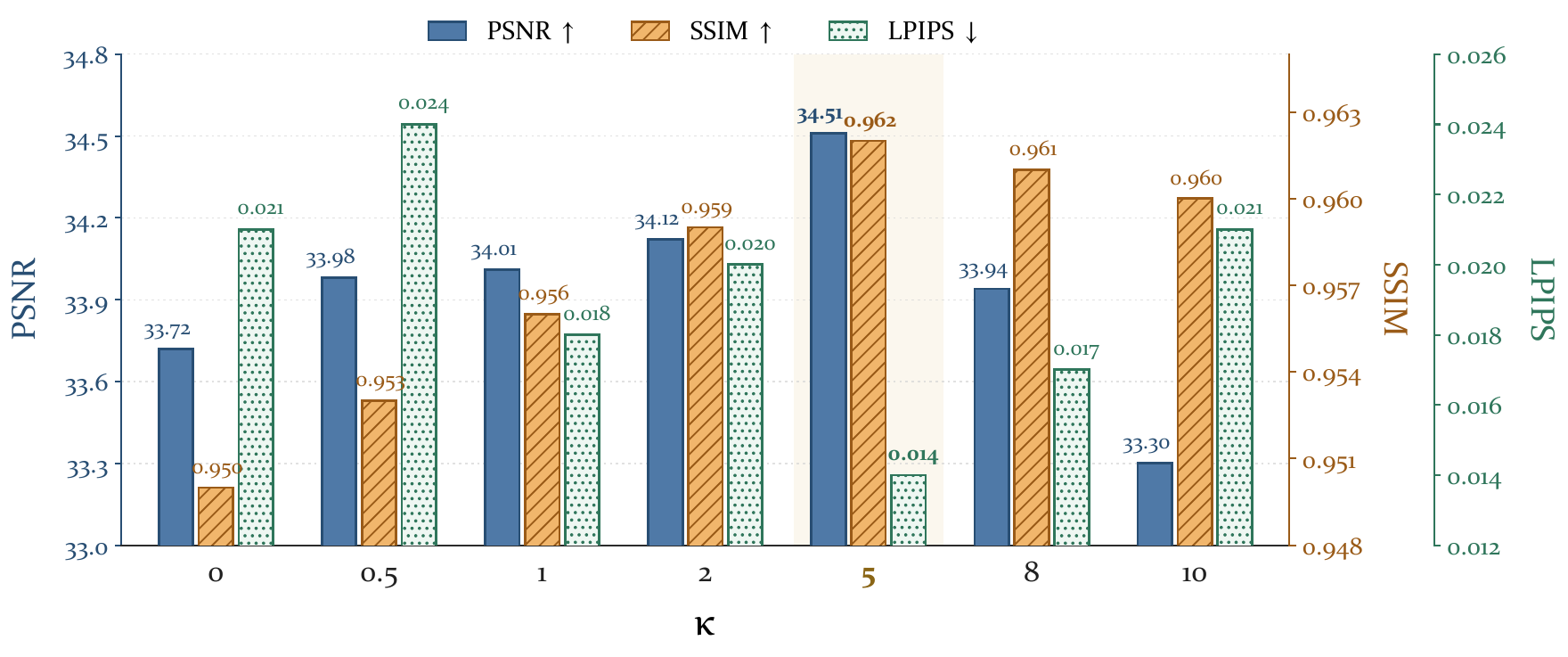} 
\caption{Sensitivity to the bridge re-coupling strength $\kappa$ on CelebA $2\times$ super-resolution.}
\label{Figure_Kappa}
\end{figure}

\subsection{Trajectory-Level Analysis of PBD Components}
To analyze the temporal roles of the Posterior Bridge Defect components during sampling, we track the Measurement Defect $\mathcal{M}_t$, clean-side Flow-Prior Deviation $\mathcal{F}_t^b$, source-side Flow-Prior Deviation $\mathcal{F}_t^a$, and Bridge Residual $\mathcal{B}_t$ at each sampling step on CelebA $2\times$ super-resolution, random inpainting, Gaussian deblurring, and box inpainting. All values are normalized by the image dimensionality, and the sampling trajectory is divided into early, middle, and late stages. We compare Clean-Side Only, Same-$\bar b_t$, Frozen Source, and ReBridge-Flow. Same-$\bar b_t$, Frozen Source applies the same clean-side anchoring as ReBridge-Flow but keeps the source endpoint fixed at $\hat a_t$, thereby isolating the effect of source-side re-coupling. The curves and shaded regions denote the mean and one standard deviation across samples, respectively.

As shown in Figure~\ref{Figure_PBD}, Same-$\bar b_t$, Frozen Source and ReBridge-Flow exhibit nearly identical Measurement Defect and clean-side Flow-Prior Deviation, indicating comparable clean-side corrections. Their main difference lies in whether the source endpoint is updated synchronously. The source-side Flow-Prior Deviation of ReBridge-Flow is concentrated in the middle stage and gradually approaches zero at both ends, consistent with the temporal factor $t(1-t)$ in the source update. This controlled source-side correction substantially reduces the Bridge Residual across all four tasks, with the largest difference observed in Gaussian deblurring. In contrast, correcting only the clean endpoint or keeping the source endpoint fixed produces a larger endpoint--state mismatch during the middle stage. These results show that ReBridge-Flow does not merely seek the lowest Measurement Defect. Instead, it uses a controlled source-side adjustment to balance observation consistency, flow-prior preservation, and local bridge compatibility.

\begin{figure}[!t]
\centering
\includegraphics[width=0.90\textwidth]{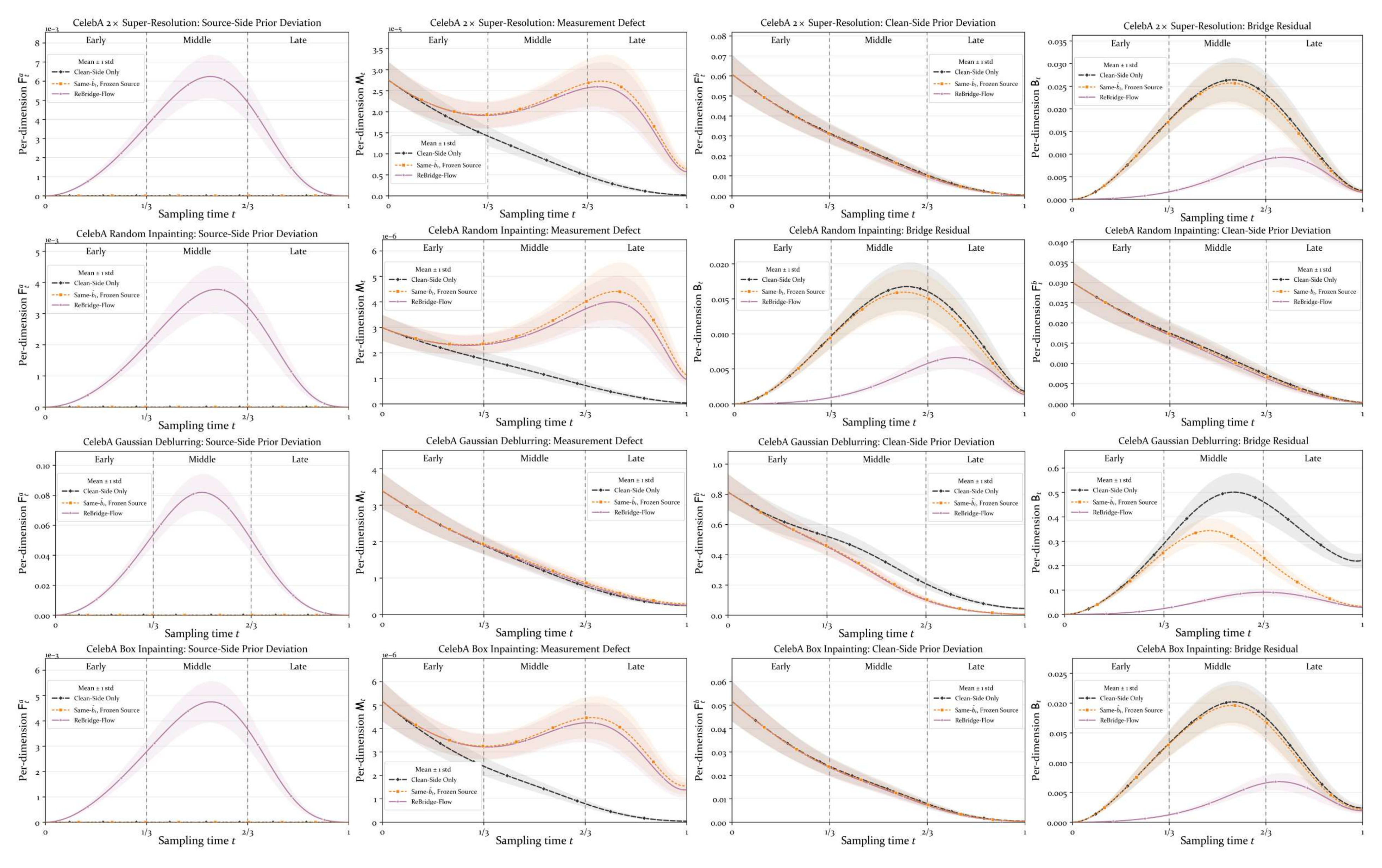} 
\caption{
Mechanism-consistent simulation of the temporal evolution of the PBD components
on four CelebA restoration tasks. We compare Clean-Side Only,
Same-$\bar b_t$, Frozen Source, and ReBridge-Flow. Curves show
the mean $\pm$ one standard deviation over samples.
}
\label{Figure_PBD}
\end{figure}

\subsection{Challenging Tasks}
To further evaluate the restoration capability of different methods under severe information loss, we conduct additional high-factor super-resolution experiments on CelebA. We randomly select and fix 10 images from the CelebA test set and add Gaussian noise with a standard deviation of $\sigma_\mathrm{y}=0.01$ in the measurement space. We compare PnP-Flow, OT-ODE, Flower, Restora-Flow, and ReBridge-Flow.

As shown in Table~\ref{Table_Hard}, the quantitative results show that all methods experience performance degradation as the super-resolution factor increases from $4\times$ to $8\times$, due to the further reduction in observed information. ReBridge-Flow achieves the best overall reconstruction quality at both scales. It also maintains low GPU-memory usage while providing faster inference than Restora-Flow and Flower. As shown in Figure~\ref{Figure_Challenging}, the qualitative results show a consistent trend. PnP-Flow suffers from evident over-smoothing at large upscaling factors. Flower recovers the main facial structures, but some local details remain blurred. OT-ODE and Restora-Flow produce sharper reconstructions, yet still exhibit deviations in facial contours and texture recovery. In contrast, ReBridge-Flow more consistently preserves identity, facial structure, and local details such as hair and glasses. These results indicate that, when the observation contains severely limited information, source--clean endpoint re-coupling helps maintain local transport consistency, thereby reducing structural shifts and detail loss.

\begin{table}[!htbp]
  \centering
  \caption{Quantitative comparison of the large-factor super-resolution task on the CelebA dataset. The \colorbox{best1}{best} and \colorbox{best2}{suboptimal} results are highlighted.}
  \label{Table_Hard}
  
  \footnotesize 
  \setlength{\tabcolsep}{5.0pt}
  \renewcommand{\arraystretch}{1.3} 
  \setlength{\aboverulesep}{0pt}
  \setlength{\belowrulesep}{0pt}
  
  \resizebox{0.95\linewidth}{!}{%
  \begin{tabular}{clccccc}
    \toprule
    Scale & Method & PSNR $\uparrow$ & SSIM $\uparrow$ & LPIPS $\downarrow$ & Time (s/img) $\downarrow$ & Peak GPU (GB) $\downarrow$ \\
    \midrule
    \multirow{6}{*}{$4\times$} 
     & Degraded               & $24.104 \pm 2.062$ & $0.7751 \pm 0.0331$ & $0.2674 \pm 0.0465$ & -- & -- \\
     & PnP-Flow               & $20.233 \pm 1.083$ & $0.6309 \pm 0.0824$ & $0.3073 \pm 0.0877$ & 1.032 & 0.626 \\
     & OT-ODE                 & $28.869 \pm 1.605$ & $0.8565 \pm 0.0317$ & \cellcolor{best2} $0.0653 \pm 0.0374$ & 1.201 & 5.657 \\
     & Restora-Flow           & $30.852 \pm 1.545$ & \cellcolor{best2} $0.8951 \pm 0.0205$ & $0.0699 \pm 0.0413$ & 2.602 & 0.622 \\
     & Flower                 & \cellcolor{best2} $30.860 \pm 1.391$ & $0.8926 \pm 0.0192$ & $0.0937 \pm 0.0469$ & 6.405 & 0.624 \\
     & ReBridge-Flow          & \cellcolor{best1} $31.806 \pm 1.678$ & \cellcolor{best1} $0.9168 \pm 0.0199$ & \cellcolor{best1} $0.0448 \pm 0.0165$ & 2.041 & 0.632 \\
    \midrule
    \multirow{6}{*}{$8\times$} 
     & Degraded               & $18.779 \pm 1.960$ & $0.5349 \pm 0.0673$ & $0.5021 \pm 0.0682$ & -- & -- \\
     & PnP-Flow               & $13.514 \pm 1.766$ & $0.3987 \pm 0.0910$ & $0.6318 \pm 0.0751$ & 1.032 & 0.579 \\
     & OT-ODE                 & $23.447 \pm 1.419$ & $0.6722 \pm 0.0667$ & $0.1990 \pm 0.0850$ & 1.186 & 5.610 \\
     & Restora-Flow           & \cellcolor{best2} $25.238 \pm 2.081$ & \cellcolor{best2} $0.7535 \pm 0.0491$ & \cellcolor{best2} $0.1866 \pm 0.0815$ & 2.605 & 0.576 \\
     & Flower                 & $24.652 \pm 2.112$ & $0.7367 \pm 0.0493$ & $0.2137 \pm 0.0850$ & 5.735 & 0.578 \\
     & ReBridge-Flow          & \cellcolor{best1} $25.787 \pm 2.024$ & \cellcolor{best1} $0.7787 \pm 0.0484$ & \cellcolor{best1} $0.0963 \pm 0.0249$ & 1.761 & 0.586 \\
    \bottomrule
  \end{tabular}
  }
\end{table}

\begin{figure}[!htbp]
\centering
\includegraphics[width=0.98\textwidth]{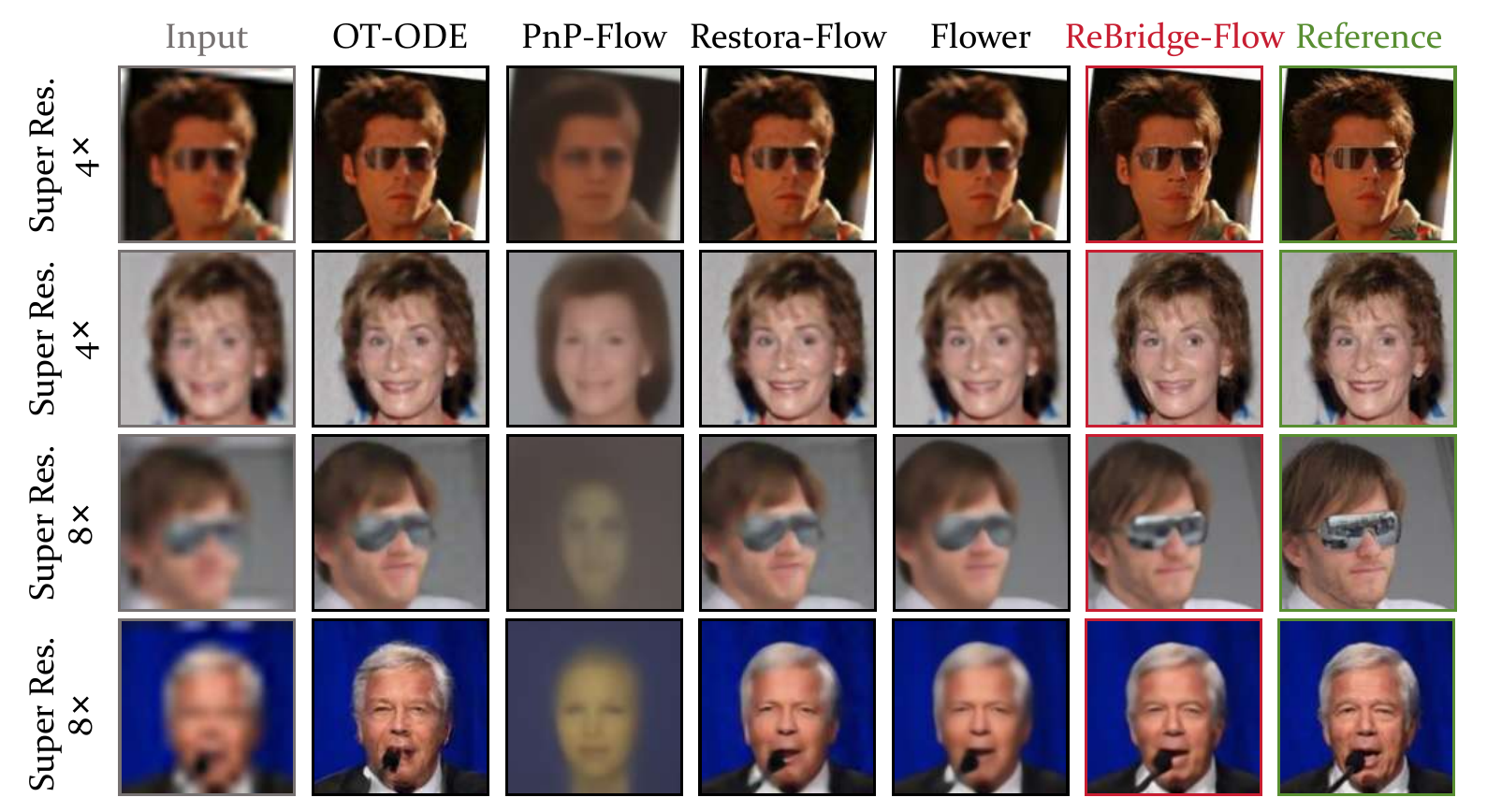} 
  \caption{Qualitative comparison of the large-factor super-resolution task on the CelebA dataset.}
\label{Figure_Challenging}
\end{figure}

\section{More Quantitative and Qualitative Results}
In this section, we provide the complete quantitative comparison results of ReBridge-Flow on the medical image datasets: IXI-Brain, PMUB, and X-Ray Hand, as shown in Table~\ref{Medical}. Furthermore, we provide additional qualitative comparison results for ReBridge-Flow, as illustrated in Figures~\ref{Figure_Supplementary_Materials_Qualitative_1}-\ref{Figure_Supplementary_Materials_Qualitative_27}.

\begin{table}[!t]
\centering
\scriptsize
\renewcommand{\arraystretch}{1.1}
\setlength{\tabcolsep}{5pt}
\caption{Hyperparameters used by all compared methods on the CelebA dataset.}
\label{Param2}
\begin{tabular}{@{}llccccc@{}}
\toprule
Method & Hyperparameter & Denoising & Deblurring & \makecell{Super-\\resolution} & \makecell{Random\\inpainting} & \makecell{Box\\inpainting} \\
\midrule
\multirow{2}{*}{OT-ODE}
& $t_0$ (initial time) & 0.3 & 0.4 & 0.1 & 0.1 & 0.1 \\
& $\gamma$ (guidance schedule) & \makecell{time-\\dependent} & \makecell{time-\\dependent} & constant & constant & \makecell{time-\\dependent} \\
\midrule
\multirow{2}{*}{Flow-Priors}
& $\lambda$ (regularization) & 100 & 1{,}000 & 10{,}000 & 10{,}000 & 10{,}000 \\
& $\eta$ (learning rate) & 0.01 & 0.01 & 0.1 & 0.01 & 0.01 \\
\midrule
\multirow{3}{*}{D-Flow}
& $\lambda$ (regularization) & 0.001 & 0.001 & 0.001 & 0.01 & 0.001 \\
& $\alpha$ (blending) & 0.1 & 0.1 & 0.1 & 0.1 & 0.1 \\
& $n_{\mathrm{iter}}$ (number of iterations) & 3 & 7 & 10 & 20 & 9 \\
\midrule
\multirow{2}{*}{PnP-Flow}
& $\alpha$ (learning-rate factor) & 0.8 & 0.01 & 0.3 & 0.01 & 0.5 \\
& $N$ (number of time steps) & 100 & 100 & 100 & 100 & 100 \\
\midrule
\multirow{2}{*}{Restora-Flow}
& $C$ (number of corrections) & 1 & 1 & 1 & 1 & 1 \\
& $N$ (number of ODE steps) & 64 & 128 & 128 & 128 & 64 \\
\midrule
\multirow{2}{*}{Flower}
& $\gamma$ (refinement uncertainty) & 0 & 0 & 0 & 0 & 0 \\
& $N$ (number of time steps) & 100 & 100 & 100 & 100 & 100 \\
\midrule
\multirow{4}{*}{ReBridge-Flow}
& $\rho$ (clean-side prior weight) & 1 & 1 & 1 & 1 & 1 \\
& $\lambda$ (source-side prior weight) & 1 & 1 & 1 & 1 & 1 \\
& $\kappa$ (bridge re-coupling weight) & 5 & 5 & 5 & 5 & 5 \\
& $K$ (number of sampling steps) & 100 & 100 & 100 & 100 & 100 \\
\bottomrule
\end{tabular}
\end{table}

\begin{table}[!t]
\centering
\scriptsize
\renewcommand{\arraystretch}{1.1}
\setlength{\tabcolsep}{5pt}
\caption{Hyperparameters used by all compared methods on the AFHQ-Cat dataset.}
\label{Param3}
\begin{tabular}{@{}llccccc@{}}
\toprule
Method & Hyperparameter & Denoising & Deblurring & \makecell{Super-\\resolution} & \makecell{Random\\inpainting} & \makecell{Box\\inpainting} \\
\midrule
\multirow{2}{*}{OT-ODE}
& $t_0$ (initial time) & 0.3 & 0.3 & 0.1 & 0.1 & 0.1 \\
& $\gamma$ (guidance schedule) & \makecell{time-\\dependent} & \makecell{time-\\dependent} & constant & constant & \makecell{time-\\dependent} \\
\midrule
\multirow{2}{*}{Flow-Priors}
& $\lambda$ (regularization) & 100 & 1{,}000 & 10{,}000 & 10{,}000 & 10{,}000 \\
& $\eta$ (learning rate) & 0.01 & 0.01 & 0.1 & 0.01 & 0.01 \\
\midrule
\multirow{3}{*}{D-Flow}
& $\lambda$ (regularization) & 0.001 & 0.01 & 0.001 & 0.001 & 0.01 \\
& $\alpha$ (blending) & 0.1 & 0.5 & 0.1 & 0.1 & 0.1 \\
& $n_{\mathrm{iter}}$ (number of iterations) & 3 & 20 & 20 & 20 & 9 \\
\midrule
\multirow{2}{*}{PnP-Flow}
& $\alpha$ (learning-rate factor) & 0.8 & 0.01 & 0.01 & 0.01 & 0.5 \\
& $N$ (number of time steps) & 100 & 500 & 500 & 200 & 100 \\
\midrule
\multirow{2}{*}{Restora-Flow}
& $C$ (number of corrections) & 1 & 1 & 1 & 1 & 1 \\
& $N$ (number of ODE steps) & 64 & 128 & 256 & 128 & 64 \\
\midrule
\multirow{2}{*}{Flower}
& $\gamma$ (refinement uncertainty) & 0 & 0 & 0 & 0 & 0 \\
& $N$ (number of time steps) & 100 & 100 & 500 & 200 & 100 \\
\midrule
\multirow{4}{*}{ReBridge-Flow}
& $\rho$ (clean-side prior weight) & 1 & 1 & 1 & 1 & 1 \\
& $\lambda$ (source-side prior weight) & 1 & 1 & 1 & 1 & 1 \\
& $\kappa$ (bridge re-coupling weight) & 5 & 5 & 5 & 5 & 5 \\
& $K$ (number of sampling steps) & 100 & 100 & 100 & 100 & 100 \\
\bottomrule
\end{tabular}
\end{table}

\begin{table}[!t]
\centering
\scriptsize
\renewcommand{\arraystretch}{1.1}
\setlength{\tabcolsep}{5pt}
\caption{Hyperparameters used by all compared methods on the COCO dataset.}
\label{Param4}
\begin{tabular}{@{}llccccc@{}}
\toprule
Method & Hyperparameter & Denoising & Deblurring & \makecell{Super-\\resolution} & \makecell{Random\\inpainting} & \makecell{Box\\inpainting} \\
\midrule
\multirow{2}{*}{OT-ODE}
& $t_0$ (initial time) & 0.3 & 0.4 & 0.1 & 0.1 & 0.1 \\
& $\gamma$ (guidance schedule) & \makecell{time-\\dependent} & \makecell{time-\\dependent} & constant & constant & \makecell{time-\\dependent} \\
\midrule
\multirow{2}{*}{Flow-Priors}
& $\lambda$ (regularization) & 100 & 1{,}000 & 10{,}000 & 10{,}000 & 10{,}000 \\
& $\eta$ (learning rate) & 0.01 & 0.01 & 0.1 & 0.01 & 0.01 \\
\midrule
\multirow{3}{*}{D-Flow}
& $\lambda$ (regularization) & 0.001 & 0.001 & 0.001 & 0.01 & 0.001 \\
& $\alpha$ (blending) & 0.1 & 0.1 & 0.1 & 0.1 & 0.1 \\
& $n_{\mathrm{iter}}$ (number of iterations) & 3 & 7 & 10 & 20 & 9 \\
\midrule
\multirow{2}{*}{PnP-Flow}
& $\alpha$ (learning-rate factor) & 0.8 & 0.01 & 0.3 & 0.01 & 0.5 \\
& $N$ (number of time steps) & 100 & 100 & 100 & 100 & 100 \\
\midrule
\multirow{2}{*}{Restora-Flow}
& $C$ (number of corrections) & 1 & 1 & 1 & 1 & 1 \\
& $N$ (number of ODE steps) & 64 & 128 & 128 & 128 & 64 \\
\midrule
\multirow{2}{*}{Flower}
& $\gamma$ (refinement uncertainty) & 0 & 0 & 0 & 0 & 0 \\
& $N$ (number of time steps) & 100 & 100 & 100 & 100 & 100 \\
\midrule
\multirow{4}{*}{ReBridge-Flow}
& $\rho$ (clean-side prior weight) & 1 & 1 & 1 & 1 & 1 \\
& $\lambda$ (source-side prior weight) & 1 & 1 & 1 & 1 & 1 \\
& $\kappa$ (bridge re-coupling weight) & 5 & 5 & 5 & 5 & 5 \\
& $K$ (number of sampling steps) & 100 & 100 & 100 & 100 & 100 \\
\bottomrule
\end{tabular}
\end{table}

\begin{table}[!t]
\centering
\scriptsize
\renewcommand{\arraystretch}{1.1}
\setlength{\tabcolsep}{6pt}
\caption{Hyperparameters used by all compared methods on the IXI-Brain dataset.}
\label{Param5}
\begin{tabular}{@{}llcccc@{}}
\toprule
Method & Hyperparameter & Denoising & \makecell{Super-\\resolution} & \makecell{Random\\inpainting} & \makecell{Box\\inpainting} \\
\midrule
\multirow{2}{*}{OT-ODE}
& $t_0$ (initial time) & 0.3 & 0.1 & 0.1 & 0.1 \\
& $\gamma$ (guidance schedule) & \makecell{time-\\dependent} & constant & constant & \makecell{time-\\dependent} \\
\midrule
\multirow{2}{*}{Flow-Priors}
& $\lambda$ (regularization) & 100 & 10{,}000 & 10{,}000 & 10{,}000 \\
& $\eta$ (learning rate) & 0.01 & 0.1 & 0.01 & 0.01 \\
\midrule
\multirow{3}{*}{D-Flow}
& $\lambda$ (regularization) & 0.001 & 0.001 & 0.001 & 0.01 \\
& $\alpha$ (blending) & 0.1 & 0.1 & 0.1 & 0.1 \\
& $n_{\mathrm{iter}}$ (number of iterations) & 3 & 20 & 20 & 9 \\
\midrule
\multirow{2}{*}{PnP-Flow}
& $\alpha$ (learning-rate factor) & 0.8 & 0.01 & 0.01 & 0.5 \\
& $N$ (number of time steps) & 100 & 500 & 200 & 100 \\
\midrule
\multirow{2}{*}{Restora-Flow}
& $C$ (number of corrections) & 1 & 1 & 1 & 1 \\
& $N$ (number of ODE steps) & 64 & 64 & 32 & 32 \\
\midrule
\multirow{2}{*}{Flower}
& $\gamma$ (refinement uncertainty) & 0 & 0 & 0 & 0 \\
& $N$ (number of time steps) & 100 & 500 & 200 & 100 \\
\midrule
\multirow{4}{*}{ReBridge-Flow}
& $\rho$ (clean-side prior weight) & 1 & 1 & 1 & 1 \\
& $\lambda$ (source-side prior weight) & 1 & 1 & 1 & 1 \\
& $\kappa$ (bridge re-coupling weight) & 5 & 5 & 5 & 5 \\
& $K$ (number of sampling steps) & 100 & 100 & 100 & 100 \\
\bottomrule
\end{tabular}
\end{table}

\begin{table}[!t]
\centering
\scriptsize
\renewcommand{\arraystretch}{1.1}
\setlength{\tabcolsep}{6pt}
\caption{Hyperparameters used by all compared methods on the PMUB dataset.}
\label{Param6}
\begin{tabular}{@{}llcccc@{}}
\toprule
Method & Hyperparameter & Denoising & \makecell{Super-\\resolution} & \makecell{Random\\inpainting} & \makecell{Box\\inpainting} \\
\midrule
\multirow{2}{*}{OT-ODE}
& $t_0$ (initial time) & 0.3 & 0.1 & 0.1 & 0.1 \\
& $\gamma$ (guidance schedule) & \makecell{time-\\dependent} & constant & constant & \makecell{time-\\dependent} \\
\midrule
\multirow{2}{*}{Flow-Priors}
& $\lambda$ (regularization) & 100 & 10{,}000 & 10{,}000 & 10{,}000 \\
& $\eta$ (learning rate) & 0.01 & 0.1 & 0.01 & 0.01 \\
\midrule
\multirow{3}{*}{D-Flow}
& $\lambda$ (regularization) & 0.001 & 0.001 & 0.001 & 0.01 \\
& $\alpha$ (blending) & 0.1 & 0.1 & 0.1 & 0.1 \\
& $n_{\mathrm{iter}}$ (number of iterations) & 3 & 20 & 20 & 9 \\
\midrule
\multirow{2}{*}{PnP-Flow}
& $\alpha$ (learning-rate factor) & 0.8 & 0.01 & 0.01 & 0.5 \\
& $N$ (number of time steps) & 100 & 500 & 200 & 100 \\
\midrule
\multirow{2}{*}{Restora-Flow}
& $C$ (number of corrections) & 1 & 1 & 1 & 1 \\
& $N$ (number of ODE steps) & 64 & 64 & 32 & 32 \\
\midrule
\multirow{2}{*}{Flower}
& $\gamma$ (refinement uncertainty) & 0 & 0 & 0 & 0 \\
& $N$ (number of time steps) & 100 & 500 & 200 & 100 \\
\midrule
\multirow{4}{*}{ReBridge-Flow}
& $\rho$ (clean-side prior weight) & 1 & 1 & 1 & 1 \\
& $\lambda$ (source-side prior weight) & 1 & 1 & 1 & 1 \\
& $\kappa$ (bridge re-coupling weight) & 5 & 5 & 5 & 5 \\
& $K$ (number of sampling steps) & 100 & 100 & 100 & 100 \\
\bottomrule
\end{tabular}
\end{table}

\begin{table}[!t]
\centering
\scriptsize
\renewcommand{\arraystretch}{1.1}
\setlength{\tabcolsep}{6pt}
\caption{Hyperparameters used by all compared methods on the X-Ray Hand dataset.}
\label{Param7}
\begin{tabular}{@{}llcccc@{}}
\toprule
Method & Hyperparameter & Denoising & \makecell{Super-\\resolution} & \makecell{Random\\inpainting} & \makecell{Box\\inpainting} \\
\midrule
\multirow{2}{*}{OT-ODE}
& $t_0$ (initial time) & 0.3 & 0.1 & 0.1 & 0.1 \\
& $\gamma$ (guidance schedule) & \makecell{time-\\dependent} & constant & constant & \makecell{time-\\dependent} \\
\midrule
\multirow{2}{*}{Flow-Priors}
& $\lambda$ (regularization) & 100 & 10{,}000 & 10{,}000 & 10{,}000 \\
& $\eta$ (learning rate) & 0.01 & 0.1 & 0.01 & 0.01 \\
\midrule
\multirow{3}{*}{D-Flow}
& $\lambda$ (regularization) & 0.001 & 0.001 & 0.001 & 0.01 \\
& $\alpha$ (blending) & 0.1 & 0.1 & 0.1 & 0.1 \\
& $n_{\mathrm{iter}}$ (number of iterations) & 3 & 20 & 20 & 9 \\
\midrule
\multirow{2}{*}{PnP-Flow}
& $\alpha$ (learning-rate factor) & 0.8 & 0.01 & 0.01 & 0.5 \\
& $N$ (number of time steps) & 100 & 500 & 200 & 100 \\
\midrule
\multirow{2}{*}{Restora-Flow}
& $C$ (number of corrections) & 1 & 1 & 1 & 1 \\
& $N$ (number of ODE steps) & 64 & 64 & 32 & 32 \\
\midrule
\multirow{2}{*}{Flower}
& $\gamma$ (refinement uncertainty) & 0 & 0 & 0 & 0 \\
& $N$ (number of time steps) & 100 & 500 & 200 & 100 \\
\midrule
\multirow{4}{*}{ReBridge-Flow}
& $\rho$ (clean-side prior weight) & 1 & 1 & 1 & 1 \\
& $\lambda$ (source-side prior weight) & 1 & 1 & 1 & 1 \\
& $\kappa$ (bridge re-coupling weight) & 5 & 5 & 5 & 5 \\
& $K$ (number of sampling steps) & 100 & 100 & 100 & 100 \\
\bottomrule
\end{tabular}
\end{table}

\begin{table}[!t]
\centering
\caption{Quantitative comparison on IXI-Brain, PMUB, and X-Ray Hand under different medical image restoration tasks. The \colorbox{best1}{best} and \colorbox{best2}{suboptimal} results are highlighted.}
\label{Medical}

\footnotesize 
\renewcommand{\arraystretch}{1.8} 
\setlength{\tabcolsep}{1.8pt} 
\setlength{\aboverulesep}{0pt}
\setlength{\belowrulesep}{0pt}

\resizebox{\textwidth}{!}{
\begin{tabular}{@{}l ccc ccc ccc ccc@{}}
\toprule

\multicolumn{13}{c}{IXI-Brain} \\
\midrule
\multirow{2}{*}{Model}
& \multicolumn{3}{c}{Denoising $\sigma_\mathbf{y}=0.08$}
& \multicolumn{3}{c}{Super Res. $2\times$}
& \multicolumn{3}{c}{Rand. Inpaint. 30\%}
& \multicolumn{3}{c}{Box Inpaint. $32\times32$} \\
\cmidrule(lr){2-4}
\cmidrule(lr){5-7}
\cmidrule(lr){8-10}
\cmidrule(l){11-13}
& PSNR$\uparrow$ & SSIM$\uparrow$ & LPIPS$\downarrow$
& PSNR$\uparrow$ & SSIM$\uparrow$ & LPIPS$\downarrow$
& PSNR$\uparrow$ & SSIM$\uparrow$ & LPIPS$\downarrow$
& PSNR$\uparrow$ & SSIM$\uparrow$ & LPIPS$\downarrow$ \\
\midrule
Degraded & 21.89$_{(0.11)}$ & 0.238$_{(0.047)}$ & 0.313$_{(0.127)}$ & 14.83$_{(1.12)}$ & 0.563$_{(0.078)}$ & 0.365$_{(0.140)}$ & 8.84$_{(0.26)}$ & 0.072$_{(0.015)}$ & 1.076$_{(0.250)}$ & 21.70$_{(1.14)}$ & 0.894$_{(0.009)}$ & 0.108$_{(0.047)}$ \\
\midrule

OT-ODE~\textcolor{gray}{\scriptsize [TMLR2024]} & 33.11$_{(1.17)}$ & 0.877$_{(0.017)}$ & \cellcolor{best2}0.023$_{(0.010)}$ & 22.07$_{(1.66)}$ & 0.713$_{(0.059)}$ & 0.110$_{(0.050)}$ & 28.53$_{(3.61)}$ & 0.837$_{(0.192)}$ & 0.083$_{(0.042)}$ & 24.25$_{(3.69)}$ & 0.901$_{(0.017)}$ & 0.073$_{(0.025)}$ \\
Flow-Priors~\textcolor{gray}{\scriptsize [NeurIPS2024]} & 25.17$_{(1.96)}$ & 0.839$_{(0.035)}$ & 0.069$_{(0.025)}$ & \cellcolor{best2}24.31$_{(1.28)}$ & 0.824$_{(0.028)}$ & \cellcolor{best1}0.071$_{(0.033)}$ & 29.63$_{(1.83)}$ & 0.919$_{(0.018)}$ & 0.032$_{(0.016)}$ & \cellcolor{best2}27.25$_{(2.61)}$ & \cellcolor{best2}0.925$_{(0.014)}$ & 0.042$_{(0.018)}$ \\
D-Flow~\textcolor{gray}{\scriptsize [PMLR2024]} & 27.02$_{(1.37)}$ & 0.835$_{(0.060)}$ & 0.095$_{(0.035)}$ & 19.53$_{(0.68)}$ & 0.654$_{(0.016)}$ & 0.276$_{(0.132)}$ & 28.52$_{(0.74)}$ & 0.794$_{(0.026)}$ & 0.091$_{(0.037)}$ & 22.70$_{(1.44)}$ & 0.785$_{(0.064)}$ & 0.087$_{(0.038)}$ \\
PnP-Flow~\textcolor{gray}{\scriptsize [ICLR2025]} & \cellcolor{best2}33.90$_{(1.15)}$ & \cellcolor{best2}0.930$_{(0.018)}$ & 0.040$_{(0.019)}$ & 21.34$_{(1.55)}$ & 0.724$_{(0.070)}$ & 0.192$_{(0.089)}$ & 32.35$_{(1.30)}$ & \cellcolor{best2}0.942$_{(0.015)}$ & 0.032$_{(0.013)}$ & 26.30$_{(3.67)}$ & 0.913$_{(0.018)}$ & 0.057$_{(0.023)}$ \\
Restora-Flow~\textcolor{gray}{\scriptsize [WACV2026]} & 29.87$_{(1.09)}$ & 0.852$_{(0.028)}$ & 0.053$_{(0.019)}$ & 22.24$_{(1.54)}$ & 0.768$_{(0.067)}$ & 0.150$_{(0.057)}$ & 30.03$_{(1.27)}$ & 0.914$_{(0.022)}$ & 0.036$_{(0.018)}$ & 25.92$_{(3.32)}$ & 0.920$_{(0.016)}$ & \cellcolor{best2}0.040$_{(0.015)}$ \\
Flower~\textcolor{gray}{\scriptsize [ICLR2026]} & 32.54$_{(1.24)}$ & 0.907$_{(0.025)}$ & 0.035$_{(0.015)}$ & 24.26$_{(1.45)}$ & \cellcolor{best2}0.865$_{(0.062)}$ & 0.107$_{(0.050)}$ & \cellcolor{best2}32.46$_{(1.43)}$ & 0.933$_{(0.017)}$ & \cellcolor{best2}0.028$_{(0.012)}$ & 26.62$_{(3.57)}$ & 0.915$_{(0.016)}$ & 0.053$_{(0.019)}$ \\
\midrule
ReBridge-Flow~\textcolor{gray}{\scriptsize [Ours]} & \cellcolor{best1}34.06$_{(1.38)}$ & \cellcolor{best1}0.942$_{(0.019)}$ & \cellcolor{best1}0.022$_{(0.009)}$ & \cellcolor{best1}26.39$_{(1.99)}$ & \cellcolor{best1}0.889$_{(0.013)}$ & \cellcolor{best2}0.076$_{(0.030)}$ & \cellcolor{best1}33.73$_{(2.46)}$ & \cellcolor{best1}0.954$_{(0.001)}$ & \cellcolor{best1}0.021$_{(0.011)}$ & \cellcolor{best1}28.21$_{(4.16)}$ & \cellcolor{best1}0.939$_{(0.002)}$ & \cellcolor{best1}0.028$_{(0.011)}$ \\

\midrule[0.8pt]

\multicolumn{13}{c}{PMUB} \\
\midrule
\multirow{2}{*}{Model}
& \multicolumn{3}{c}{Denoising $\sigma_\mathbf{y}=0.08$}
& \multicolumn{3}{c}{Super Res. $2\times$}
& \multicolumn{3}{c}{Rand. Inpaint. 30\%}
& \multicolumn{3}{c}{Box Inpaint. $60\times60$} \\
\cmidrule(lr){2-4}
\cmidrule(lr){5-7}
\cmidrule(lr){8-10}
\cmidrule(l){11-13}
& PSNR$\uparrow$ & SSIM$\uparrow$ & LPIPS$\downarrow$
& PSNR$\uparrow$ & SSIM$\uparrow$ & LPIPS$\downarrow$
& PSNR$\uparrow$ & SSIM$\uparrow$ & LPIPS$\downarrow$
& PSNR$\uparrow$ & SSIM$\uparrow$ & LPIPS$\downarrow$ \\
\midrule
Degraded & 21.50$_{(1.61)}$ & 0.751$_{(0.098)}$ & 0.085$_{(0.034)}$ & 8.93$_{(2.60)}$ & 0.132$_{(0.028)}$ & 0.592$_{(0.226)}$ & 14.71$_{(1.24)}$ & 0.410$_{(0.128)}$ & 0.606$_{(0.250)}$ & 12.79$_{(2.11)}$ & 0.848$_{(0.011)}$ & 0.075$_{(0.031)}$ \\
\midrule
OT-ODE~\textcolor{gray}{\scriptsize [TMLR2024]} & 26.49$_{(1.25)}$ & 0.899$_{(0.041)}$ & 0.026$_{(0.011)}$ & 15.23$_{(1.89)}$ & 0.581$_{(0.063)}$ & 0.124$_{(0.048)}$ & 24.48$_{(2.49)}$ & 0.906$_{(0.036)}$ & 0.020$_{(0.009)}$ & 17.16$_{(1.80)}$ & 0.930$_{(0.038)}$ & 0.058$_{(0.024)}$ \\
Flow-Priors~\textcolor{gray}{\scriptsize [NeurIPS2024]} & 25.95$_{(1.41)}$ & 0.908$_{(0.038)}$ & 0.031$_{(0.012)}$ & 18.59$_{(1.28)}$ & 0.726$_{(0.053)}$ & 0.052$_{(0.021)}$ & \cellcolor{best2}28.44$_{(2.75)}$ & \cellcolor{best1}0.967$_{(0.007)}$ & \cellcolor{best1}0.006$_{(0.004)}$ & 18.31$_{(2.53)}$ & \cellcolor{best2}0.949$_{(0.012)}$ & 0.046$_{(0.020)}$ \\
D-Flow~\textcolor{gray}{\scriptsize [PMLR2024]} & 22.96$_{(1.71)}$ & 0.786$_{(0.073)}$ & 0.090$_{(0.037)}$ & 18.18$_{(2.16)}$ & 0.714$_{(0.034)}$ & \cellcolor{best2}0.048$_{(0.018)}$ & 27.16$_{(1.76)}$ & 0.936$_{(0.025)}$ & 0.008$_{(0.004)}$ & 17.81$_{(2.23)}$ & 0.867$_{(0.037)}$ & 0.049$_{(0.021)}$ \\
PnP-Flow~\textcolor{gray}{\scriptsize [ICLR2025]} & \cellcolor{best2}28.69$_{(1.32)}$ & \cellcolor{best2}0.932$_{(0.038)}$ & \cellcolor{best1}0.011$_{(0.005)}$ & \cellcolor{best2}19.73$_{(2.07)}$ & \cellcolor{best2}0.802$_{(0.021)}$ & 0.061$_{(0.025)}$ & 26.79$_{(1.44)}$ & 0.931$_{(0.021)}$ & 0.025$_{(0.010)}$ & \cellcolor{best2}19.22$_{(2.30)}$ & 0.935$_{(0.016)}$ & 0.052$_{(0.023)}$ \\
Restora-Flow~\textcolor{gray}{\scriptsize [WACV2026]} & 24.16$_{(1.43)}$ & 0.864$_{(0.075)}$ & 0.016$_{(0.007)}$ & 18.17$_{(1.94)}$ & 0.726$_{(0.030)}$ & 0.060$_{(0.027)}$ & 24.32$_{(1.22)}$ & 0.882$_{(0.035)}$ & 0.021$_{(0.010)}$ & 17.55$_{(2.10)}$ & 0.890$_{(0.031)}$ & \cellcolor{best2}0.044$_{(0.017)}$ \\
Flower~\textcolor{gray}{\scriptsize [ICLR2026]} & 26.18$_{(1.49)}$ & 0.907$_{(0.041)}$ & 0.025$_{(0.009)}$ & 19.28$_{(2.01)}$ & 0.786$_{(0.032)}$ & 0.064$_{(0.026)}$ & 27.66$_{(1.37)}$ & \cellcolor{best2}0.938$_{(0.025)}$ & 0.024$_{(0.012)}$ & 18.49$_{(2.29)}$ & 0.919$_{(0.016)}$ & 0.057$_{(0.022)}$ \\
\midrule
ReBridge-Flow~\textcolor{gray}{\scriptsize [Ours]} & \cellcolor{best1}29.03$_{(1.50)}$ & \cellcolor{best1}0.936$_{(0.018)}$ & \cellcolor{best2}0.014$_{(0.006)}$ & \cellcolor{best1}20.09$_{(2.07)}$ & \cellcolor{best1}0.826$_{(0.021)}$ & \cellcolor{best1}0.044$_{(0.019)}$ & \cellcolor{best1}29.35$_{(1.51)}$ & \cellcolor{best1}0.967$_{(0.006)}$ & \cellcolor{best2}0.007$_{(0.004)}$ & \cellcolor{best1}19.30$_{(5.17)}$ & \cellcolor{best1}0.951$_{(0.005)}$ & \cellcolor{best1}0.041$_{(0.015)}$ \\

\midrule[0.8pt]

\multicolumn{13}{c}{X-Ray Hand} \\
\midrule
\multirow{2}{*}{Model}
& \multicolumn{3}{c}{Denoising $\sigma_\mathbf{y}=0.08$}
& \multicolumn{3}{c}{Super Res. $2\times$}
& \multicolumn{3}{c}{Rand. Inpaint. 30\%}
& \multicolumn{3}{c}{Box Inpaint. $32\times32$} \\
\cmidrule(lr){2-4}
\cmidrule(lr){5-7}
\cmidrule(lr){8-10}
\cmidrule(l){11-13}
& PSNR$\uparrow$ & SSIM$\uparrow$ & LPIPS$\downarrow$
& PSNR$\uparrow$ & SSIM$\uparrow$ & LPIPS$\downarrow$
& PSNR$\uparrow$ & SSIM$\uparrow$ & LPIPS$\downarrow$
& PSNR$\uparrow$ & SSIM$\uparrow$ & LPIPS$\downarrow$ \\
\midrule
Degraded & 20.62$_{(0.93)}$ & 0.461$_{(0.045)}$ & 0.310$_{(0.112)}$ & 9.37$_{(0.60)}$ & 0.242$_{(0.069)}$ & 0.502$_{(0.198)}$ & 11.00$_{(0.90)}$ & 0.137$_{(0.031)}$ & 1.100$_{(0.250)}$ & 19.94$_{(1.24)}$ & 0.631$_{(0.070)}$ & 0.303$_{(0.127)}$ \\
\midrule
OT-ODE~\textcolor{gray}{\scriptsize [TMLR2024]} & 29.69$_{(1.79)}$ & 0.885$_{(0.072)}$ & 0.027$_{(0.011)}$ & 18.78$_{(1.64)}$ & 0.451$_{(0.067)}$ & 0.152$_{(0.063)}$ & 21.08$_{(1.99)}$ & 0.568$_{(0.080)}$ & 0.136$_{(0.068)}$ & 24.96$_{(3.70)}$ & \cellcolor{best1}0.847$_{(0.095)}$ & \cellcolor{best1}0.045$_{(0.016)}$ \\
Flow-Priors~\textcolor{gray}{\scriptsize [NeurIPS2024]} & 24.04$_{(2.52)}$ & 0.858$_{(0.067)}$ & 0.054$_{(0.024)}$ & 21.84$_{(1.38)}$ & 0.810$_{(0.062)}$ & 0.068$_{(0.031)}$ & 17.96$_{(2.77)}$ & 0.747$_{(0.075)}$ & 0.099$_{(0.047)}$ & 21.78$_{(2.58)}$ & 0.814$_{(0.072)}$ & 0.063$_{(0.026)}$ \\
D-Flow~\textcolor{gray}{\scriptsize [PMLR2024]} & 23.77$_{(1.91)}$ & 0.743$_{(0.079)}$ & 0.065$_{(0.025)}$ & 22.60$_{(2.96)}$ & 0.792$_{(0.113)}$ & 0.070$_{(0.031)}$ & 16.55$_{(2.86)}$ & 0.639$_{(0.122)}$ & 0.125$_{(0.067)}$ & 20.90$_{(2.40)}$ & 0.782$_{(0.079)}$ & 0.070$_{(0.025)}$ \\
PnP-Flow~\textcolor{gray}{\scriptsize [ICLR2025]} & \cellcolor{best2}30.81$_{(1.88)}$ & \cellcolor{best2}0.896$_{(0.074)}$ & \cellcolor{best1}0.023$_{(0.010)}$ & 24.26$_{(2.30)}$ & 0.855$_{(0.070)}$ & 0.056$_{(0.024)}$ & 23.68$_{(2.15)}$ & 0.830$_{(0.067)}$ & \cellcolor{best2}0.060$_{(0.026)}$ & 24.33$_{(2.50)}$ & \cellcolor{best2}0.838$_{(0.081)}$ & 0.053$_{(0.020)}$ \\
Restora-Flow~\textcolor{gray}{\scriptsize [WACV2026]} & 25.33$_{(1.20)}$ & 0.610$_{(0.073)}$ & 0.065$_{(0.024)}$ & 23.19$_{(2.39)}$ & 0.825$_{(0.080)}$ & \cellcolor{best2}0.052$_{(0.020)}$ & 22.54$_{(2.22)}$ & 0.802$_{(0.072)}$ & 0.068$_{(0.035)}$ & 23.98$_{(2.27)}$ & 0.819$_{(0.085)}$ & \cellcolor{best2}0.048$_{(0.020)}$ \\
Flower~\textcolor{gray}{\scriptsize [ICLR2026]} & 30.46$_{(1.69)}$ & 0.883$_{(0.063)}$ & 0.034$_{(0.013)}$ & \cellcolor{best2}24.69$_{(2.03)}$ & \cellcolor{best2}0.861$_{(0.076)}$ & \cellcolor{best2}0.052$_{(0.022)}$ & \cellcolor{best2}24.36$_{(2.23)}$ & \cellcolor{best2}0.832$_{(0.079)}$ & 0.071$_{(0.032)}$ & \cellcolor{best2}25.13$_{(2.56)}$ & 0.826$_{(0.079)}$ & 0.056$_{(0.019)}$ \\
\midrule
ReBridge-Flow~\textcolor{gray}{\scriptsize [Ours]} & \cellcolor{best1}32.28$_{(2.00)}$ & \cellcolor{best1}0.913$_{(0.031)}$ & \cellcolor{best2}0.025$_{(0.011)}$ & \cellcolor{best1}26.34$_{(2.53)}$ & \cellcolor{best1}0.884$_{(0.035)}$ & \cellcolor{best1}0.049$_{(0.019)}$ & \cellcolor{best1}25.95$_{(2.39)}$ & \cellcolor{best1}0.863$_{(0.002)}$ & \cellcolor{best1}0.057$_{(0.024)}$ & \cellcolor{best1}25.34$_{(2.73)}$ & \cellcolor{best1}0.847$_{(0.001)}$ & 0.051$_{(0.023)}$ \\

\bottomrule
\end{tabular}
}
\end{table}

\begin{figure}[!t]
\centering
\includegraphics[width=0.88\textwidth]{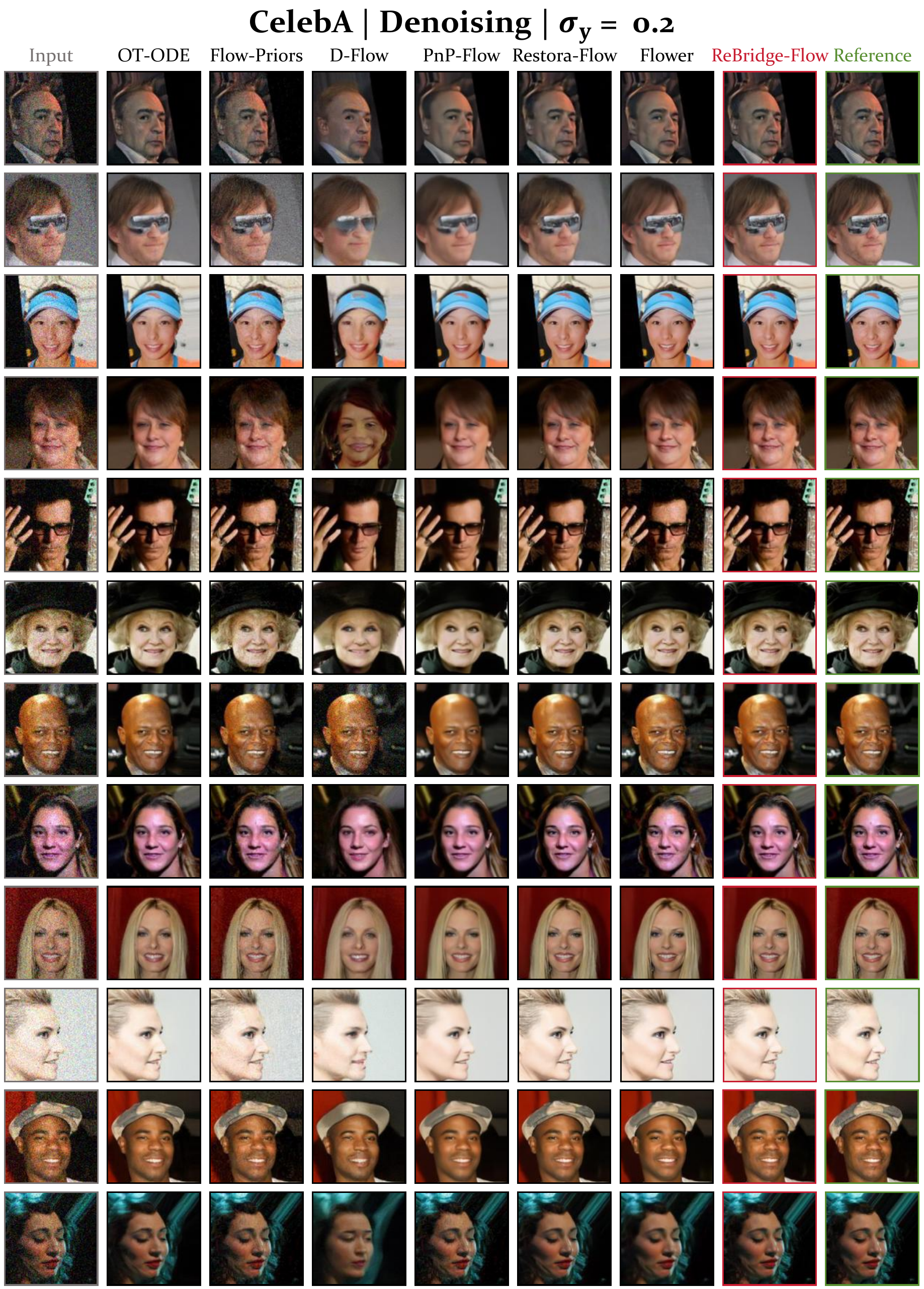} 
\caption{Qualitative visual comparison of Denoising ($\sigma_\mathbf{y} = 0.2$) on CelebA.}
\label{Figure_Supplementary_Materials_Qualitative_1}
\end{figure}

\begin{figure}[!t]
\centering
\includegraphics[width=0.88\textwidth]{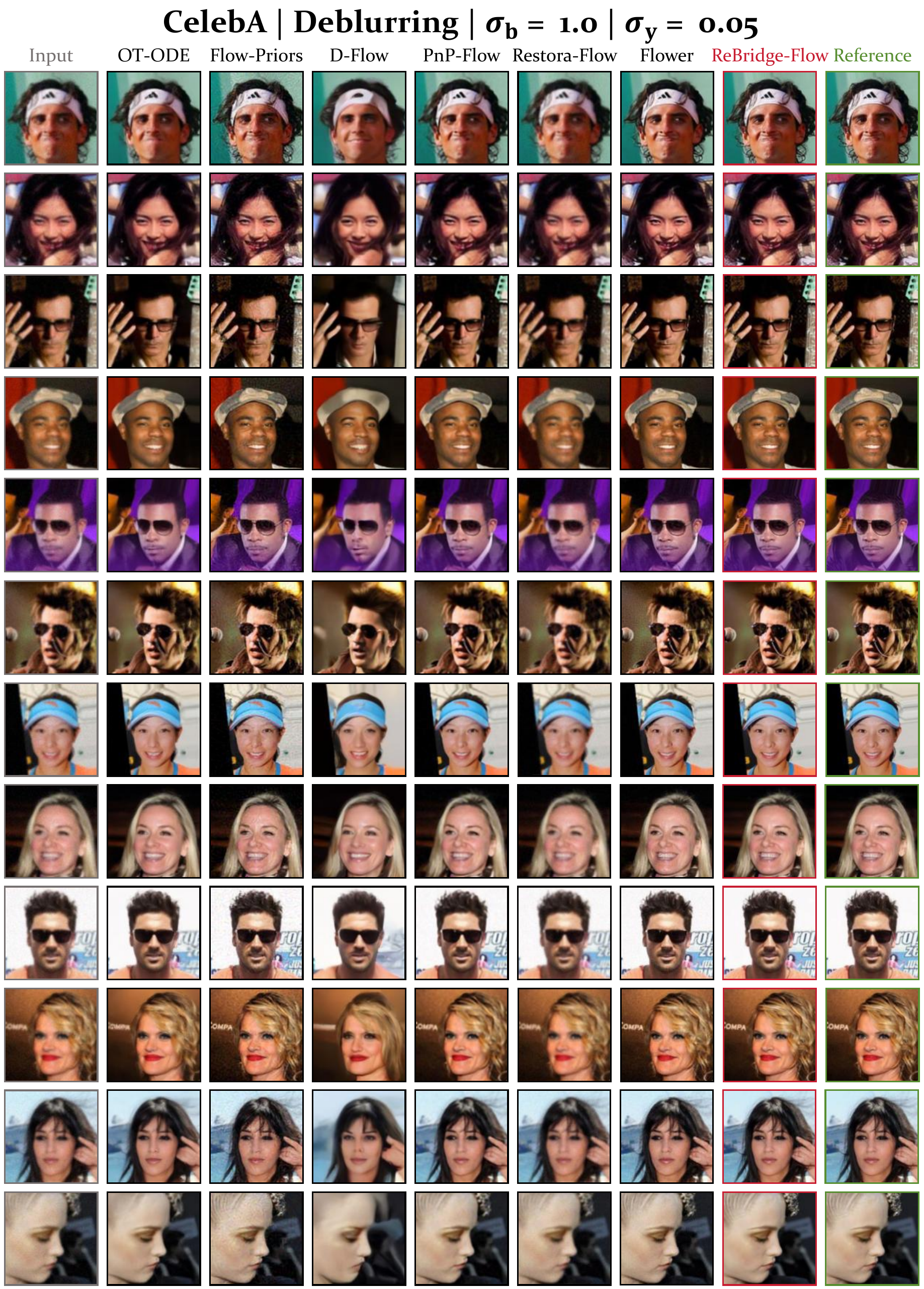} 
\caption{Qualitative visual comparison of Deblurring ($\sigma_{b} = 1.0$, $\sigma_\mathbf{y} = 0.05$) on CelebA.}
\label{Figure_Supplementary_Materials_Qualitative_2}
\end{figure}

\begin{figure}[!t]
\centering
\includegraphics[width=0.88\textwidth]{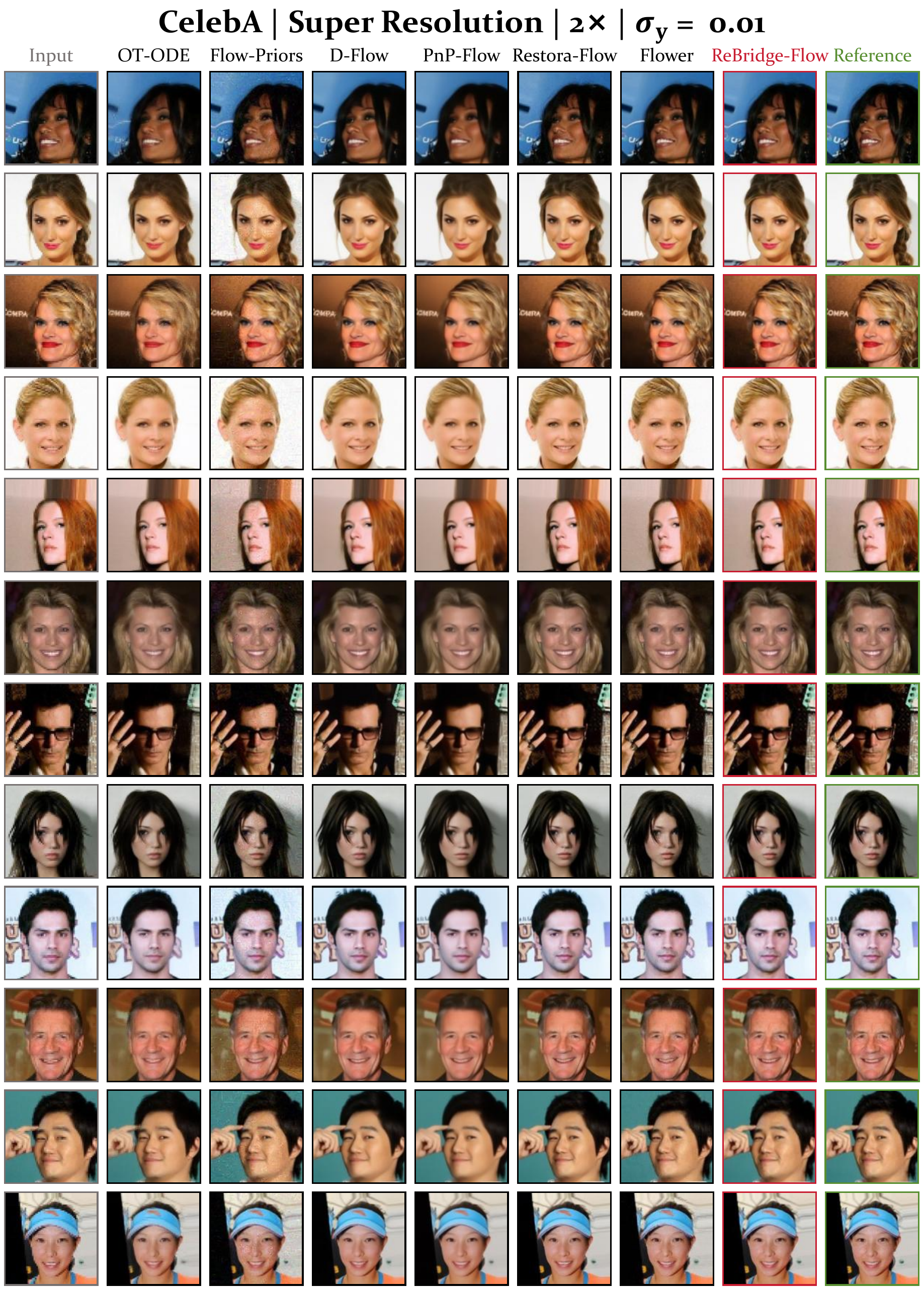} 
\caption{Qualitative visual comparison of Super Resolution ($2\times$, $\sigma_\mathbf{y} = 0.01$) on CelebA.}
\label{Figure_Supplementary_Materials_Qualitative_3}
\end{figure}

\begin{figure}[!t]
\centering
\includegraphics[width=0.88\textwidth]{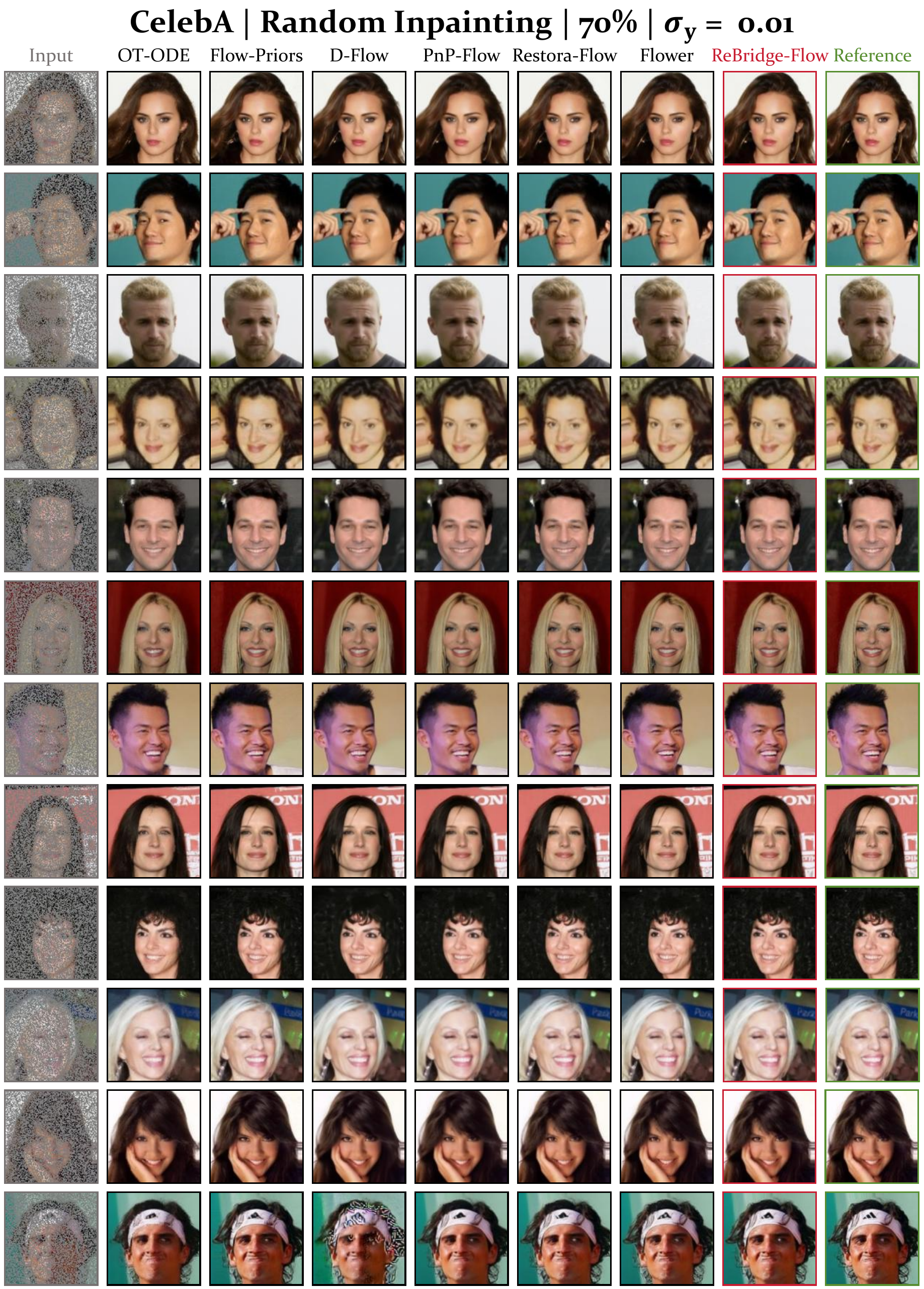} 
\caption{Qualitative visual comparison of Random Inpainting (70\%, $\sigma_\mathbf{y} = 0.01$) on CelebA.}
\label{Figure_Supplementary_Materials_Qualitative_4}
\end{figure}

\begin{figure}[!t]
\centering
\includegraphics[width=0.88\textwidth]{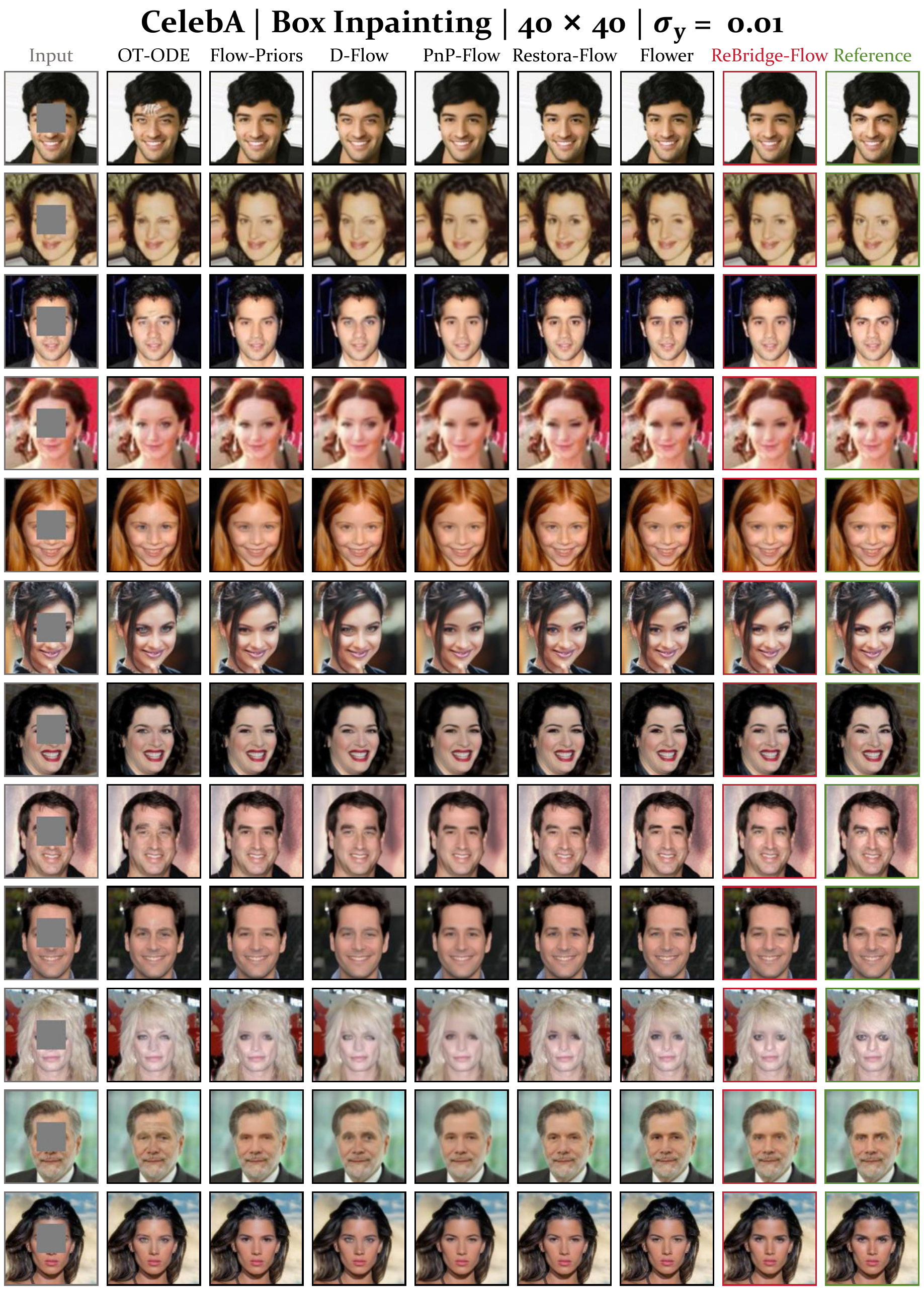} 
\caption{Qualitative visual comparison of Box Inpainting ($40 \times 40$, $\sigma_\mathbf{y} = 0.01$) on CelebA.}
\label{Figure_Supplementary_Materials_Qualitative_5}
\end{figure}

\begin{figure}[!t]
\centering
\includegraphics[width=0.88\textwidth]{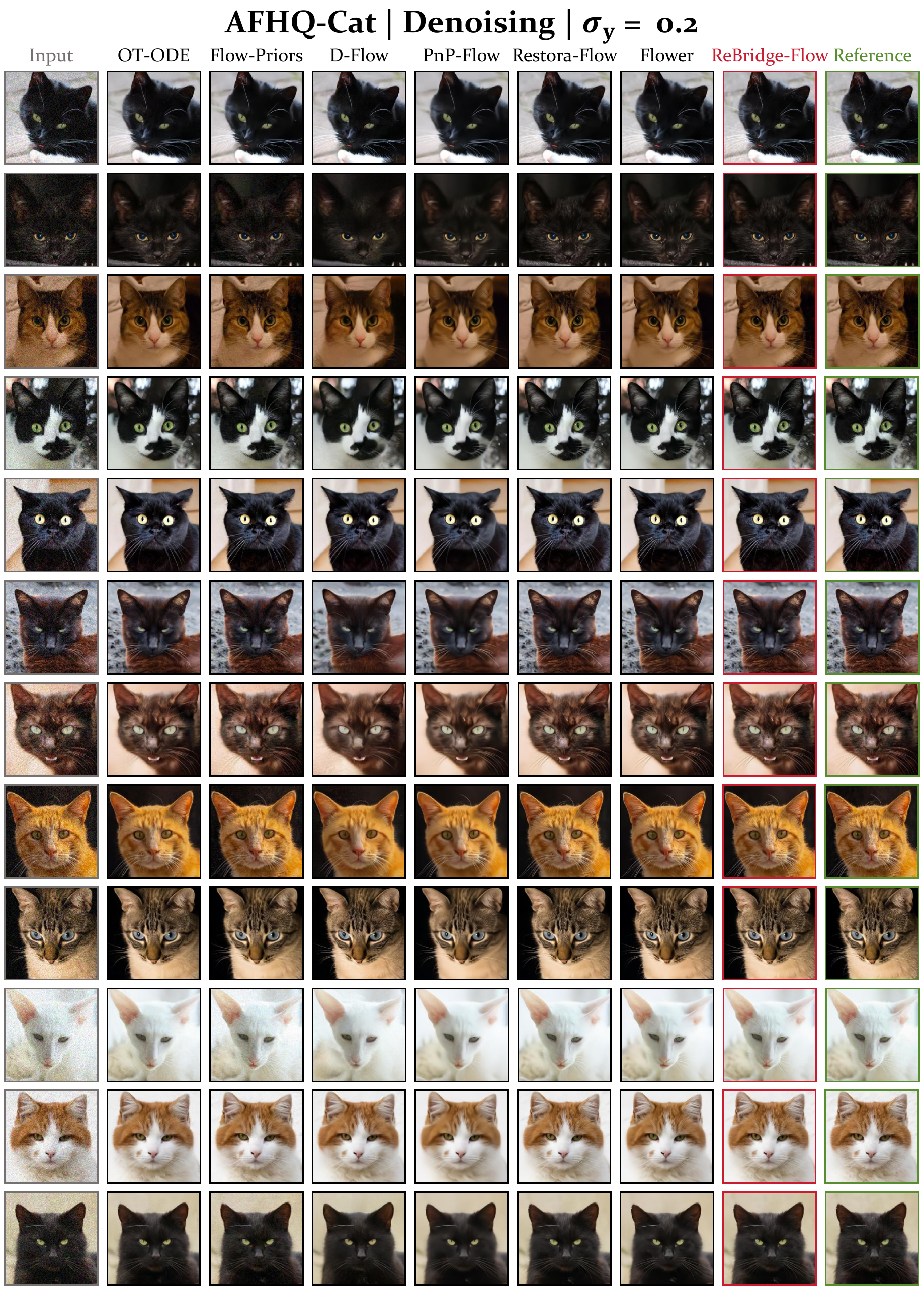} 
\caption{Qualitative visual comparison of Denoising ($\sigma_\mathbf{y} = 0.2$) on AFHQ-Cat.}
\label{Figure_Supplementary_Materials_Qualitative_6}
\end{figure}

\begin{figure}[!t]
\centering
\includegraphics[width=0.88\textwidth]{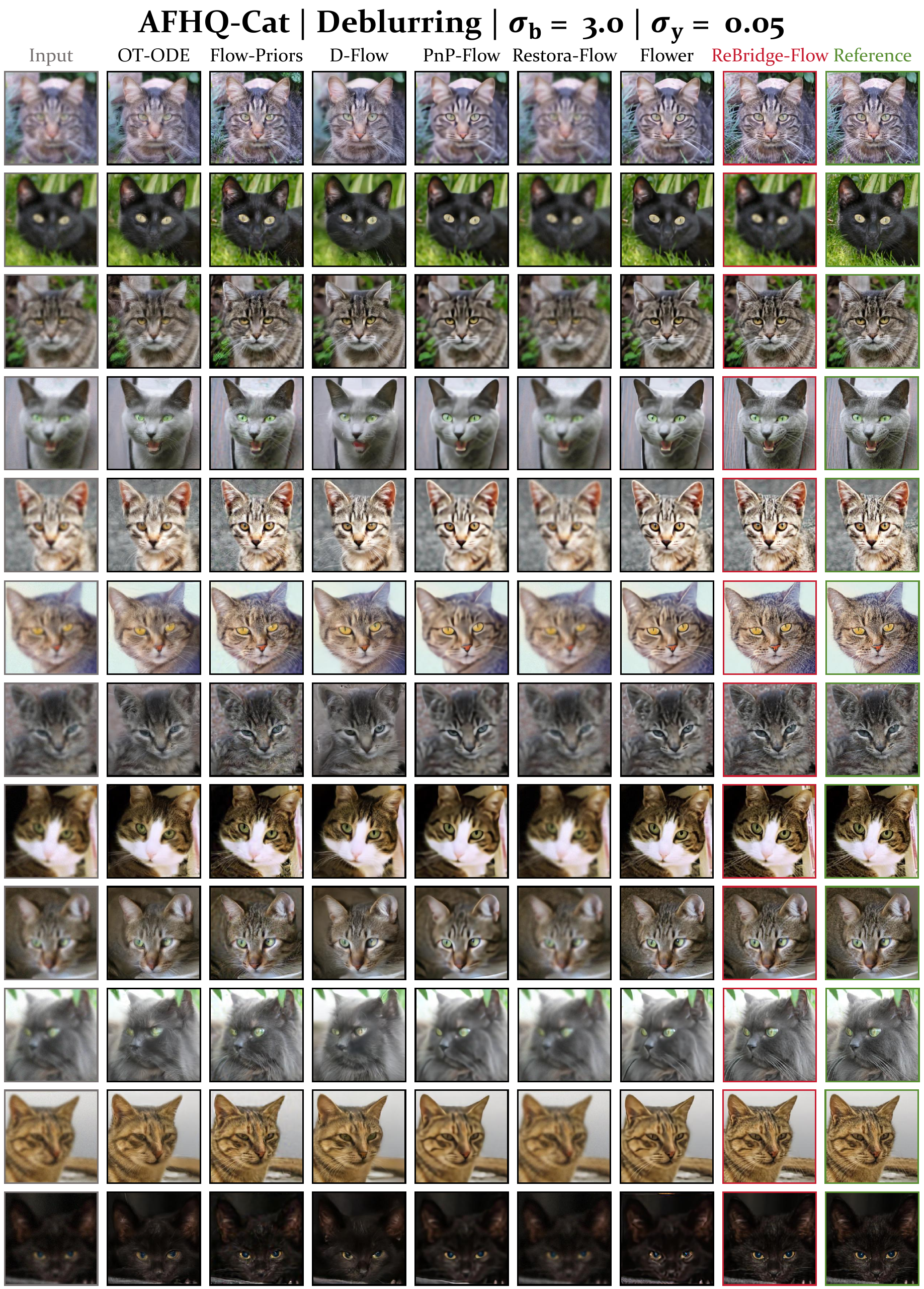} 
\caption{Qualitative visual comparison of Deblurring ($\sigma_b = 3.0$, $\sigma_\mathbf{y} = 0.05$) on AFHQ-Cat.}
\label{Figure_Supplementary_Materials_Qualitative_7}
\end{figure}

\begin{figure}[!t]
\centering
\includegraphics[width=0.88\textwidth]{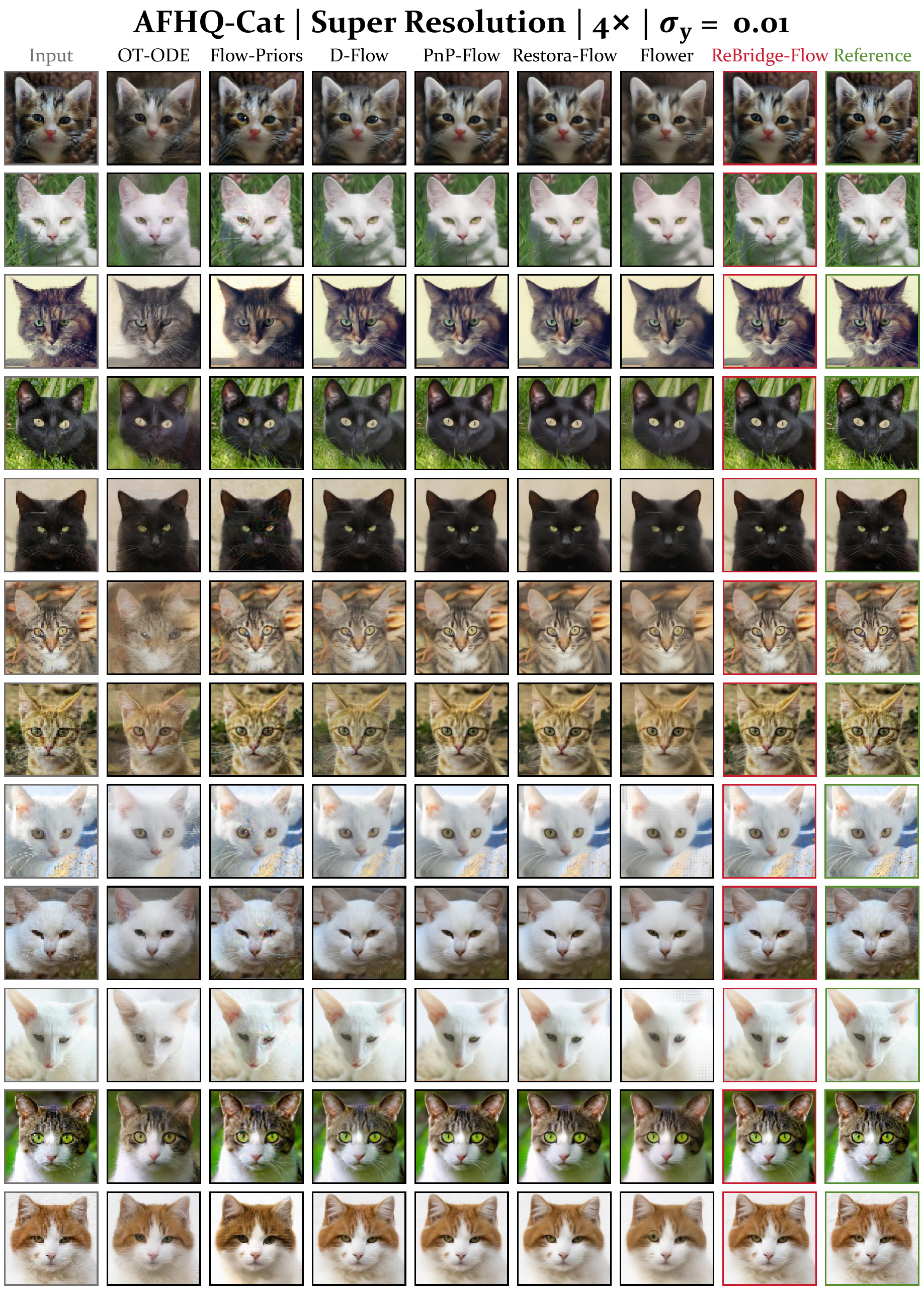} 
\caption{Qualitative visual comparison of Super Resolution ($4\times$, $\sigma_\mathbf{y} = 0.01$) on AFHQ-Cat.}
\label{Figure_Supplementary_Materials_Qualitative_8}
\end{figure}

\begin{figure}[!t]
\centering
\includegraphics[width=0.88\textwidth]{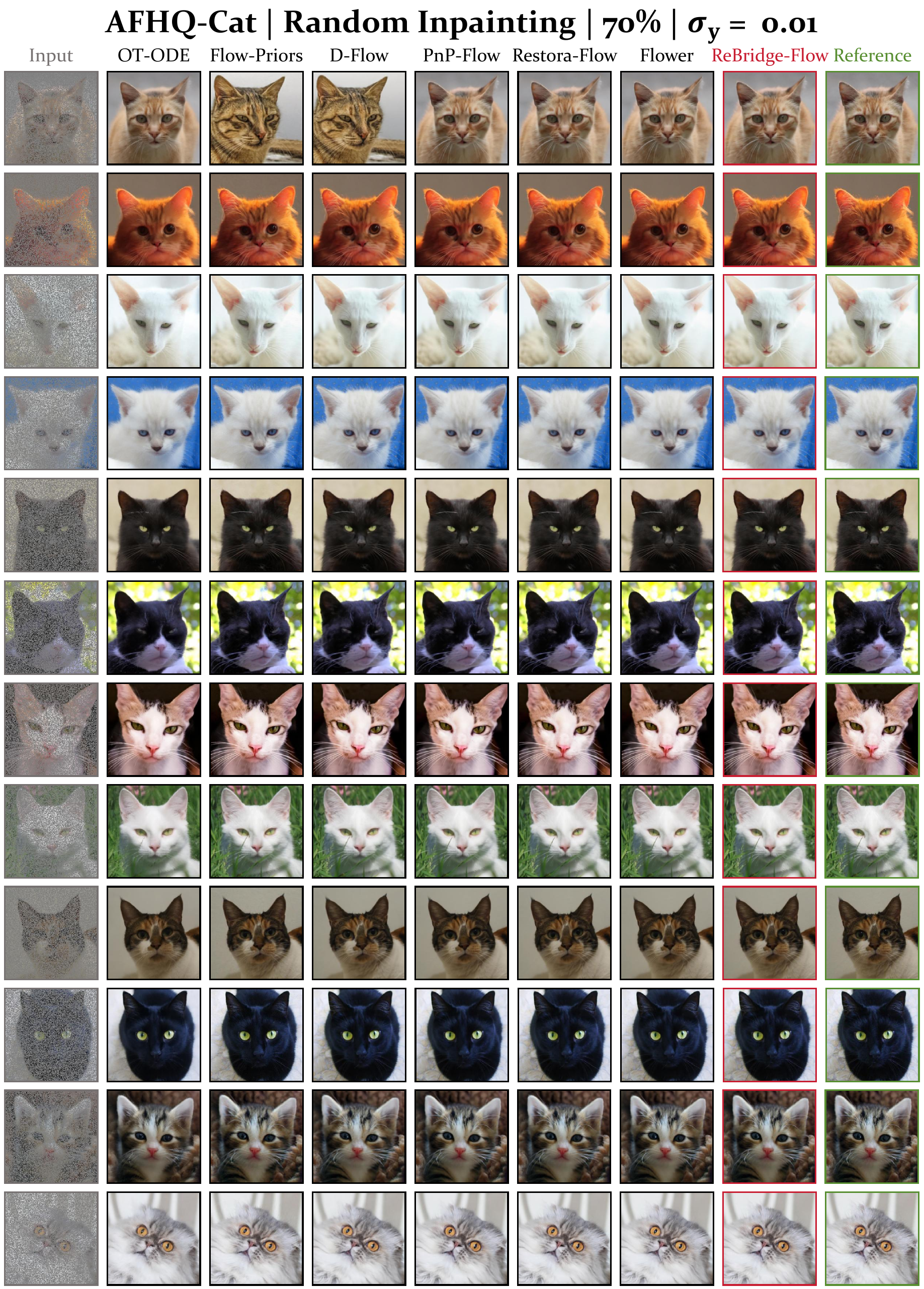} 
\caption{Qualitative visual comparison of Random Inpainting (70\%, $\sigma_\mathbf{y} = 0.01$) on AFHQ-Cat.}
\label{Figure_Supplementary_Materials_Qualitative_9}
\end{figure}

\begin{figure}[!t]
\centering
\includegraphics[width=0.88\textwidth]{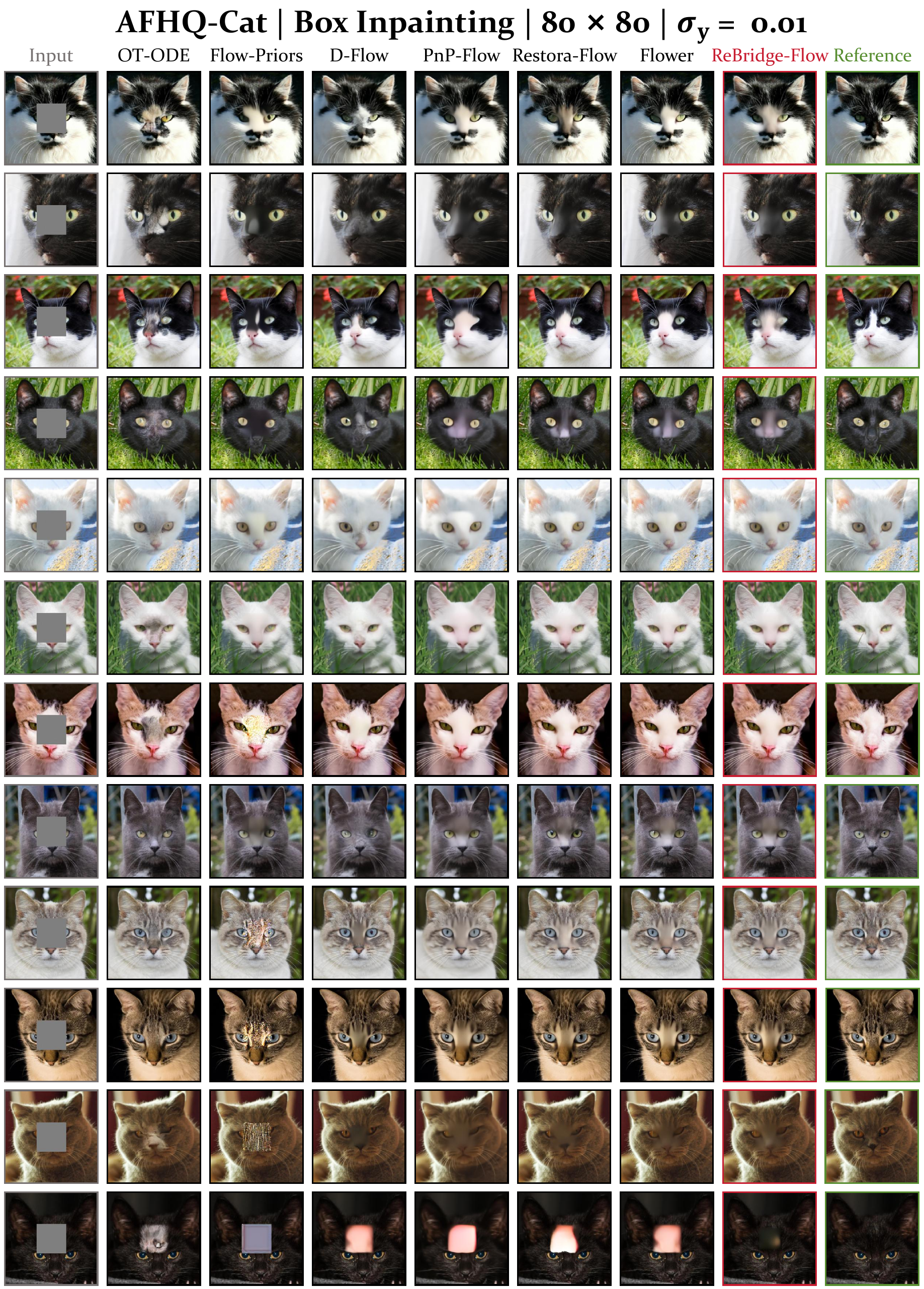} 
\caption{Qualitative visual comparison of Box Inpainting ($80 \times 80$, $\sigma_\mathbf{y} = 0.01$) on AFHQ-Cat.}
\label{Figure_Supplementary_Materials_Qualitative_10}
\end{figure}

\begin{figure}[!t]
\centering
\includegraphics[width=0.88\textwidth]{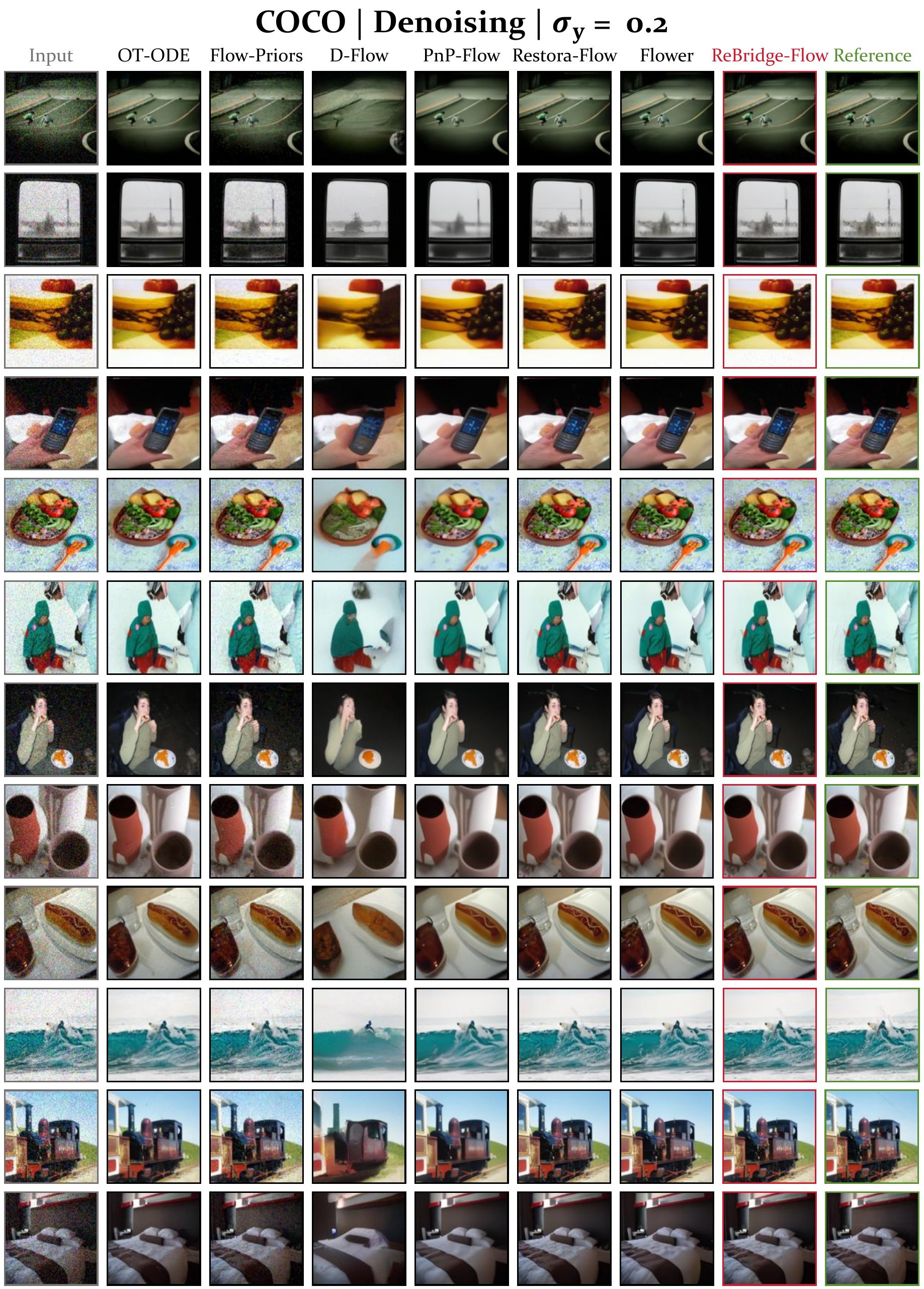} 
\caption{Qualitative visual comparison of Denoising ($\sigma_\mathbf{y} = 0.2$) on COCO.}
\label{Figure_Supplementary_Materials_Qualitative_11}
\end{figure}

\begin{figure}[!t]
\centering
\includegraphics[width=0.88\textwidth]{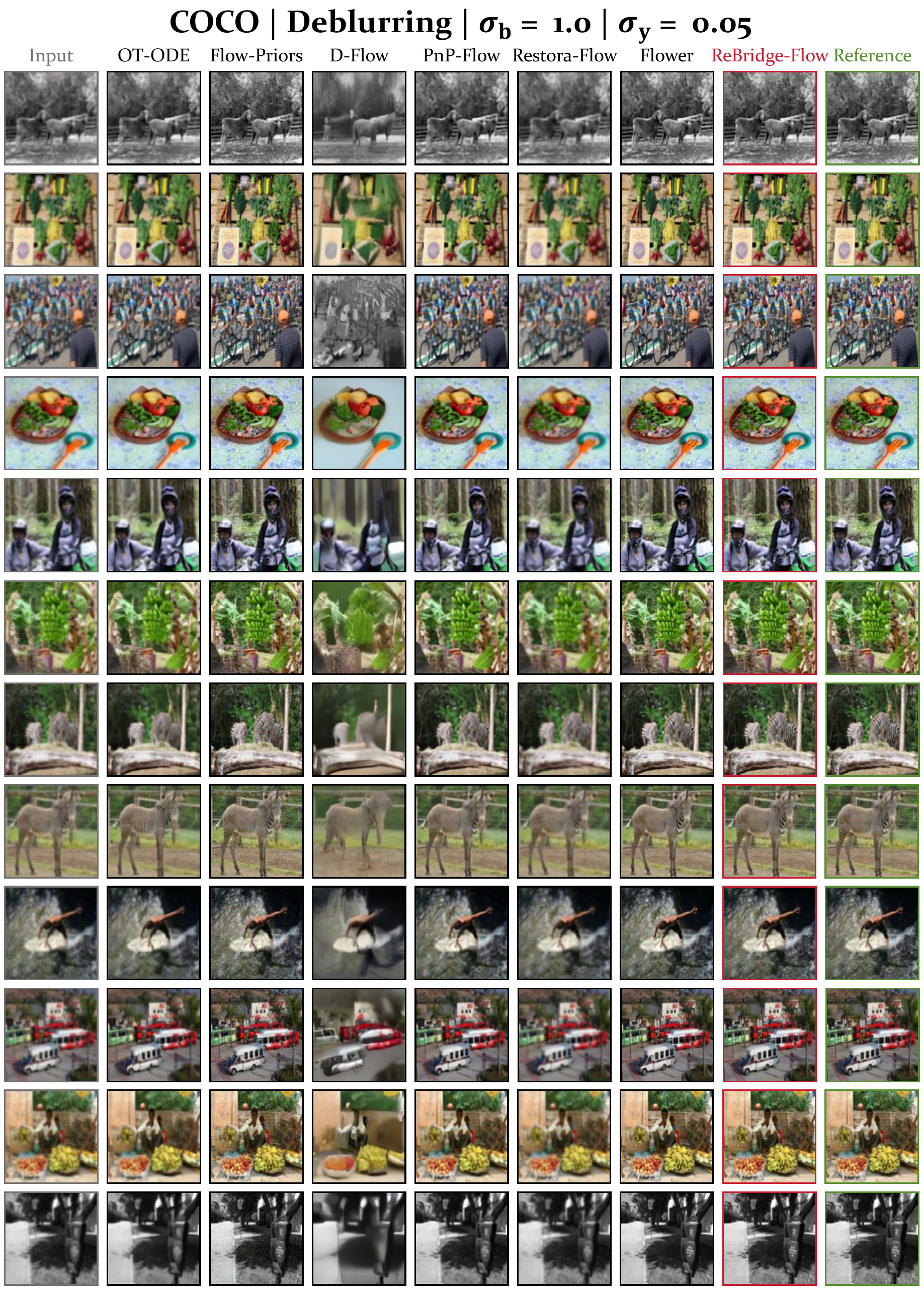} 
\caption{Qualitative visual comparison of Deblurring ($\sigma_b = 1.0$, $\sigma_\mathbf{y} = 0.05$) on COCO.}
\label{Figure_Supplementary_Materials_Qualitative_12}
\end{figure}

\begin{figure}[!t]
\centering
\includegraphics[width=0.88\textwidth]{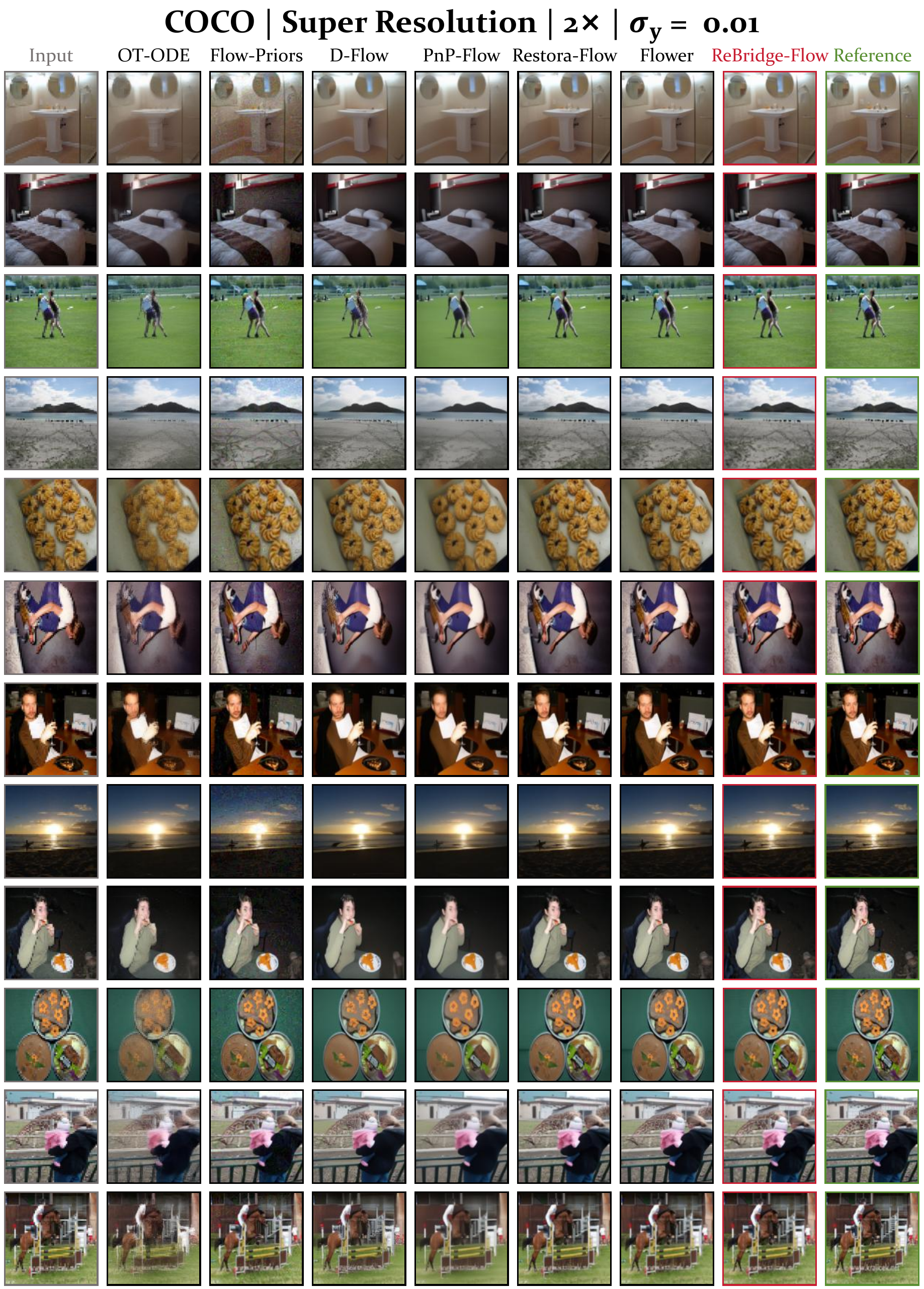} 
\caption{Qualitative visual comparison of Super Resolution ($2\times$, $\sigma_\mathbf{y} = 0.01$) on COCO.}
\label{Figure_Supplementary_Materials_Qualitative_13}
\end{figure}

\begin{figure}[!t]
\centering
\includegraphics[width=0.88\textwidth]{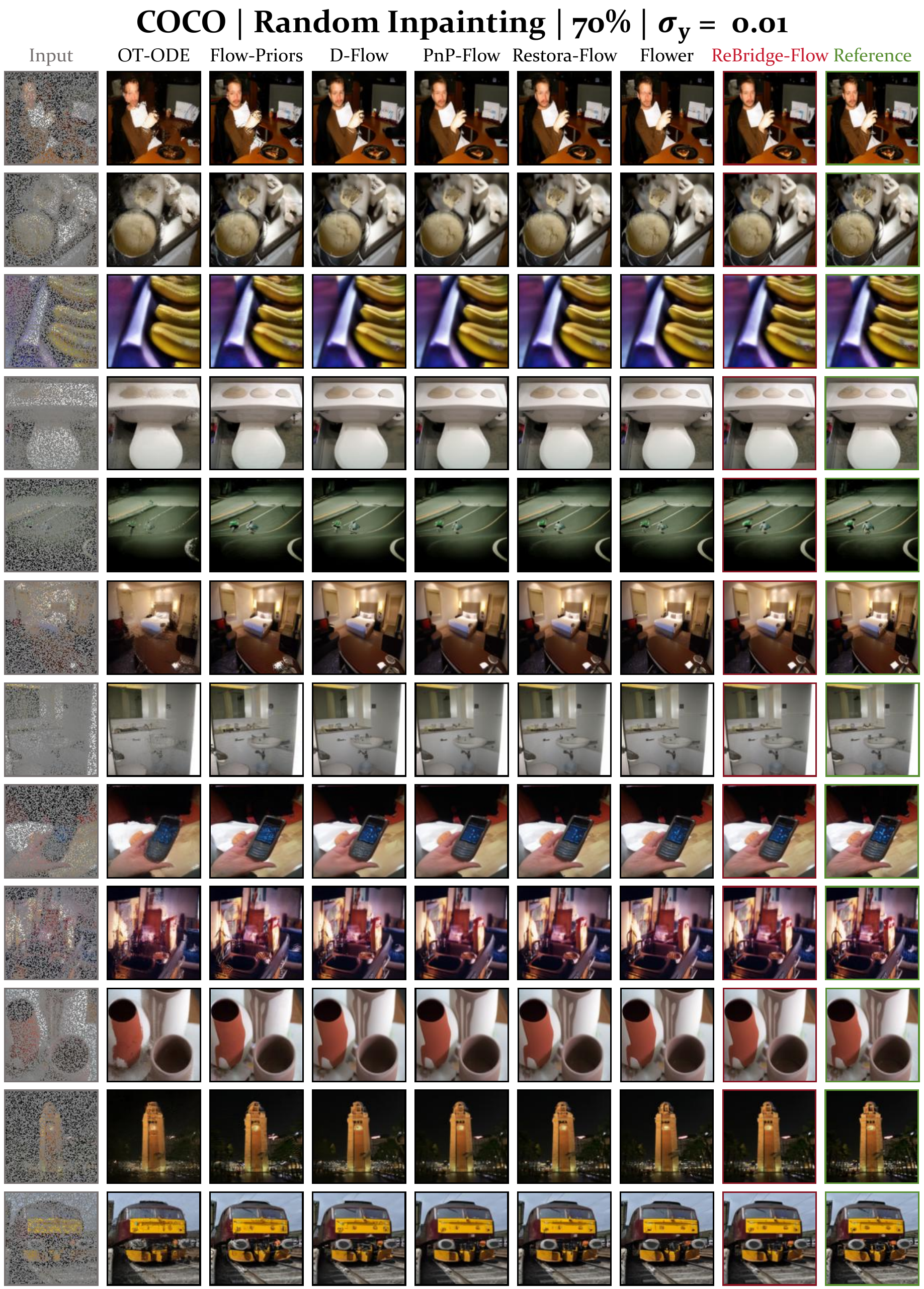} 
\caption{Qualitative visual comparison of Random Inpainting (70\%, $\sigma_\mathbf{y} = 0.01$) on COCO.}
\label{Figure_Supplementary_Materials_Qualitative_14}
\end{figure}

\begin{figure}[!t]
\centering
\includegraphics[width=0.88\textwidth]{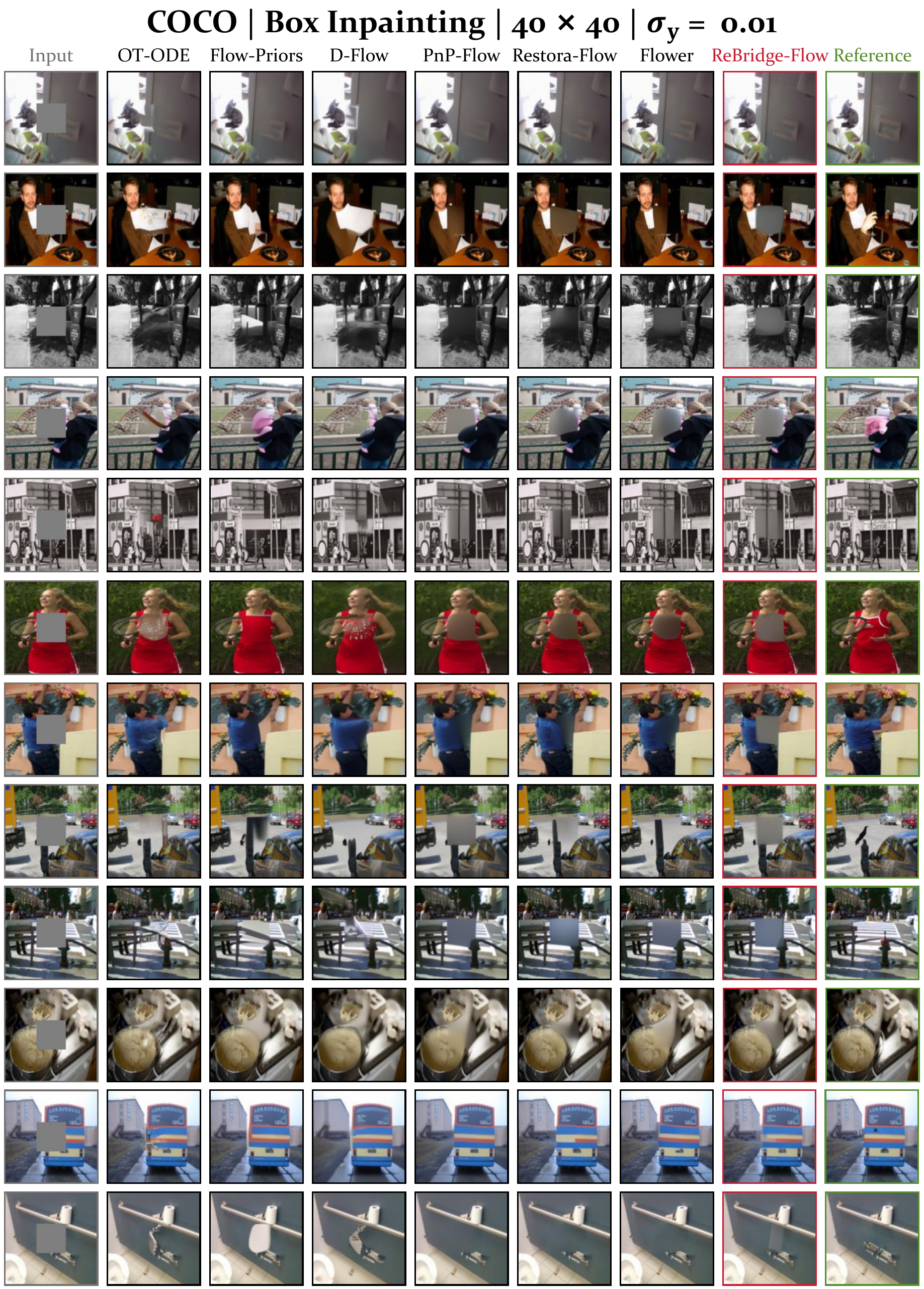} 
\caption{Qualitative visual comparison of Box Inpainting ($40 \times 40$, $\sigma_\mathbf{y} = 0.01$) on COCO.}
\label{Figure_Supplementary_Materials_Qualitative_15}
\end{figure}

\begin{figure}[!t]
\centering
\includegraphics[width=0.88\textwidth]{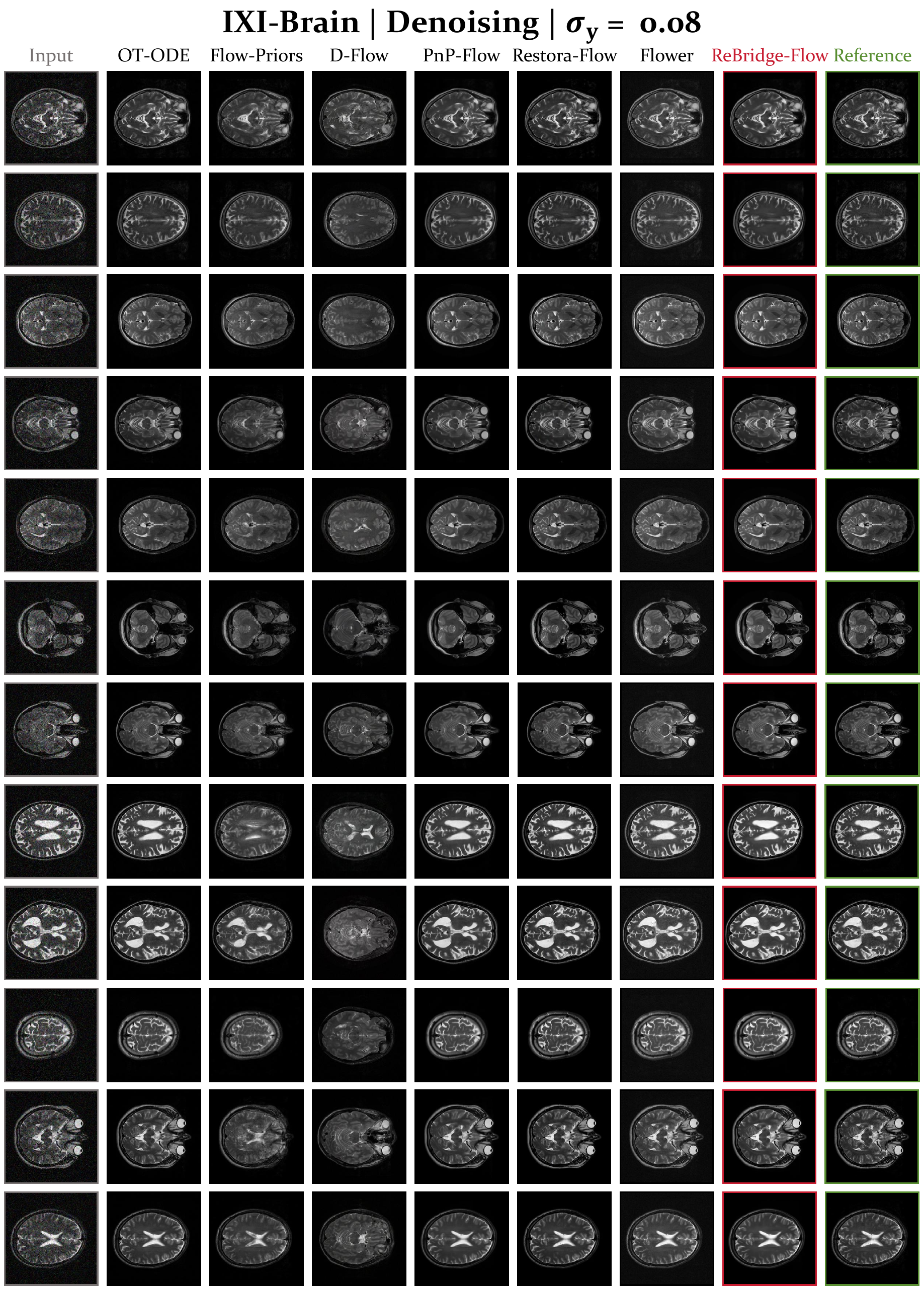} 
\caption{Qualitative visual comparison of Denoising ($\sigma_\mathbf{y} = 0.08$) on IXI-Brain.}
\label{Figure_Supplementary_Materials_Qualitative_16}
\end{figure}

\begin{figure}[!t]
\centering
\includegraphics[width=0.88\textwidth]{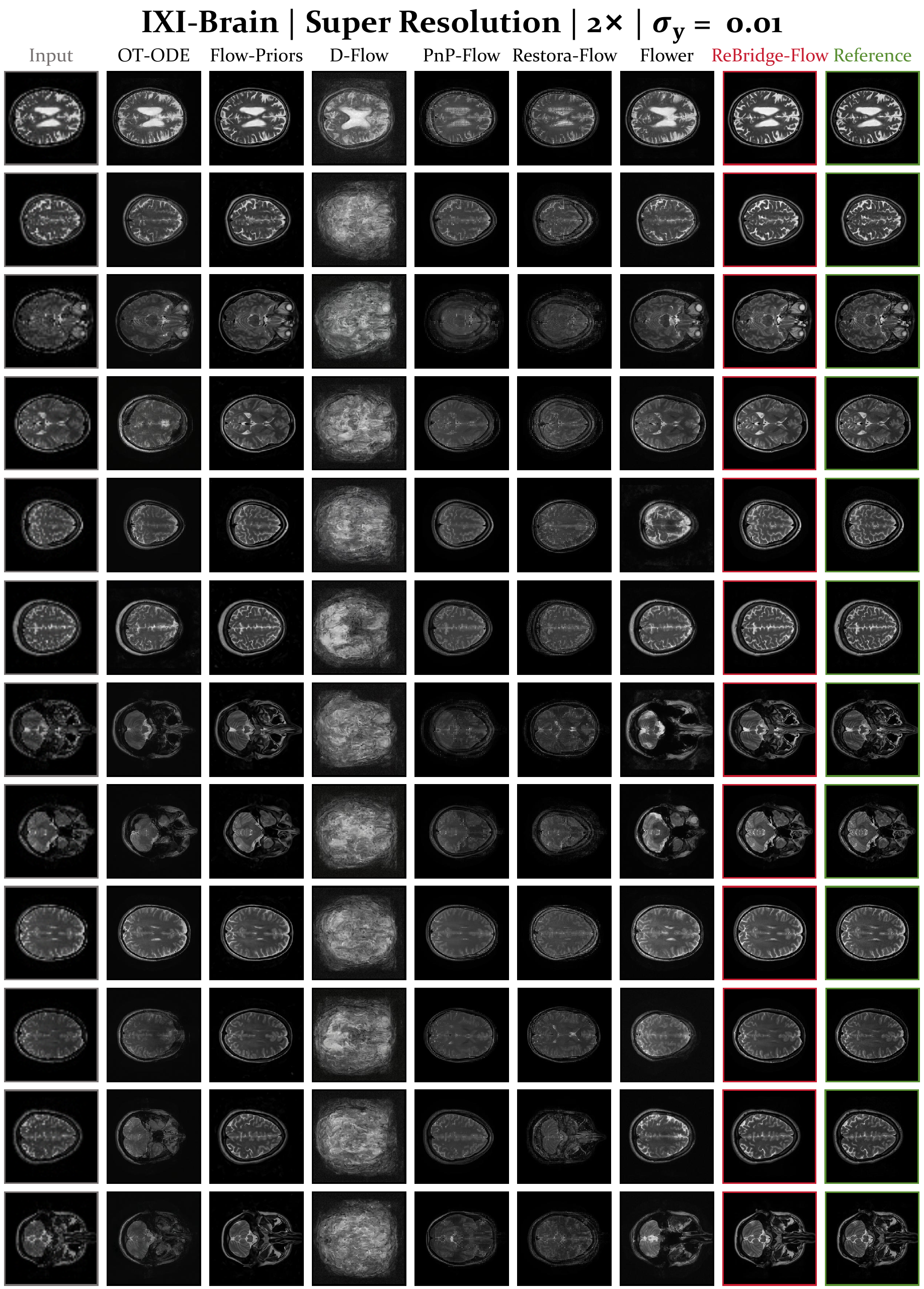} 
\caption{Qualitative visual comparison of Super Resolution ($2\times$, $\sigma_\mathbf{y} = 0.01$) on IXI-Brain.}
\label{Figure_Supplementary_Materials_Qualitative_17}
\end{figure}

\begin{figure}[!t]
\centering
\includegraphics[width=0.88\textwidth]{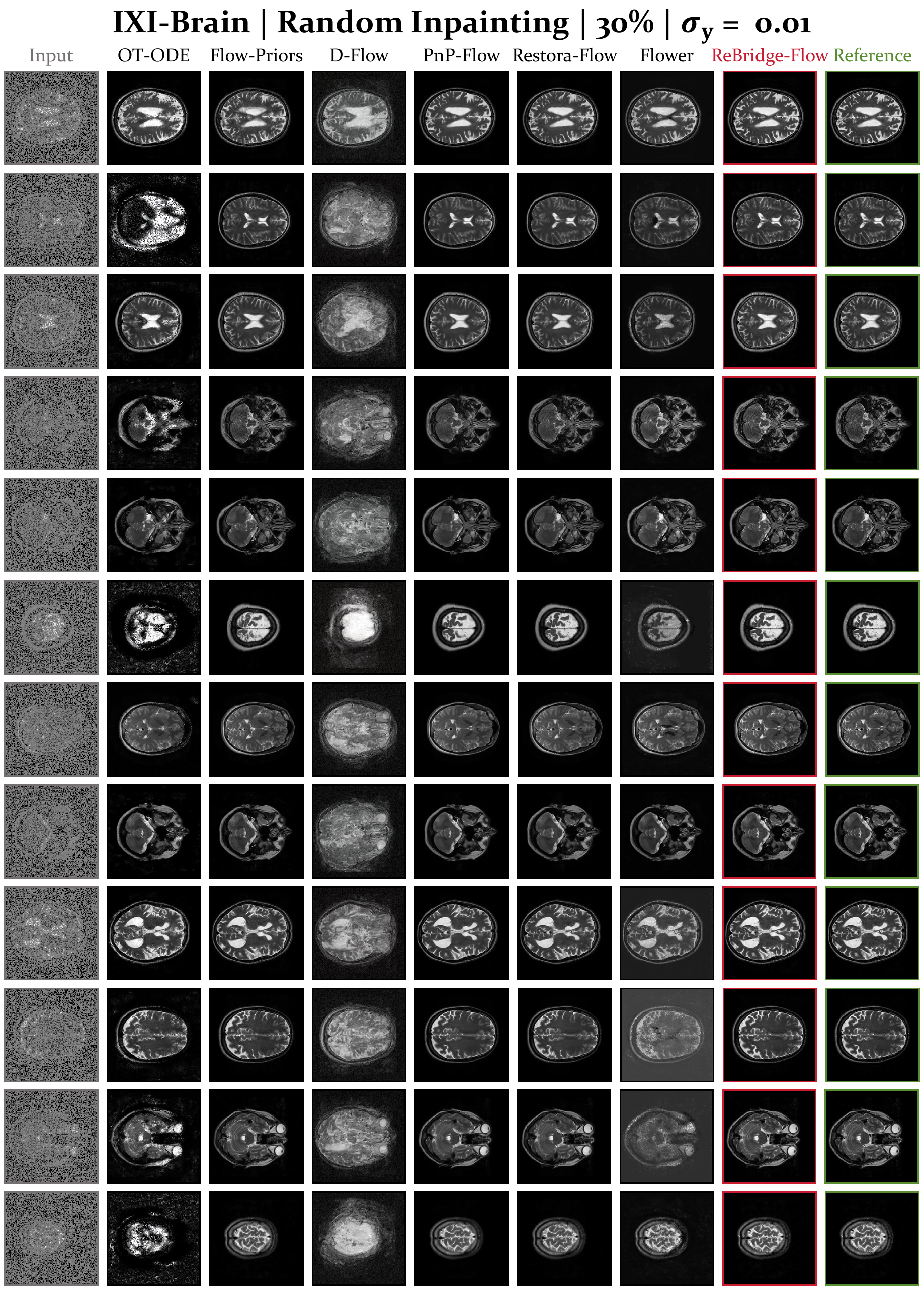} 
\caption{Qualitative visual comparison of Random Inpainting (30\%, $\sigma_\mathbf{y} = 0.01$) on IXI-Brain.}
\label{Figure_Supplementary_Materials_Qualitative_18}
\end{figure}

\begin{figure}[!t]
\centering
\includegraphics[width=0.88\textwidth]{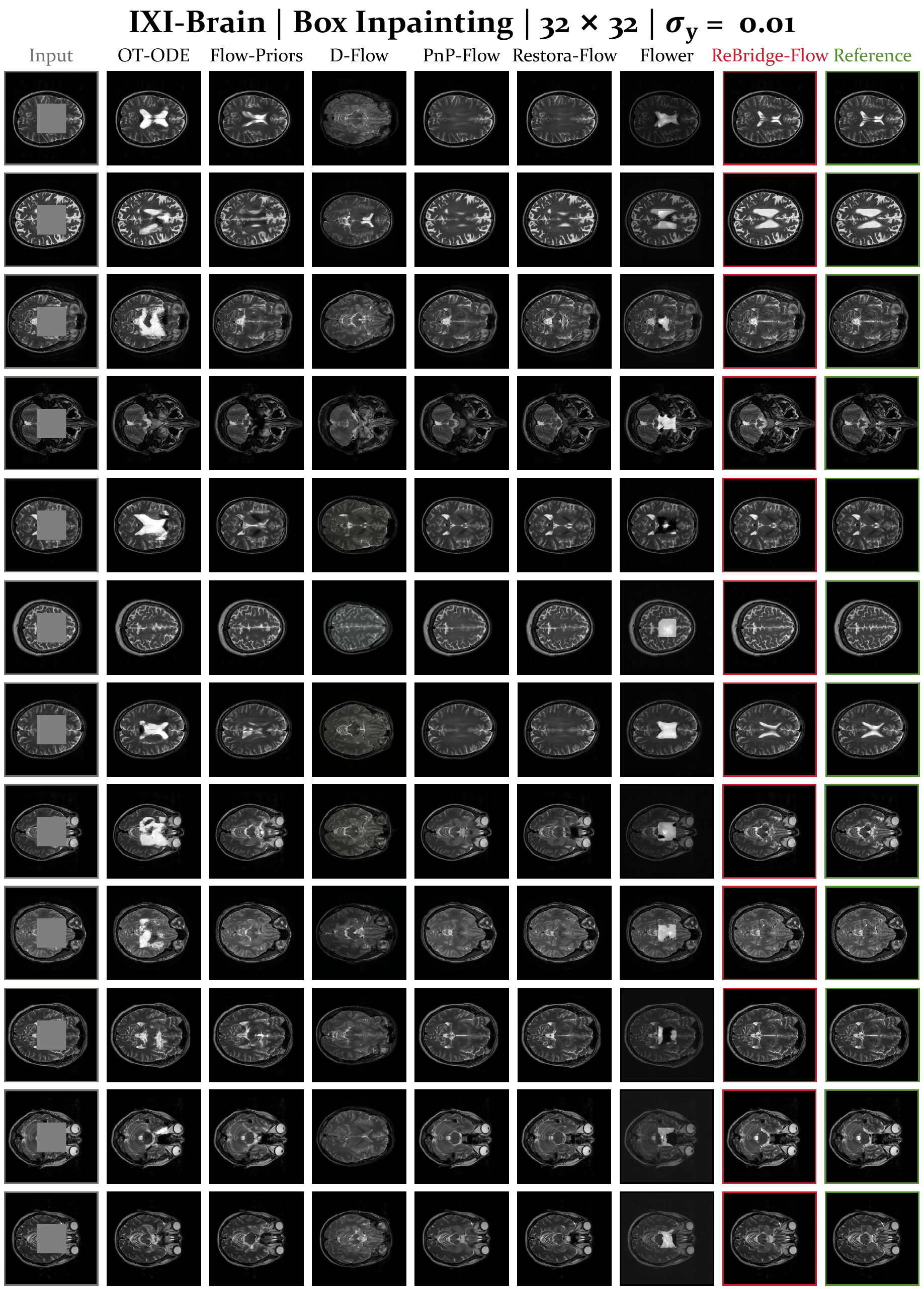} 
\caption{Qualitative visual comparison of Box Inpainting ($32 \times 32$, $\sigma_\mathbf{y} = 0.01$) on IXI-Brain.}
\label{Figure_Supplementary_Materials_Qualitative_19}
\end{figure}

\begin{figure}[!t]
\centering
\includegraphics[width=0.88\textwidth]{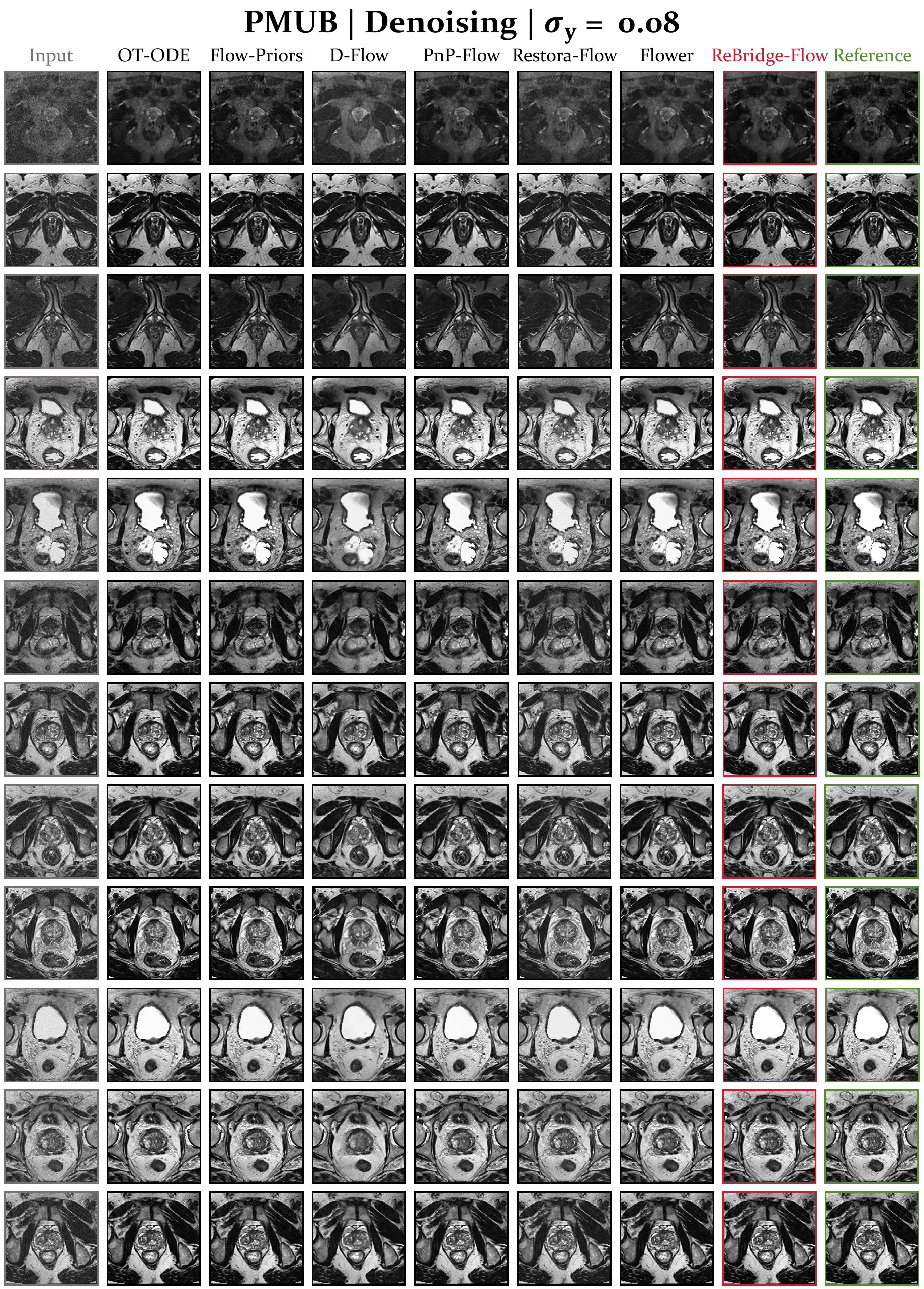} 
\caption{Qualitative visual comparison of Denoising ($\sigma_\mathbf{y} = 0.08$) on PMUB.}
\label{Figure_Supplementary_Materials_Qualitative_20}
\end{figure}

\begin{figure}[!t]
\centering
\includegraphics[width=0.88\textwidth]{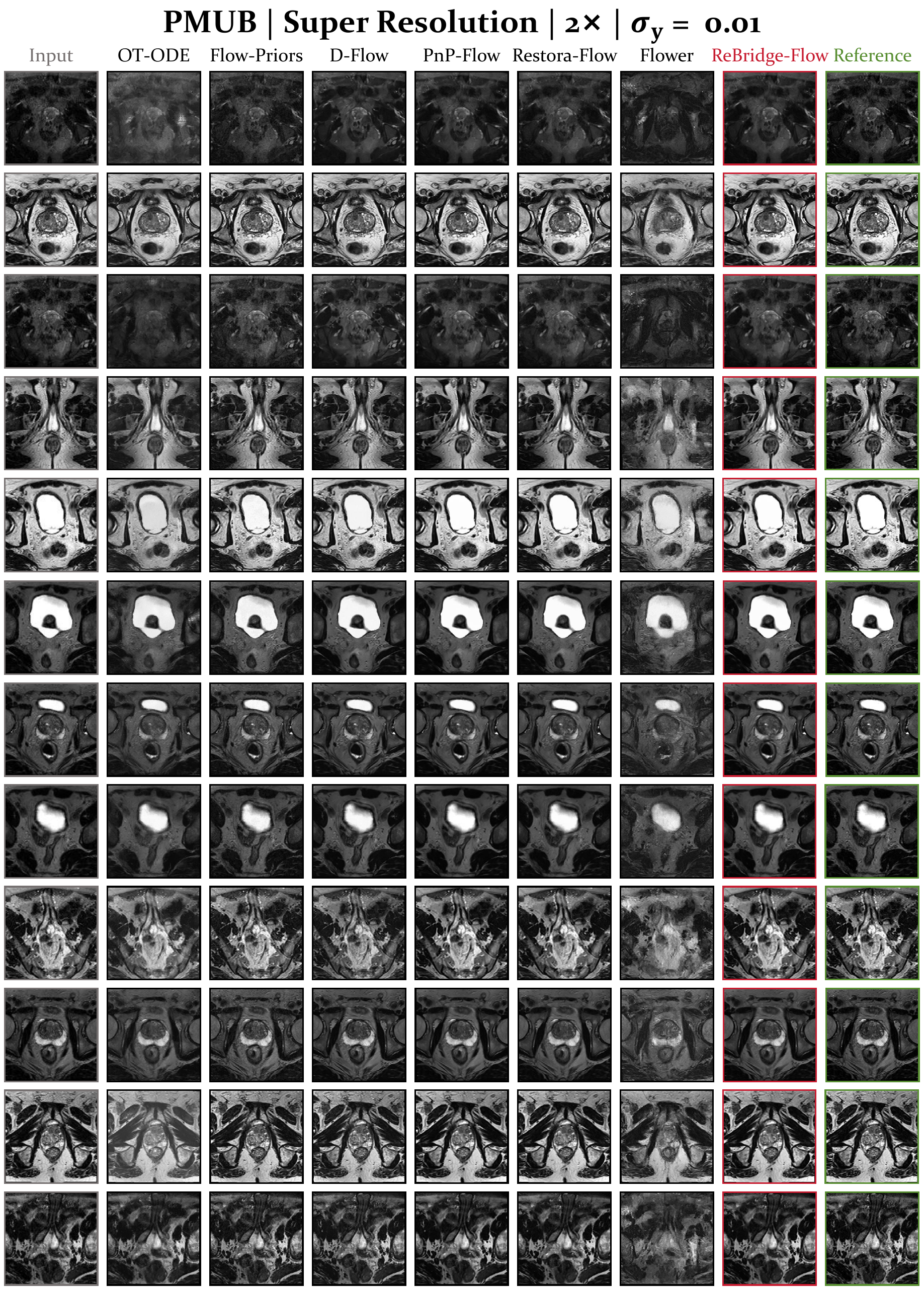} 
\caption{Qualitative visual comparison of Super Resolution ($2\times$, $\sigma_\mathbf{y} = 0.01$) on PMUB.}
\label{Figure_Supplementary_Materials_Qualitative_21}
\end{figure}

\begin{figure}[!t]
\centering
\includegraphics[width=0.88\textwidth]{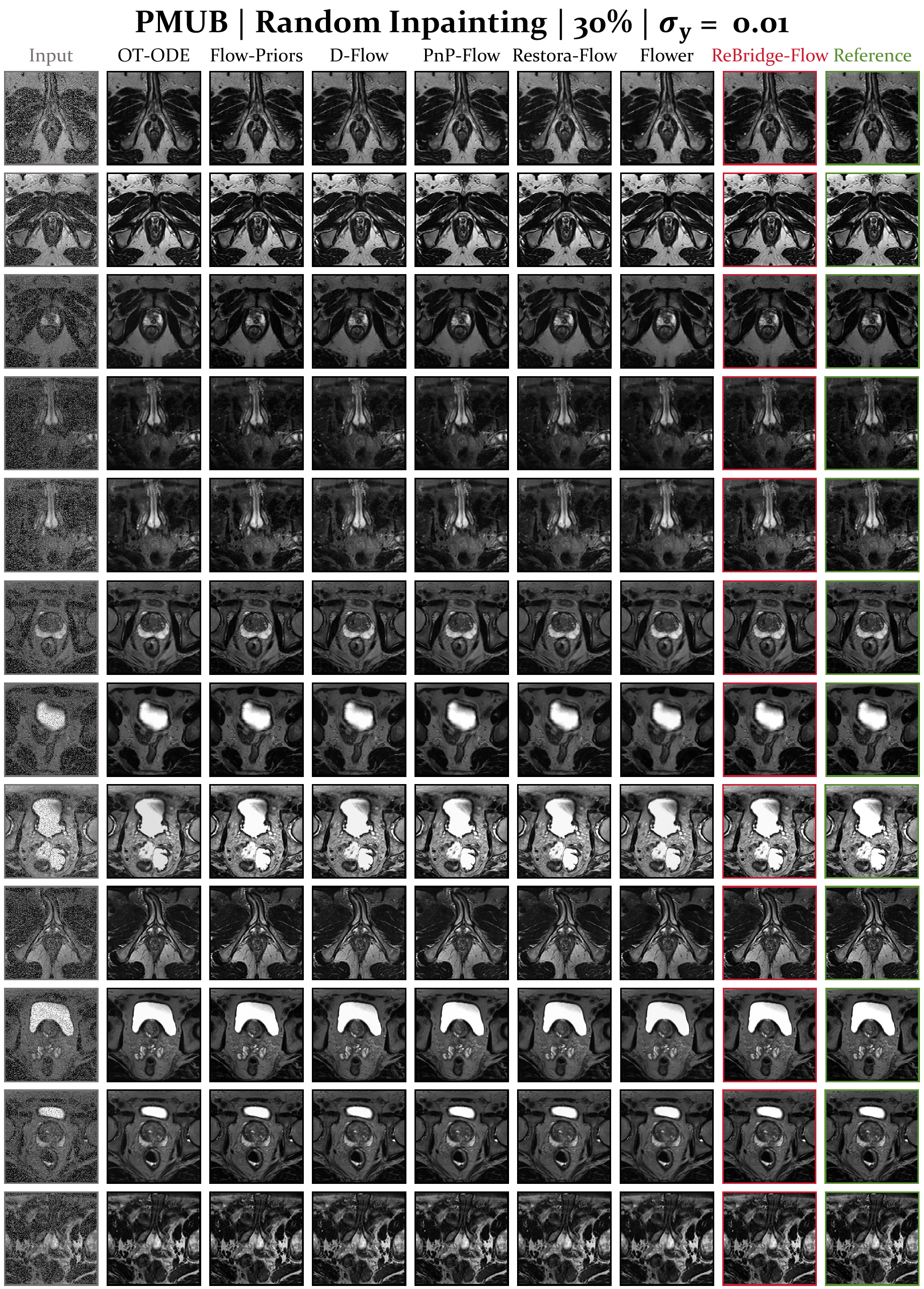} 
\caption{Qualitative visual comparison of Random Inpainting (30\%, $\sigma_\mathbf{y} = 0.01$) on PMUB.}
\label{Figure_Supplementary_Materials_Qualitative_22}
\end{figure}

\begin{figure}[!t]
\centering
\includegraphics[width=0.88\textwidth]{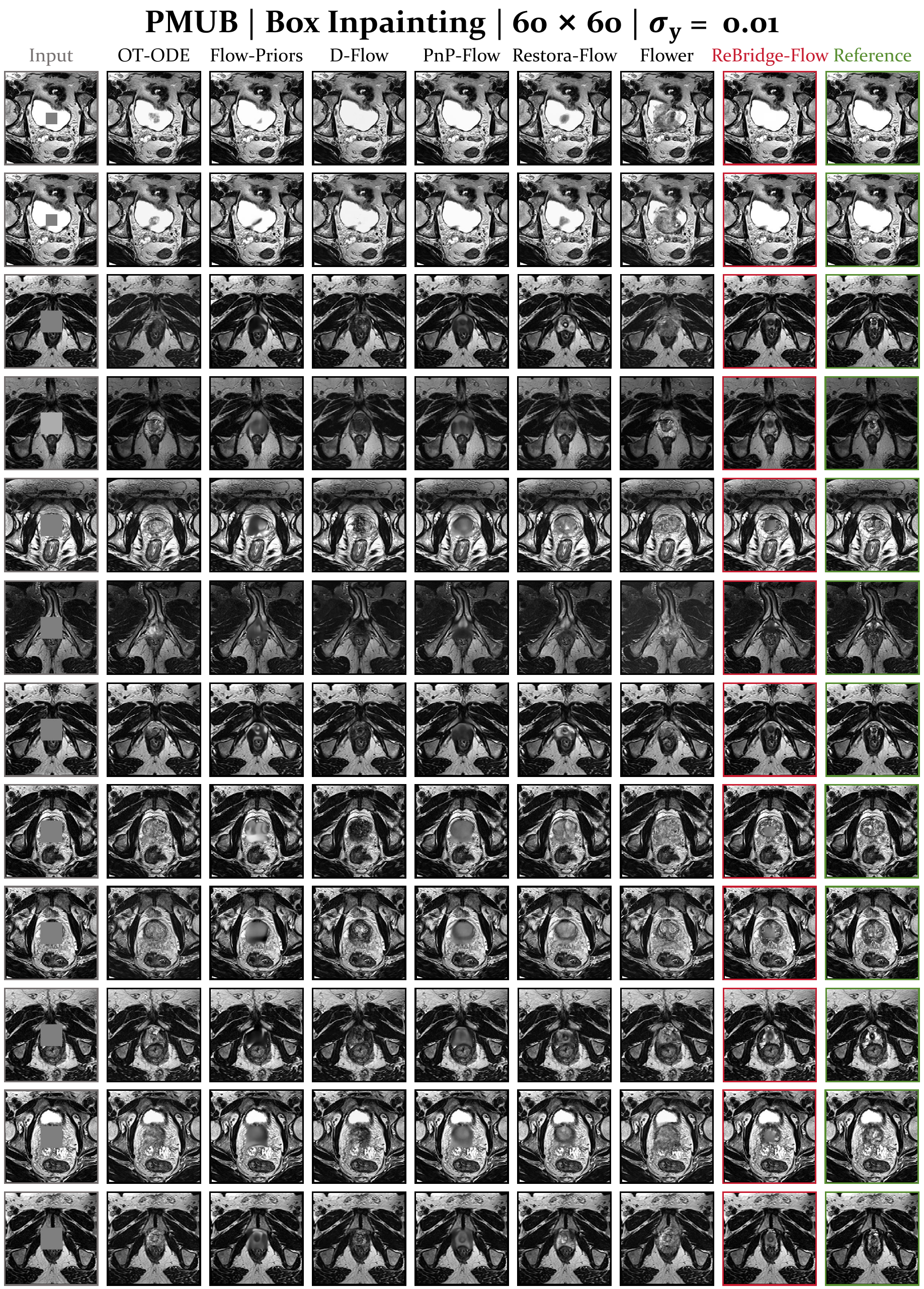} 
\caption{Qualitative visual comparison of Box Inpainting ($60 \times 60$, $\sigma_\mathbf{y} = 0.01$) on PMUB.}
\label{Figure_Supplementary_Materials_Qualitative_23}
\end{figure}

\begin{figure}[!t]
\centering
\includegraphics[width=0.88\textwidth]{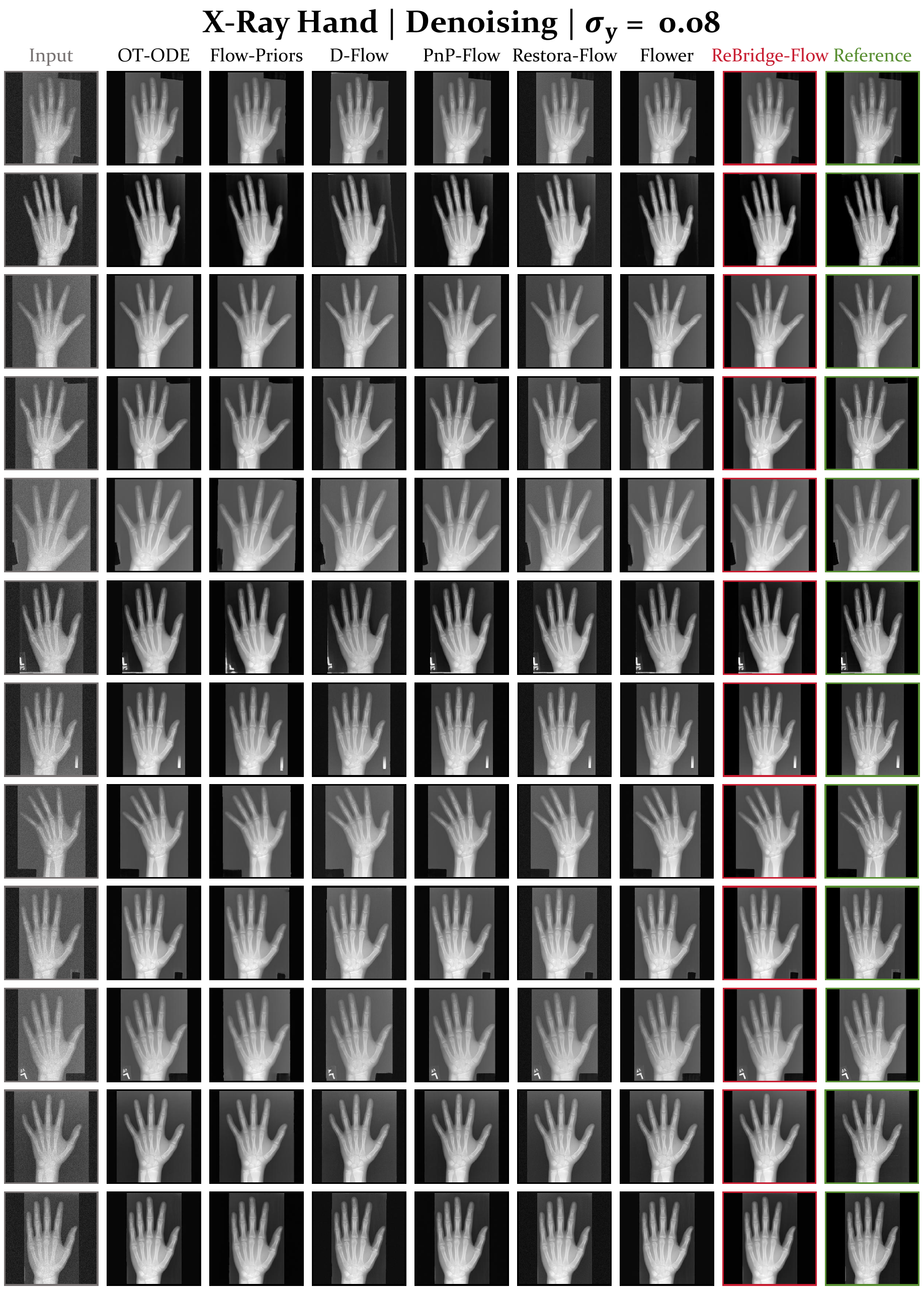} 
\caption{Qualitative visual comparison of Denoising ($\sigma_\mathbf{y} = 0.08$) on X-Ray Hand.}
\label{Figure_Supplementary_Materials_Qualitative_24}
\end{figure}

\begin{figure}[!t]
\centering
\includegraphics[width=0.88\textwidth]{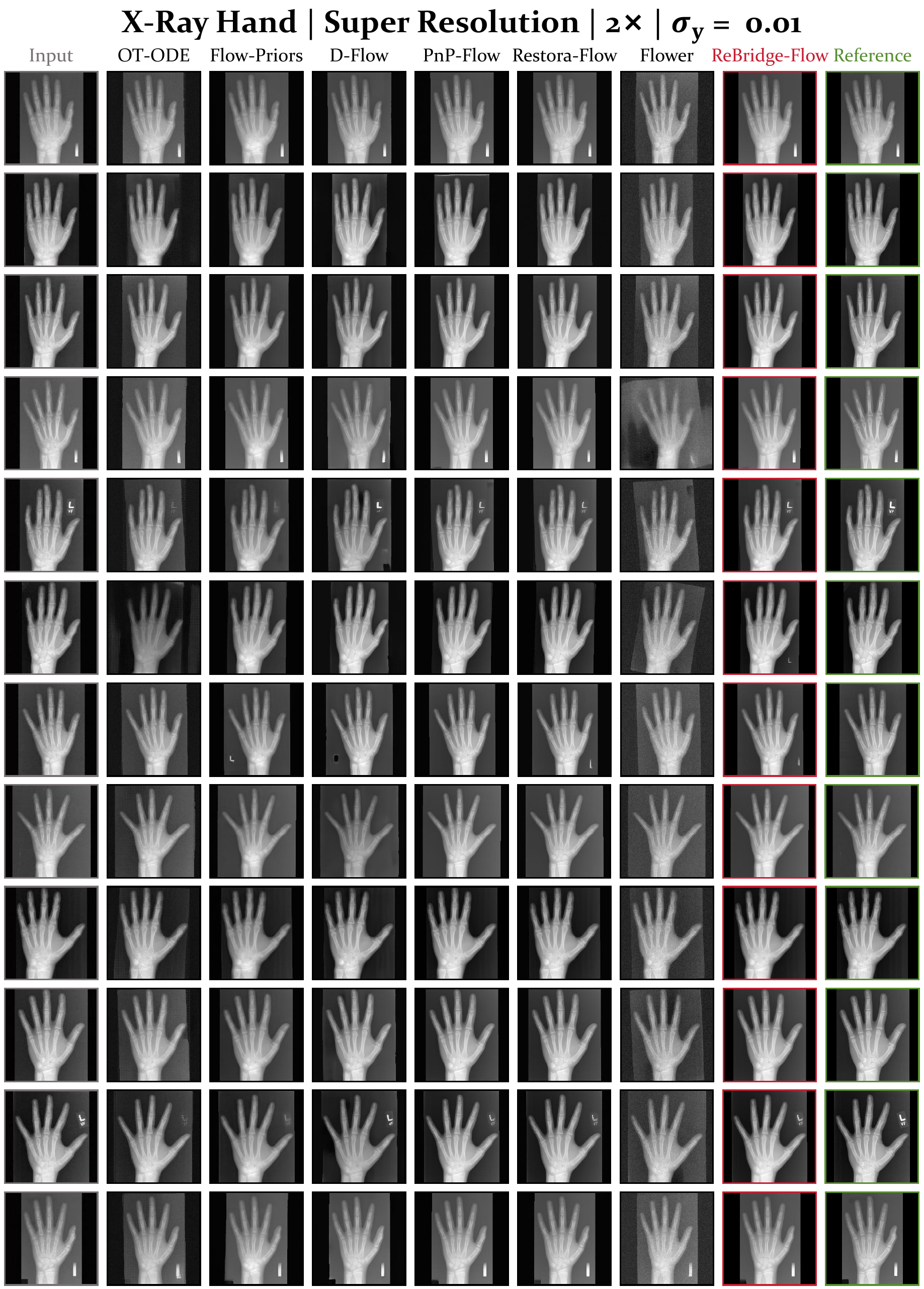} 
\caption{Qualitative visual comparison of Super Resolution ($2\times$, $\sigma_\mathbf{y} = 0.01$) on X-Ray Hand.}
\label{Figure_Supplementary_Materials_Qualitative_25}
\end{figure}

\begin{figure}[!t]
\centering
\includegraphics[width=0.88\textwidth]{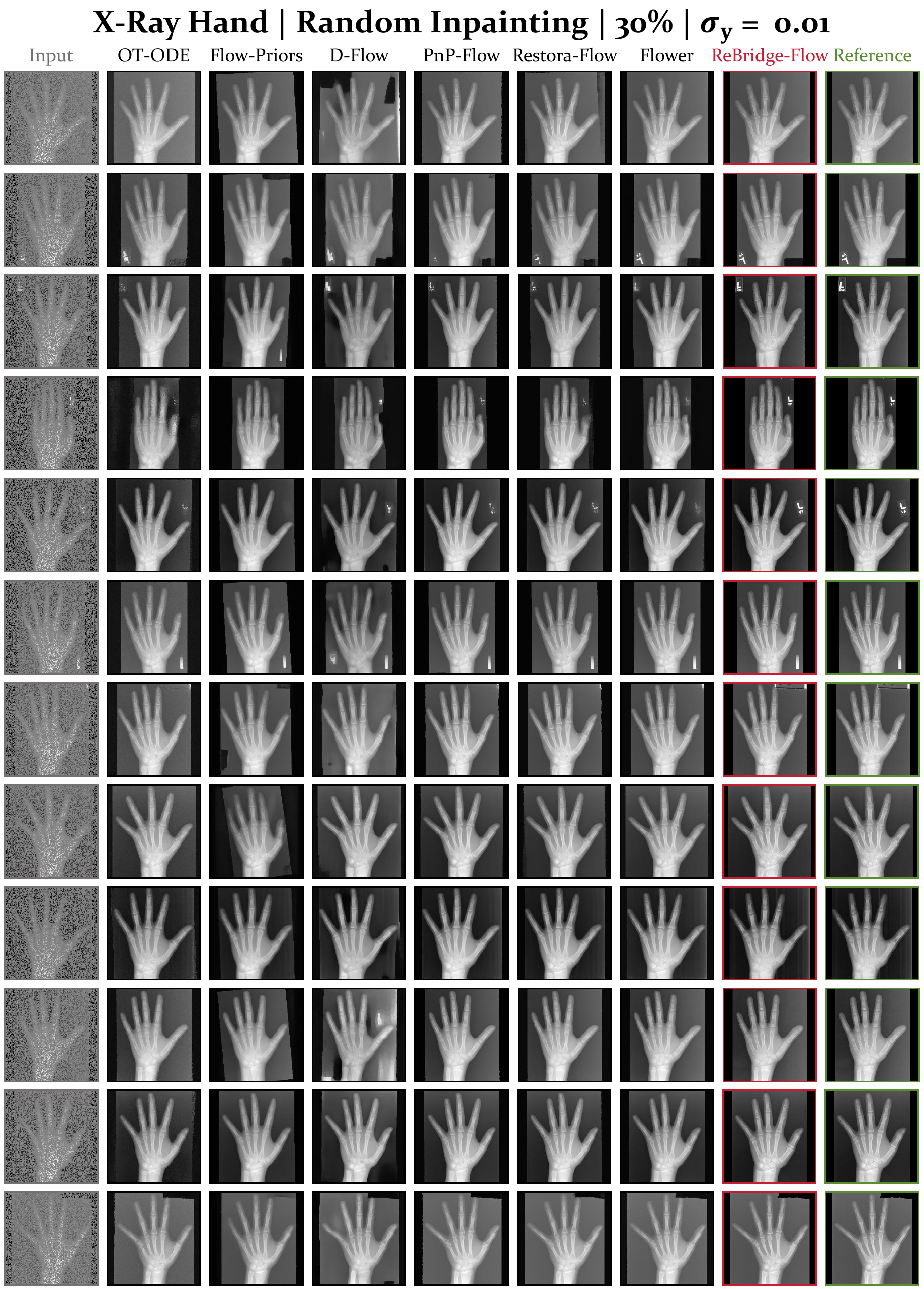} 
\caption{Qualitative visual comparison of Random Inpainting (30\%, $\sigma_\mathbf{y} = 0.01$) on X-Ray Hand.}
\label{Figure_Supplementary_Materials_Qualitative_26}
\end{figure}

\begin{figure}[!t]
\centering
\includegraphics[width=0.88\textwidth]{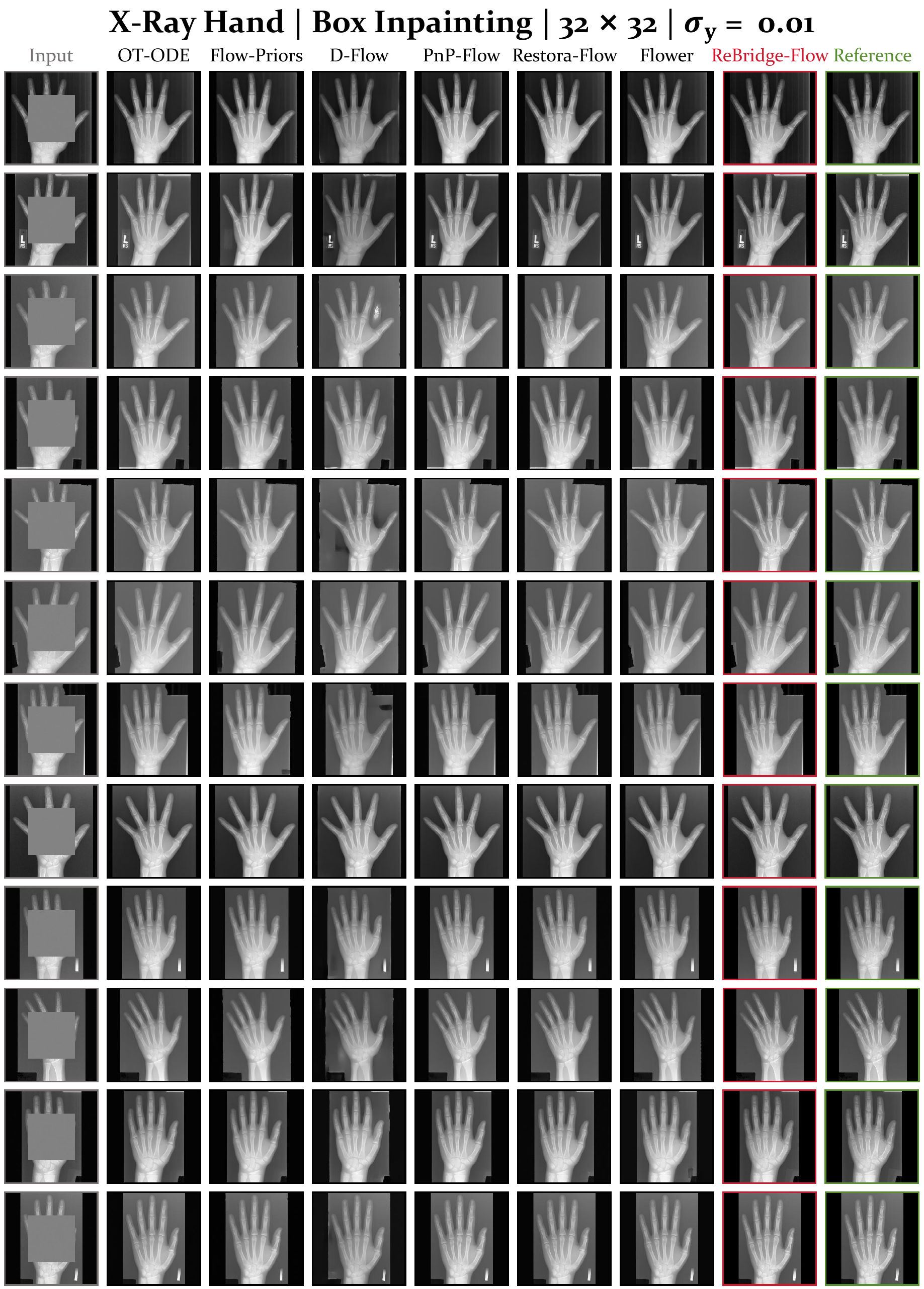} 
\caption{Qualitative visual comparison of Box Inpainting ($32 \times 32$, $\sigma_\mathbf{y} = 0.01$) on X-Ray Hand.}
\label{Figure_Supplementary_Materials_Qualitative_27}
\end{figure}

\end{document}